\documentclass[twoside]{article}

\usepackage[preprint]{aistats2026}
\usepackage{tikz}
\usepackage{algorithm}
\usepackage{algorithmic}
\usepackage{natbib}
\usepackage{dutchcal}
\usepackage{amsmath}
\usepackage{amsthm}
\usepackage{amssymb}
\usepackage{graphicx}
\usepackage{mathtools}
\usepackage{natbib}
\usepackage{url}
\usepackage{graphicx}
\usepackage{tikz-cd}
\usepackage{array,epsfig,fancyheadings,rotating}
\usepackage{bm}
\usepackage{enumitem} 
\usepackage{subcaption}
\usepackage{wrapfig} 

\newtheorem{theorem}{Theorem}
 
\newtheorem{lemma}{Lemma}

\newtheorem{definition}{Definition}

\newtheorem{condition}{Condition} 

\theoremstyle{definition}

\usepackage{hyperref}
\usepackage{cleveref}

\setlist[itemize]{itemsep=1.5pt, topsep=4pt} 
\setlist[enumerate]{itemsep=1.5pt, topsep=4pt}

\begin{document}

\runningtitle{Provable Effects of Data Replay in Continual Learning: A Feature Learning Perspective}

% \runningauthor{Surname 1, Surname 2, Surname 3, ...., Surname n}

\twocolumn[

\aistatstitle{Provable Effects of Data Replay in Continual Learning: \\ A Feature Learning Perspective}

\aistatsauthor{Meng Ding$^{2}$ \quad Jinhui Xu$^{1}$ \quad Kaiyi Ji$^{2}$ }

\aistatsaddress{ 
$^{1}$ School of Information Science and Technology, USTC and Institute of Artificial Intelligence, HCNSC \\
$^{2}$ Department of Computer Science and Engineering, SUNY at Buffalo }

]

\begin{abstract}
Continual learning (CL) aims to train models on a sequence of tasks while retaining performance on previously learned ones. A core challenge in this setting is \textit{catastrophic forgetting}, where new learning interferes with past knowledge. Among various mitigation strategies, data-replay methods—where past samples are periodically revisited—are considered simple yet effective, especially when memory constraints are relaxed. However, the theoretical effectiveness of full data replay, where all past data is accessible during training, remains largely unexplored. In this paper, we present a comprehensive theoretical framework for analyzing full data-replay training in continual learning from a feature learning perspective. Adopting a multi-view data model, we identify the signal-to-noise ratio (SNR) as a critical factor affecting forgetting. Focusing on task-incremental binary classification across $M$ tasks, our analysis verifies two key conclusions: (1) forgetting can still occur under full replay when the cumulative noise from later tasks dominates the signal from earlier ones; and (2) with sufficient signal accumulation, data replay can recover earlier tasks-even if their initial learning was poor. Notably, we uncover a novel insight into task ordering: prioritizing higher-signal tasks not only facilitates learning of lower-signal tasks but also helps prevent catastrophic forgetting.
% —highlighting the importance of order-aware replay strategies
We validate our theoretical findings through synthetic experiments that visualize the interplay between signal learning and noise memorization across varying SNRs and task correlation regimes.

\end{abstract}

\section{Introduction}
\label{sec:intro}
Continual learning (CL) is a paradigm in machine learning where models learn sequentially from a stream of tasks or datasets, continually adapting to new information while preserving performance on previously learned tasks \cite{parisi2019continual,wang2024comprehensive}. The key challenge in continual learning is \textit{catastrophic forgetting}, a phenomenon where modern models drastically lose previously acquired knowledge when learning new tasks \cite{mccloskey1989catastrophic,kirkpatrick2017overcoming,korbak2022controlling}.

Previous empirical research alleviating catastrophic forgetting in continual learning can be broadly classified into five categories \cite{wang2024comprehensive}: regularization-, replay-, optimization-, representation-, and architecture-based approaches. Regularization-based methods \cite{ritter2018online,aljundi2018memory,titsias2019functional,pan2020continual,benzing2022unifying,lin2022towards} introduce explicit regularizers to balance learning across tasks, often relying on a frozen copy of the old model for reference. Replay-based methods \cite{lopez2017gradient,riemer2018learning,chaudhry2019tiny,yoon2021online,shim2021online,tiwari2022gcr,van2020brain,liu2020generative,zheng2024multi} approximate and recover past data distributions to reinforce old knowledge. Optimization-based methods \cite{lopez2017gradient,chaudhry2018efficient,tang2021layerwise,liu2020generative,wang2022anti} focus on modifying the learning dynamics, such as through gradient projection, to avoid interference. Representation-based methods \cite{wu2022class,shi2022mimicking,wang2022s,mcdonnell2023ranpac,le2024mixture} aim to develop and leverage task-robust representations via the advantages of pretraining, while architecture-based methods \cite{gurbuz2022nispa,douillard2022dytox,miao2021continual,ostapenko2021continual} design adaptable model structures that share parameters across tasks to retain knowledge.

Among these approaches, data-replay methods are often regarded as the most straightforward to implement—particularly when buffer constraints are ignored—since they rely on storing and periodically retraining on past task samples to preserve prior knowledge. However, their empirical success typically hinges on careful sample selection \cite{chaudhry2019tiny,riemer2018learning}. When full data replay is employed, exposing the model to all historical data, the effectiveness of this strategy remains an open question: does it still reliably counteract forgetting under such conditions?

To address this, we present a comprehensive theoretical analysis showing that full data-replay training does not always effectively mitigate forgetting. 
% we develop the first theoretical framework that rigorously analyzes full data-replay training in continual learning and reveals that it does not always succeed in mitigating forgetting. 
% \meng{Existing theoretical studies in continual learning primarily focus on simplified two-task setups or naive sequential training, leaving fundamental gaps in understanding the behavior of replay-based methods across general multi-task settings (see \cref{sec:related_work} for details).} 

Our contribution can be summarized as follows:

\begin{itemize}
    \item We develop a thorough theoretical framework that rigorously analyzes full data-replay training within the theoretical continual learning community. Prior studies have primarily focused on simplified linear regression models, two-task setups, or naive sequential training, leaving fundamental gaps in understanding the behavior of replay-based methods in general multi-task settings (see \cref{sec:related_work} for details). More specifically: (1) we adopt a multi-view data model (following \cite{allen2020towards}), where each data point consists of both feature signals and noise, allowing us to introduce the signal-to-noise ratio as a key factor governing whether forgetting occurs; and (2) we focus on task-incremental binary classification in a general $M$-task setting, where each task is associated with a distinct feature signal vector. This formulation enables us to characterize how task ordering and inter-task correlation influence forgetting.
    \item Based on the above data model, our results formally show two interesting findings: (1) Even with full data replay, forgetting of task $k$ after replaying up to task $m$ ($m > k$) can still occur under certain SNR regimes, particularly when the cumulative noise from later tasks outweighs the signal intensity of task $k$. (2) Even if the performance on task $k$ is initially unsatisfactory, data replay can help amplify the signal intensity, enabling the model to recover task $k$ 's information in later stages-provided the accumulated signal outweighs the noise. Furthermore, by incorporating task correlation, we uncover a key insight into task ordering: prioritizing higher-signal tasks not only facilitates learning for lower-signal tasks but can also help prevent catastrophic forgetting. This observation suggests a promising direction for designing order-aware replay strategies in future continual learning frameworks.
    \item We complement our theory with synthetic experiments that examine the dynamics of signal learning and noise memorization during continual training under full data replay, comparing different task orderings across varying levels of task correlation and SNR conditions.
\end{itemize}

% Specifically, we reveal that its performance critically depends on the signal-to-noise ratio (SNR) inherent in the data, highlighting conditions under which data-replay methods may fail to preserve learned knowledge from previous tasks.

\section{Related Work}
\label{sec:related_work}

\paragraph{Replay-based Continual Learning.}

Replay‑based approaches mitigate catastrophic forgetting by approximating the original data distribution during continual training. Specifically, they can be categorized based on how they reconstruct previous data: (1) Experience replay. A small subset of historical samples is stored in a memory buffer and replayed alongside new data. Early work stored a fixed or class‑balanced share of examples from each batch to enforce simple selection rules \cite{lopez2017gradient,riemer2018learning,chaudhry2019tiny}. Later studies introduced gradient‑aware or optimizable selection schemes to maximize sample diversity \cite{yoon2021online,shim2021online,tiwari2022gcr}, and used data‑augmentation techniques to improve storage efficiency \cite{ebrahimi2021remembering,kumari2022retrospective}. 
(2) Generative replay (pseudo‑rehearsal). Instead of storing raw inputs, an auxiliary generative model is trained to synthesise data from previous tasks, and these pseudo‑examples are replayed alongside new data during subsequent training. To mitigate forgetting in the generative model itself, additional strategies are often employed, such as weight regularization to preserve past knowledge \cite{nguyen2017variational,wang2021triple}, task-specific parameter allocation (e.g., binary masks) \cite{ostapenko2019learning,cong2020gan} to reduce inter-task interference, and feature-level replay to simplify conditional generation by replaying intermediate features instead of raw data \cite{van2020brain,liu2020generative}.
In practice, replay methods must work with a limited memory buffer. For analytical clarity, however, we assume an unlimited buffer that stores all past data; extending the theory to constrained‑memory settings will be left for future work.

\paragraph{Theoretical Continual Learning.}

Recent theoretical work on catastrophic forgetting has focused mainly on linear regression models, leaving more complex settings largely unexplored. \cite{evron2022catastrophic} analyzed catastrophic forgetting under two task‑ordering schemes—cyclic and random—using alternating projections and the Kaczmarz method to pinpoint both the worst‑case and the no‑forgetting scenarios. Building on this, \cite{swartworth2023nearly} tightened nearly optimal forgetting bounds for cyclic orderings, and \cite{evron2025better} further improved the rates for random orderings with replacement. Additionally, \cite{goldfarb2023analysis} provided analysis that overparameterization accounts for most of the performance loss caused by catastrophic forgetting.  \cite{lin2023theory} examined how overparameterization, task similarity, and task ordering jointly influence both forgetting and generalization error in continual learning, and \cite{li2024theory} extended this analysis by characterizing the role of Mixture-of-Experts (MoE) architectures. \cite{ding2024understanding} developed a general theoretical framework for catastrophic forgetting under Stochastic Gradient Descent, revealing that the task order shapes the extent of forgetting in continual learning. \cite{zhao2024statistical} offered a statistical perspective on regularization‑based continual learning, showing how various regularizers affect model performance.

Beyond linear‑regression settings, several studies have investigated catastrophic forgetting in neural networks settings. \cite{doan2021theoretical} investigated catastrophic forgetting in the Neural Tangent Kernel (NTK) regime and showed that projected‑gradient algorithms can mitigate forgetting by introducing a task‑similarity measure called the NTK overlap matrix. 
\cite{cao2022provable} demonstrated that, for any target accuracy, one can keep the learned representation’s dimension nearly as small as the true underlying representation with the proposed CL algorithm. 
% Despite recent theoretical advances, catastrophic forgetting in continual learning remains far from fully understood. Specifically, (1) most analyses rely heavily on linear regression models and closed-form solutions, leaving more complex and realistic scenarios unexplored; (2) existing studies on neural networks predominantly examine simplified two-task scenarios, with extensions to general $M$-task cases still posing significant challenges.
The most relevant works to ours with data-replay strategies are \cite{banayeeanzade2024theoretical,zhengtowards}, where \cite{banayeeanzade2024theoretical} primarily focuses on the comparison between multi-task learning and continual learning, while \cite{zhengtowards} extends previous continual learning theory to memory-based methods. Both works are limited to the linear regression setting and leave the behavior of more complex models unexplored. 
% ; and (3) current theoretical works primarily focus on naive sequential training, neglecting a rigorous exploration of data-replay methods. 
We will provide more discussion in \cref{para:comparison}.

\section{Preliminaries}
\label{sec:preliminaries}

\begin{figure*}[t]
\centering
\begin{tikzpicture}
    % Include the first image
    \node[anchor=south west, inner sep=0] (img1) at (0, 0) {\includegraphics[width=0.22\textwidth]{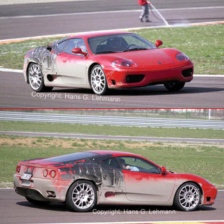}};
    \draw[green, thick] (1.6, 2.4) circle (0.3);
    \draw[green, thick] (1.5, 0.5) circle (0.3);
    \node[green] at (1.0, 1.3) {\textbf{\small car's wheel}};
    
    % Include the second image
    \node[anchor=south west, inner sep=0] (img2) at (3.5, 0) {\includegraphics[width=0.22\textwidth]{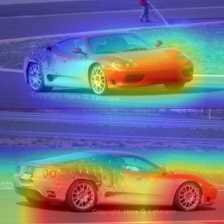}};
    \draw[green, thick] (5.2, 2.4) circle (0.3);
    \draw[green, thick] (4.8, 0.5) circle (0.3);
    \node[green] at (4.5, 1.3) {\textbf{\small car's wheel}};
    
    % Include the third image with annotations
    \node[anchor=south west, inner sep=0] (img3) at (7.0, 0) {\includegraphics[width=0.22\textwidth]{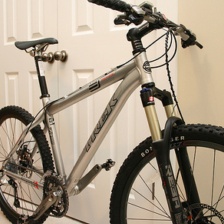}};
    \draw[red, thick] (7.6, 0.9) circle (0.7);
    \node[red] at (8.4, 2.0) {\textbf{\small  bicycle’s wheel}};
    
    % Include the fourth image
    \node[anchor=south west, inner sep=0] (img4) at (10.5, 0) {\includegraphics[width=0.22\textwidth]{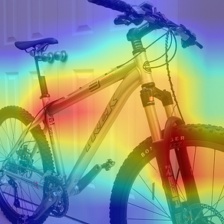}};
    \draw[red, thick] (11.1, 0.9) circle (0.7);
    \node[red] at (11.9, 2.0) {\textbf{\small  bicycle’s wheel}};
\end{tikzpicture}
\caption{Illustration of feature signals across multiple tasks using images from Salient ImageNet. }
\label{fig:image_features}
\end{figure*}
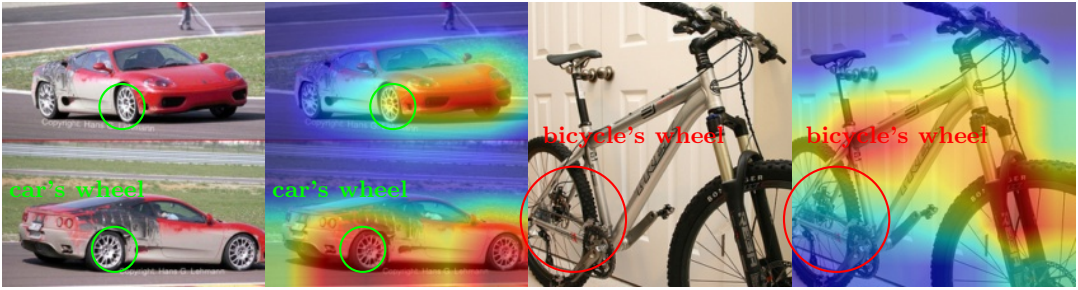

% \textbf{Notation.} In this paper, we adhere to a consistent notation
% style for clarity. We use bold lowercase letters (e.g., $\mathbf{x , w}$ ) for vectors. For positive functions $f$ and $g$, we write $f(x) \lesssim g(x)$ (or $f(x) \gtrsim g(x)$ ) when $f(x) \leq c g(x)$ (or $f(x) \geq c g(x)$ ) for some absolute constant $c>0$, and $f(x) \approx g(x)$ when both inequalities hold. Given a positive integer $n$, we denote $[n]=\{1, \ldots, n\}$; for any index set $\mathcal{S},|\mathcal{S}|$ is its cardinality. The Gaussian distribution with mean $\mu$ and covariance $\Sigma$ is $\mathcal{N}(\mu, \Sigma)$, and $\mathbf{I}_d$ denotes the $d \times$ $d$ identity matrix. We use $\widetilde{O}, \widetilde{\Theta}$, and $\widetilde{\Omega}$ to suppress polylogarithmic factors in $d$; the terms poly $(d)$ and polylog $(d)$ represent unspecified polynomial and polylogarithmic factors in $d$.

\textbf{Problem Setup.} In our setup, we consider a sequence of tasks denoted by $\mathbb{M}=\{1,2, \ldots, M\}$. For each task $m$ in this sequence, let $\{\mathbf{v}_m^*\}_{m \in[M]} \subseteq \mathbb{R}^d$ represent the feature vectors, where $\|\mathbf{v}_m^*\|=1$ for all $m \in[M]$, and $\langle\mathbf{v}_m^*, \mathbf{v}_{m^{\prime}}^*\rangle= A_{(m,m^{\prime})} \geq 0$ whenever $m \neq m^{\prime}$. Then, we define the data distributions for each task as follows. 

\begin{definition}[Data Distribution for Task $m$]\label{def:data_distribution}
For the task $m$, let $\mathbf{v}_m^* \in \mathbb{R}^d$ be a fixed vector representing the feature signal contained in each data point. Each data point $(\mathbf{x}_m,y_m)$ with input $\mathbf{x}_m = [\mathbf{x}_m^1, \mathbf{x}_m^2] \in (\mathbb{R}^d)^2$ and label $y \in \{+1, -1\}$ is generated from a data distribution $\mathcal{D}_m$ as follows:
\begin{itemize}
    \item[(1)] The label $y \in\{-1,1\}$ is sampled uniformly;
    \item[(2)] The input $\mathbf{x}_m$ is generated as a vector of $2$ patches, i.e., $\mathbf{x}_m = [\mathbf{x}_m^1, \mathbf{x}_m^2] \in (\mathbb{R}^d)^2$, where
    \begin{itemize}
        \item Feature patch. The first patch is given by $\mathbf{x}_m^1=\alpha_m y_m \cdot \mathbf{v}_m^*$, where {$\alpha_m>0$ indicates the signal intensity}. 
        \item  Noise patch. The second patch is given by $\mathbf{x}_m^2=$ $\bm{\xi}_m$, where $\bm{\xi}_m \sim \mathcal{N} (0, \sigma_{\xi}^2 \cdot \mathbf{H} )$ and is independent of the label $y_m$, where $\mathbf{H}=\mathbf{I}_d-\sum_{m=1}^M \mathbf{v}_m^* (\mathbf{v}_m^* )^{\top}$.
    \end{itemize}
\end{itemize}
\end{definition}

Our data generation model is inspired by the structure of image data, which has been widely utilized in the feature learning theory area \cite{allen2020towards,cao2022benign,jelassi2022towards,kou2023benign,zou2023benefits,ding2025understanding,han2024feature,li2024optimization,bu2024provably,bu2025provable,han2025role}. Specifically, the input data comprises two patches, among which only a subset is relevant to the class label of the image. We denote this relevant part as $y_m \alpha_m \mathbf{v}_m^*$, where $y_m$ represents the label, $\mathbf{v}_m^*$ is the corresponding feature signal vector, and $\alpha_m >0$ indicates the intensity of the feature signal. As described in Definition~\ref{def:data_distribution}, we assume that each task $m$ has its own unique feature signal vector and that the feature vectors across tasks are correlated with the correlation strength $A_{(m, m^{\prime})} >0$. For instance, in a continual learning setting where the model first classifies cars and later bicycles, the initial task may use the car’s wheel as a key feature and the subsequent task may use the bicycle’s wheel. Because both wheels share similar shapes, this overlap promotes feature reuse and helps the model recognize both objects as forms of transportation.
% When classifying cars, cats, and cameras, their corresponding feature vectors might represent characteristics such as wheels, fur, and lenses, respectively.
In contrast, the irrelevant patches, referred to as noise, are independent of the data label and do not contribute to prediction. We denote such noise as $\bm{\xi}$, which is assumed to follow a Gaussian distribution $\mathcal{N}(0, \sigma_{\xi}^2 \cdot \mathbf{H})$. For simplicity, the noise follows the same independent distribution for each task, and the noise vector is orthogonal to any feature signal vector $\mathbf{v}_m^*$.

\textbf{Learner Model.} 
Following existing work \cite{jelassi2022towards,baoprovable}, we consider a one-hidden-layer convolutional neural network architecture equipped with the cubic activation function $\sigma(z) = z^3$:
\begin{equation}
    \begin{aligned}
        F(\mathbf{W}, \mathbf{x}_m) &=\sum_{r \in[R]} \sigma(\langle\mathbf{w}_r, \mathbf{x}_m^1\rangle) + \sigma(\langle\mathbf{w}_r,
        \mathbf{x}_m^2\rangle) \\
        &=\sum_{r \in[R]} \sigma(\langle\mathbf{w}_r, \alpha_m y_m \mathbf{v}_m^*\rangle) + \sigma(\langle\mathbf{w}_r, \bm{\xi}_m\rangle),
    \end{aligned}
\end{equation} 
where $R$ is the number of hidden neurons and $\mathbf{W}=$ $\{\mathbf{w}_1, \ldots, \mathbf{w}_R\}$ represents the model weights. We denote the logistic loss function evaluated for the $m$-th task as
\begin{equation}
    L(\mathbf{W} ; {D}_m)=\frac{1}{n_m} \sum_{j \in[n_m]} \log \{1+e^{-y_{mj} F(\mathbf{W}, \mathbf{x}_{m j})}\}.
\end{equation}
Here, $D_m$ is the training data set for task $m$ with sample size $n_m$. To keep the analysis clean, we assume all tasks share the same sample size, i.e., $n_m=n$ for every $m$.
We train the model from a Gaussian initialization, drawing each hidden weight $\mathbf{w}_r^{(0)}$ independently from $\mathcal{N}\left(0, \sigma_0^2 \mathbf{I}_d\right)$.

\textbf{Data Replay Training.} Starting with the randomly initialized point $\mathbf{W}_0$ and employing a constant step size $\eta$, the model is updated by data-replay training for task $m$ over $T$ iterations, with $t =1,..,T$:
\begin{equation}\label{eq:data_replay}
    \mathbf{W}_m^{(t+1)}=\mathbf{W}_m^{(t)}-\frac{\eta}{m n_m} \sum_{j \in[n_m]} \nabla L(\mathbf{W}_m^{(t)} ; {D}_1, {D}_2, ...,{D}_m).
\end{equation}
Here, $\mathbf{W}_m^{(T)}$ denotes the parameter state after the completion of training on task $m$, which subsequently serves as the starting point for training on task $m+1$. 
In contrast to classical sequential training, the fully data-replay training incorporates all previous task datasets, $D_1, D_2,..., D_m$, into the training of the current task model. 

\textbf{Catastrophic Forgetting.} Catastrophic forgetting refers to the phenomenon where modern models substantially lose previously acquired knowledge when learning new tasks \cite{mccloskey1989catastrophic}. In the following, we provide a formal definition of this behavior in the context of continual learning over $M$ tasks.
\begin{definition}[Catastrophic Forgetting] \label{def:forgetting}
    Given a test data $(\mathbf{x}_k, y_k)$ drawn from the data distribution $\mathcal{D}_k$ of the $k$-th task, we claim \textit{Catastrophic Forgetting} occurs if the following conditions hold:
    \begin{itemize}
        \item[1.] After training on the $k$-th task (i.e., at iteration $T_k$ ), with high probability, the model correctly classifies the sample:
        $$
        \mathbb{P}\left\{y_k F\left(\mathbf{W}^{\left(T_k\right)}, \mathbf{x}_k\right)<0\right\} \leq \frac{1}{\operatorname{poly}(d)}.
        $$
        \item[2.] After training on the $m$-th task ( $m>k$, at iteration $T_m$ ), with high probability, the model's performance on task $k$ deteriorates:
        $$
        \mathbb{P}\left\{y_k F\left(\mathbf{W}^{\left(T_m\right)}, \mathbf{x}_k\right)<0\right\} \geq {\frac{1}{2}-\frac{1}{\operatorname{ polylog(d) }}} .
        $$
    \end{itemize}
\end{definition}

\section{Main Results}

In this section, we present our main results on the generalization performance for task $k$, evaluated after training on the $k$-th task and again after training on the $m$-th task ($m>k$) based on $\operatorname{SNR} = \alpha_p / \sigma_{\xi} \sqrt{d}$, respectively. Before stating the theorems, we first introduce the conditions that underlie our analysis.

\begin{condition}\label{con:parameter}
    For the data model described in Definition~\ref{def:data_distribution}, we assume that the noise standard deviation scales as $\sigma_{\xi} = \Theta(d^{-0.51})$. For the random initialization of the model weights, we assume $\sigma_0 = \Theta\left((n/R)^{1/3} d^{-0.52}\right)$. Furthermore, we assume the model is overparameterized, with both the hidden dimension $R$ and the sample size $n$ are bounded by $\operatorname{polylog}(d)$.
\end{condition}

Our conditions follow those in existing work \cite{jelassi2022towards,baoprovable}, but without imposing assumptions on the signal intensity. This relaxation allows us to explicitly investigate how the signal-to-noise ratio (SNR) influences the behavior of data replay training in continual learning.

% \subsection{Continual Forgetting of Task $k$ under Lower SNR}

\begin{theorem}\label{thm:forget_m_ffs_tasks}
    Suppose the setting in Condition~\ref{con:parameter} holds, and the SNR satisfies $ \frac{k^2}{R^{2/3} \sigma_0^2 \sigma_{\xi}^2 d^{13/6}} \lesssim \frac{\sum_{p=1}^{k} (1-\frac{p-1}{k}) \alpha_p^3 A_{(p,k)}}{(\sigma_{\xi}\sqrt{d})^3} \lesssim \frac{1}{n} $. Consider full data-replay training with learning rate $\eta \in(0, \widetilde{O}(1)]$, and let $(\mathbf{x}_k, y_k) \sim \mathcal{D}_k$ be a test sample from the task $k$. Then, with high probability, there exist training times $T_k$ and $T_m$ ($m>k$) such that 
    \begin{itemize}
        \item The model fails to correctly classify task $k$ immediately after learning it:
        \begin{equation}\label{eq:ff_k}
            \mathbb{P}\left\{y_k F\left(\mathbf{W}^{(T_k)}, \mathbf{x}_k\right)<0\right\} \geq \frac{1}{2}-{\frac{1}{\operatorname{polylog}(d)}}.
        \end{equation}
        \item (Persistent Learning Failure on Task $k$) If the additional SNR condition holds $\frac{m^2-k^2}{R^{1/3} \sigma_0 \sigma_{\xi} d} \lesssim \frac{\sum_{p=1}^{m} (1-\frac{p-1}{m}) \alpha_p^3 A_{(p,k)}}{(\sigma_{\xi}\sqrt{d})^3} \lesssim \frac{1}{n} $, then the model still fails to correctly classify task $k$ after subsequent training to task $m$: 
        \begin{equation}\label{eq:ff_m}
            \mathbb{P}\left\{y_k F\left(\mathbf{W}^{(T_m)}, \mathbf{x}_k\right)<0\right\} \geq \frac{1}{2}-{\frac{1}{\operatorname{polylog}(d)}}.
        \end{equation}
        \item (Enhanced Signal Learning on Task $k$) If the additional SNR conditions holds $\frac{\sum_{p=1}^{m}  \alpha_p^3 A_{(p,k)}}{(\sigma_{\xi}\sqrt{d})^3} \gtrsim \frac{1}{n R^{1/3} \sigma_0 \sigma_{\xi}\sqrt{d} } $, then the model can correctly classify task $k$ after subsequent training to task $m$:
        \begin{equation}\label{eq:fs_m}
            \mathbb{P}\left\{y_k F\left(\mathbf{W}^{(T_m)}, \mathbf{x}_k\right)<0\right\} \leq {\frac{1}{\operatorname{poly}(d)}}.
        \end{equation}
    \end{itemize}
\end{theorem}
% \textbf{Failure to Learn Task $k$ Under Weak Cumulative Signal.} 
Theorem~\ref{thm:forget_m_ffs_tasks} shows that if the cumulative signal from the first $k$ tasks related to task $k$ is not sufficiently strong, the model fails to correctly classify task $k$ even immediately after learning it, as shown in \cref{eq:ff_k}. This reflects poor generalization under low-SNR conditions and aligns with observations in standard (non-continual) learning settings \cite{cao2022benign}. Moreover, if the cumulative signal from the first $m$ tasks remains weak with respect to task $k$, the model continues to misclassify task $k$, indicating a persistent failure to learn its features. However, if the cumulative signal from the first $m$ tasks becomes sufficiently strong, the model can eventually classify task $k$ correctly--potentially even better than immediately after learning it-highlighting that learning subsequent tasks can help transfer useful features and improve generalization on earlier tasks. In addition, noticed that when analyzing learning failure, the SNR condition involves not only an upper bound but also a lower bound. This lower bound arises from the need to control the magnitude of noise memorization—even if effective signal learning does not occur. The model must still control the magnitude of noise memorization to ensure stable training, a principle that also holds in standard (non-continual) training settings \cite{cao2022benign}.

\textbf{Prioritizing Higher-Signal Tasks Facilitates Learning of Task $k$.} When evaluating the generalization performance for task $k$ under the SNR conditions, it can be observed that the cumulative signal depends on three key components: the coefficient $(1-\frac{p-1}{k})$, the signal intensity $\alpha_p^3$, and the correlation strength $A_{(p, k)}$. The coefficient reflects that tasks appearing earlier (i.e., smaller $p$ ) contribute more heavily to the accumulation of signal relevant to task $k$. The term $\alpha_p^3 A_{(p, k)}$ quantifies how much task $p$ contributes to the effective signal aligned with task $k$. Therefore, placing tasks with stronger signal intensity and higher alignment to task $k$ earlier in the sequence may help prevent persistent learning failure on task $k$, by boosting the overall cumulative signal in its favor.

\begin{theorem}\label{thm:forget_m_sfs_tasks}
    Suppose the setting in Condition~\ref{con:parameter} holds, and the SNR satisfies $ \frac{\sum_{p=1}^{k} \alpha_p^3 A_{(p,k)}}{(\sigma_{\xi}\sqrt{d})^3} \gtrsim \frac{1+n k^2/\sqrt{d}}{kn} $. Consider full data-replay training with learning rate $\eta \in(0, \widetilde{O}(1)]$, and let $(\mathbf{x}_k, y_k) \sim \mathcal{D}_k$ be a test sample from the task $k$. Then, with high probability, there exist training times $T_k$ and $T_m$ ($m>k$) such that 
    \begin{itemize}
        \item The model can correctly classify task $k$ immediately after learning it: 
        \begin{equation}\label{eq:k_s}
            \mathbb{P}\left\{y_k F\left(\mathbf{W}^{(T_k)}, \mathbf{x}_k\right)<0\right\} \leq {\frac{1}{\operatorname{poly}(d)}}. 
        \end{equation}
        \item (Catastrophic Forgetting on Task $k$) If the additional SNR conditions holds $ \frac{m^2}{R^{2/3} \sigma_0^2 \sigma_{\xi}^2 d^{13/6}} \lesssim \frac{\sum_{p=1}^{m} \alpha_p^3 A_{(p,k)}}{(\sigma_{\xi}\sqrt{d})^3} \lesssim \frac{\alpha_k R^{1/3}}{n} $, then it occurs \textit{Catastrophic Forgetting} on task $k$ after subsequent training to task $m$:
        \begin{equation}\label{eq:m_sf}
            \mathbb{P}\left\{y_k F\left(\mathbf{W}^{(T_m)}, \mathbf{x}_k\right)<0\right\} \geq \frac{1}{2}-{\frac{1}{\operatorname{polylog}(d)}}.
        \end{equation}
        \item (Continual Learning on Task $k$) If the additional SNR conditions holds $ \frac{\sum_{p=1}^{m} \alpha_p^3 A_{(p,k)}}{(\sigma_{\xi}\sqrt{d})^3} \gtrsim \frac{\alpha_k R^{1/3} \sigma_0 \left((1-\frac{k-1}{m}) + nm/\sqrt{d} \right)}{n} $, then the model can still correctly classify task $k$ after subsequent training to task $m$:
        \begin{equation}\label{eq:m_ss}
            \mathbb{P}\left\{y_k F\left(\mathbf{W}^{(T_m)}, \mathbf{x}_k\right)<0\right\} \leq {\frac{1}{\operatorname{poly}(d)}}.
        \end{equation}
    \end{itemize}
\end{theorem}
In contrast to Theorem~\ref{thm:forget_m_ffs_tasks}, Theorem~\ref{thm:forget_m_sfs_tasks} considers the case where the model successfully learns task $k$ after training on it, due to a sufficiently strong cumulative signal from the first $k$ tasks, as shown in \cref{eq:k_s}. This success may be maintained throughout continual learning if subsequent tasks continue to contribute meaningful signal toward task $k$ (see \cref{eq:m_ss}). However, if the cumulative signal from later tasks is insufficient or misaligned, the model may still experience forgetting of task $k$ despite its initial success—resulting in \textit{catastrophic forgetting} (refer to \cref{eq:m_sf}).

\textbf{Prioritizing Higher-Signal Tasks Mitigates Forgetting of Task $k$.} Similar to Theorem~\ref{thm:forget_m_ffs_tasks}, task ordering and signal intensity also play crucial roles in the subsequent learning and retention of task $k$. For instance, when evaluation occurs shortly after training task $k$ (i.e., when $m>k$ is close to $k$ ), a smaller amount of cumulative signal is required to satisfy the relaxed SNR condition in \cref{eq:fs_m}. Furthermore, placing tasks with stronger signal intensity and higher alignment to task $k$ between tasks $k$ and $m$ increases the cumulative signal, making it more likely to meet the continual learning condition and prevent \textit{catastrophic forgetting}. 

\label{para:comparison}\textbf{Comparison with Existing Work} Existing work shows that task ordering affects forgetting behavior from both empirical \cite{lesort2022challenging,hemati2025continual,li2025optimal} and analytical perspectives \cite{evron2022catastrophic,swartworth2023nearly,lin2023theory,ding2024understanding,evron2025better,li2025optimal}. Specifically, \cite{evron2022catastrophic} demonstrates that forgetting diminishes over time when task ordering is cyclic or random. \cite{swartworth2023nearly} and \cite{evron2025better} provide tighter forgetting bounds for cyclic and random orderings, respectively. \cite{lin2023theory}, \cite{ding2024understanding}, and \cite{li2025optimal} show that forgetting can be influenced by the arrangement of task orderings based on task similarity. Our work shares similar insights but from a novel feature signal perspective: prioritizing higher-signal tasks not only aids in learning lower-signal tasks but also mitigates forgetting. Moreover, prior analyses are primarily based on linear regression models, two-tasks settings, and naive sequential training, whereas our approach is grounded in a more general two-layer neural network model and a more challenging data replay training setup, making our work more applicable to realistic continual learning scenarios.

% Task Orderings and Challenge of Analyzing the CL

\section{Data Replay with $M$ Tasks }
\label{sec:proof_sketch}
In this section, we provide a proof sketch of the theoretical results introduced earlier. Our analysis focuses on understanding when and how a model trained via full data replay can either memorize noise or successfully learn meaningful features across multiple tasks. Before diving into the technical lemmas, we first establish the following notation:
% $$
% \begin{cases}
%     \text{The signal learning of task $k$'s feature at time $t$ under task $m$: } \Gamma_{(m, r)}^{(t,k)}:=\langle\mathbf{w}_{(m,r)}^{(t)}, \mathbf{v}_k^*\rangle, \\
%     \text{The noise memorization of sample $j$ from task $k$ at time $t$ under task $m$: }\Phi_{(m, r)}^{(t,k,j)}:=\langle\mathbf{w}_{(m,r)}^{(t)}, \bm{\xi}_k^j\rangle.
% \end{cases}
% $$
\begin{itemize}[leftmargin=*]
    \item The signal learning of task $k$'s feature at time $t$ under task $m$: $\Gamma_{(m, r)}^{(t,k)}:=\langle\mathbf{w}_{(m,r)}^{(t)}, \mathbf{v}_k^*\rangle.$
    \item The noise memorization of sample $j$ from task $k$ at time $t$ under task $m$: $\Phi_{(m, r)}^{(t,k,j)}:=\langle\mathbf{w}_{(m,r)}^{(t)}, \bm{\xi}_k^j\rangle.$
\end{itemize}
In \cref{sec:experiment}, we will illustrate the dynamics of signal learning and noise memorization during the continual training process under full data-replay.

\begin{lemma}[Continual Noise Memorization]\label{lem:noise_ff}
    Suppose the SNR condition satisfying
    $ \frac{k^2}{R^{2/3} \sigma_0^2 \sigma_{\xi}^2 d^{13/6}} \lesssim \frac{\sum_{p=1}^{k} (1-\frac{p-1}{k}) \alpha_p^3 A_{(p,k)}}{(\sigma_{\xi}\sqrt{d})^3} \lesssim \frac{1}{n} $, and there exists an iteration $\tau_{k j}^k \leq \mathrm{T}_{\xi}^k=\mathrm{T}_{\xi}^{-}+O(\log (d))$ such that $\tau_{k j}^k$ is the first iteration for which $\max _{r \in[R]}(y_{k j} \Phi_{(k, r)}^{(t,k,j)}) \geq \Theta(R^{-\frac{1}{3}})$, and for any $t \leq \mathrm{T}_{\xi}^k$ it holds that $\max _{r \in[R]}|\Gamma_{(k,r)}^{(t,k)}| \leq \widetilde{O}(\sigma_0)$.
    Then, if the additional SNR condition $\frac{m^2-k^2}{R^{1/3} \sigma_0 \sigma_{\xi} d} \lesssim \frac{\sum_{p=1}^{m} (1-\frac{p-1}{m}) \alpha_p^3 A_{(p,k)}}{(\sigma_{\xi}\sqrt{d})^3} \lesssim \frac{1}{n} $ also holds, there exists an iteration $\tau_{k j}^m$ such that $\tau_{k j}^m$ is the first iteration satisfying $\max _{r \in[R]}(y_{k j} \Phi_{(m, r)}^{(t,k,j)}) \geq \Theta(R^{-\frac{1}{4}})$. In this case, we can also guarantee that $\max _{r \in[R]}|\Gamma_{(m,r)}^{(t,k)}| \leq \widetilde{O}(\sigma_0)$ for any $t \leq \mathrm{T}_{\xi}^m$.
\end{lemma}

Lemma~\ref{lem:noise_ff} shows that the signal alignment for task $k$ remains bounded by $\widetilde{O}(\sigma_0)$, indicating that the model fails to learn sufficient features of task $k$ even by the end of its training. Instead, noise memorization dominates the learning process with a lower bound by $\Theta(R^{-\frac{1}{3}})$. This issue persists through subsequent training up to task $m$, suggesting that when the cumulative signal contribution from the first $m$ tasks is insufficient, the model consistently fails to learn task $k$. As a result, task $k$ suffers from continual learning failure and poor performance.

\begin{figure*}[t!]
    \centering
    \begin{subfigure}[b]{0.3\textwidth}
        \includegraphics[width=\textwidth]{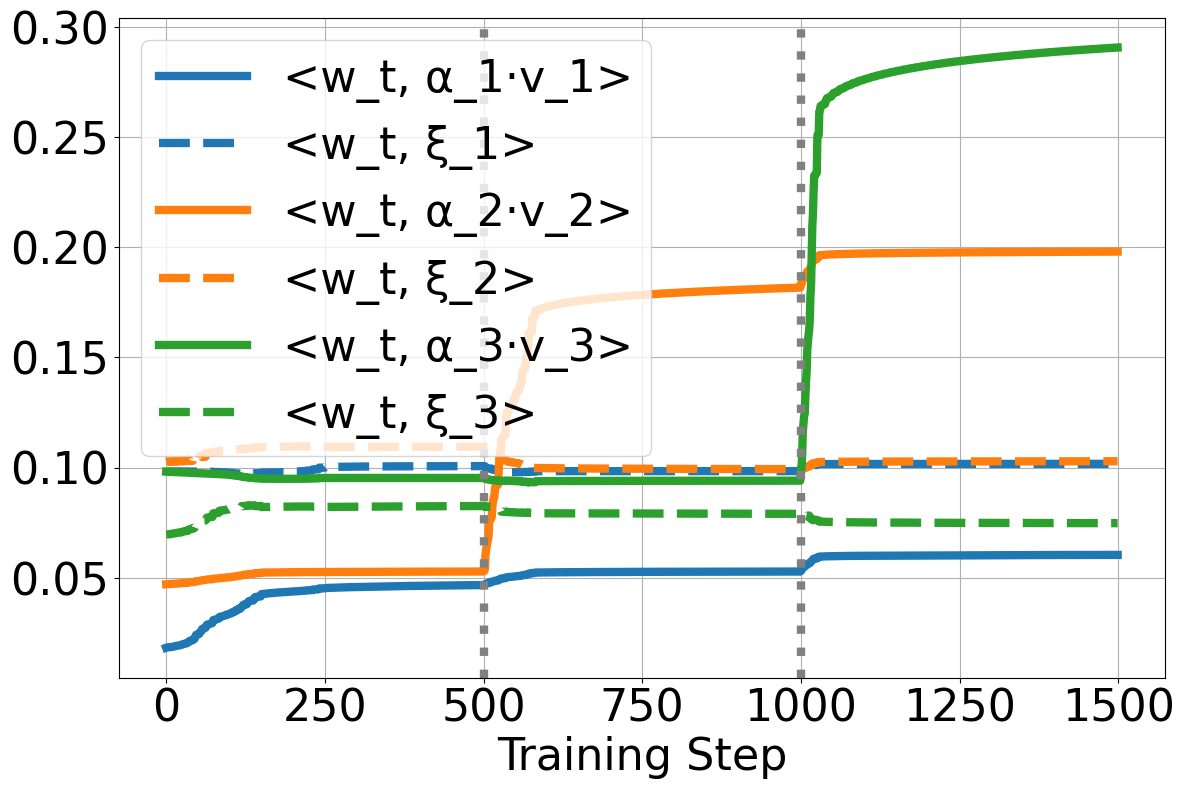}
        \caption{$A_{(m,m^{\prime})}$ = 0.1}
        \label{fig:align_01}
    \end{subfigure}
    \hfill
    \begin{subfigure}[b]{0.3\textwidth}
        \includegraphics[width=\textwidth]{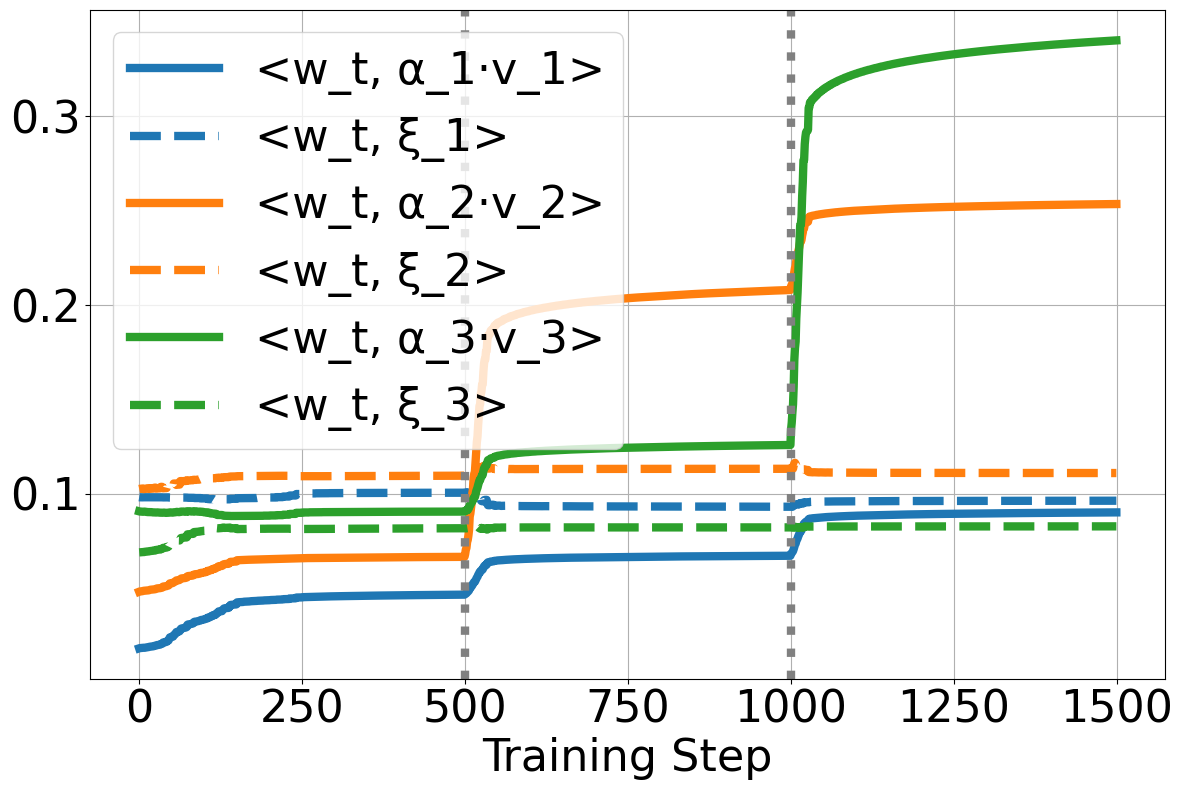}
        \caption{$A_{(m,m^{\prime})}$ = 0.3}
        \label{fig:align_03}
    \end{subfigure}
    \hfill
    \begin{subfigure}[b]{0.3\textwidth}
        \includegraphics[width=\textwidth]{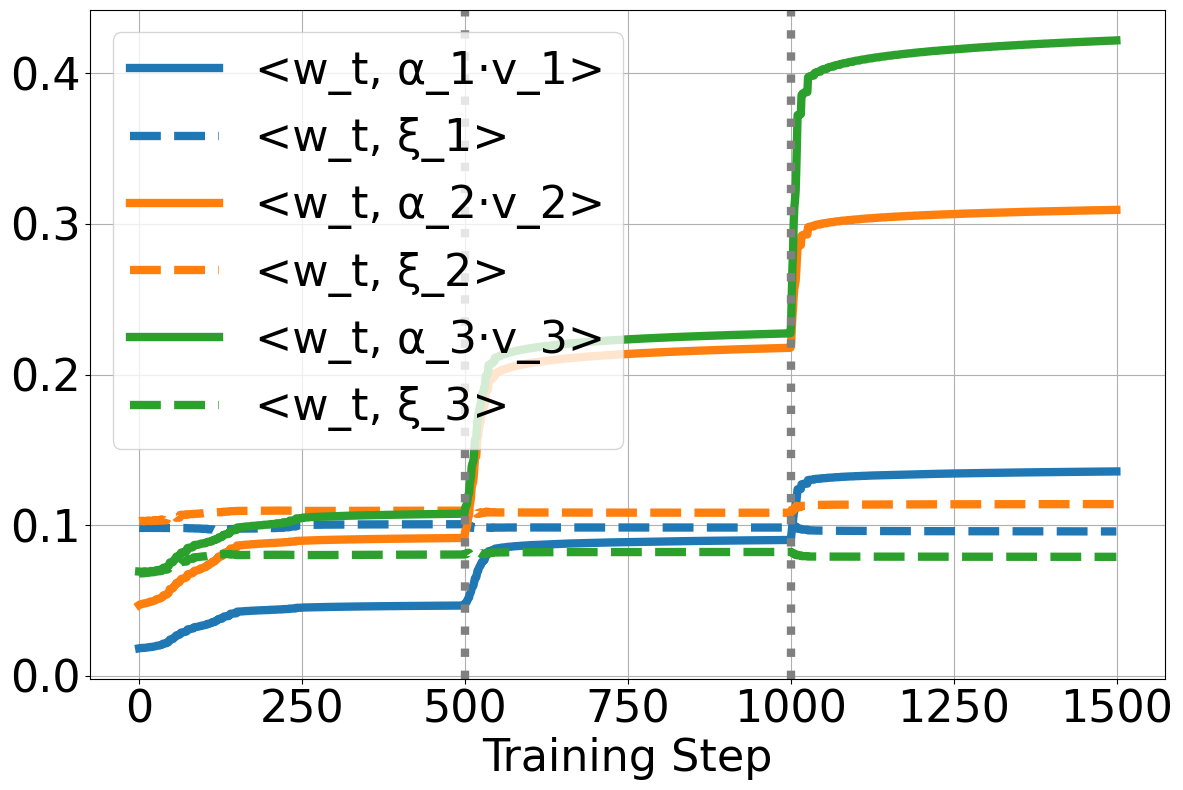}
        \caption{$A_{(m,m^{\prime})}$ = 0.7}
        \label{fig:align_07}
    \end{subfigure}

    \begin{subfigure}[b]{0.3\textwidth}
        \includegraphics[width=\textwidth]{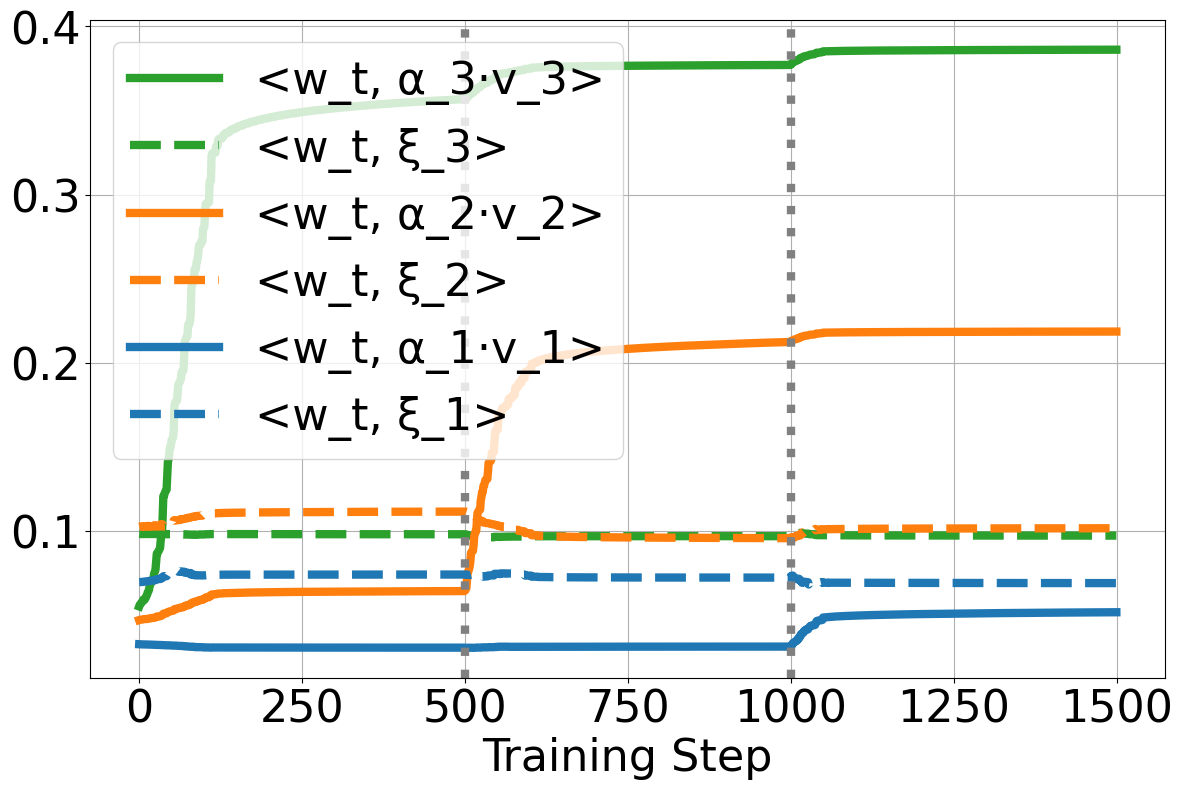}
        \caption{$A_{(m,m^{\prime})}$ = 0.1}
        \label{fig:align_01_order}
    \end{subfigure}
    \hfill
    \begin{subfigure}[b]{0.3\textwidth}
        \includegraphics[width=\textwidth]{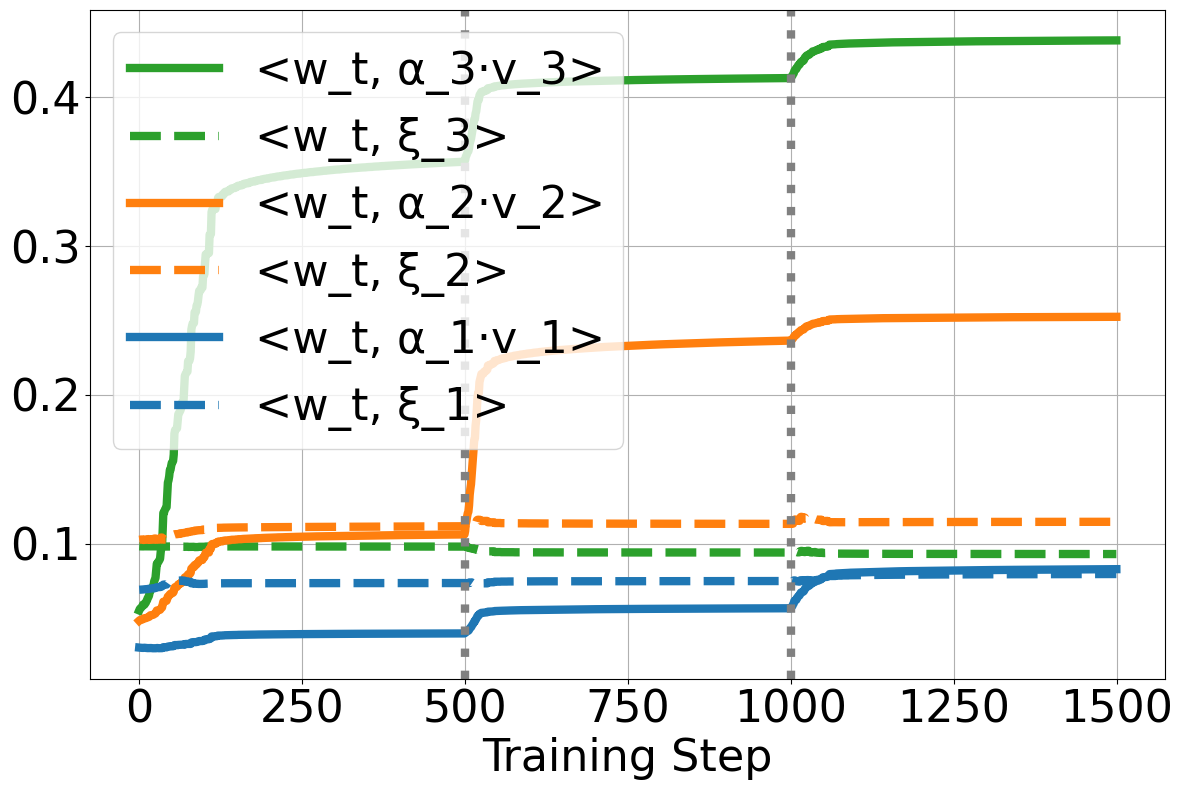}
        \caption{$A_{(m,m^{\prime})}$ = 0.3}
        \label{fig:align_03_order}
    \end{subfigure}
    \hfill
    \begin{subfigure}[b]{0.3\textwidth}
        \includegraphics[width=\textwidth]{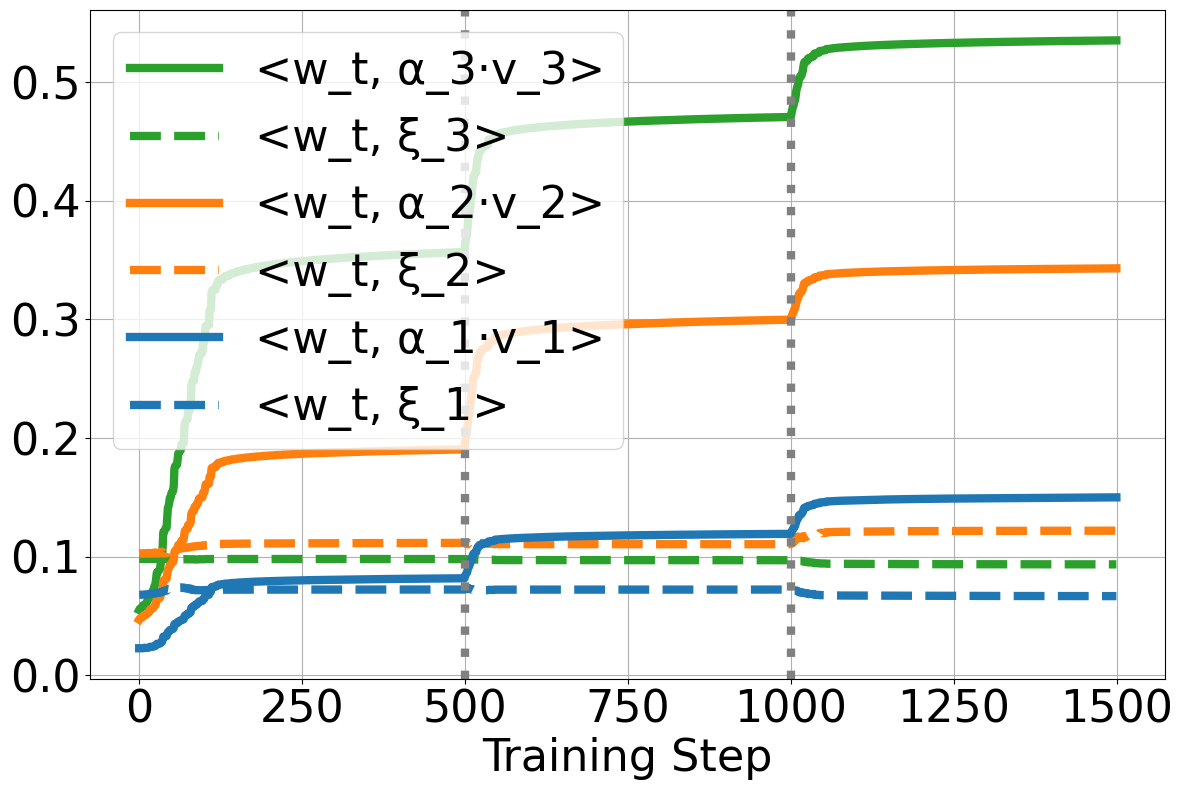}
        \caption{$A_{(m,m^{\prime})}$ = 0.7}
        \label{fig:align_07_order}
    \end{subfigure}
    \caption{Dynamics of signal learning and noise memorization during full data-replay continual training across different task orderings and correlation strengths. }
    \label{fig:dynamics}
    \vspace{-0.1in}
\end{figure*}
\begin{lemma}[Enhanced Signal Learning]\label{lem:noise_fs}
    Suppose the SNR satisfying $ \frac{k^2}{R^{2/3} \sigma_0^2 \sigma_{\xi}^2 d^{13/6}} \lesssim \frac{\sum_{p=1}^{k} (1-\frac{p-1}{k}) \alpha_p^3 A_{(p,k)}}{(\sigma_{\xi}\sqrt{d})^3} \lesssim \frac{1}{n} $, and there exists an iteration $\tau_{k j}^k \leq \mathrm{T}_{\xi}^k=\mathrm{T}_{\xi}^{-}+O(\log (d))$ such that $\tau_{k j}^k$ is the first iteration where $\max _{r \in[R]}(y_{k j} \Phi_{(k, r)}^{(t,k,j)}) \geq \Theta(R^{-\frac{1}{3}})$, and for any $t \leq \mathrm{T}_{\xi}^k$ it holds that $\max _{r \in[R]}|\Gamma_{(k,r)}^{(t,k)}| \leq \widetilde{O}(\sigma_0)$.
    Then, if the additional SNR condition $\frac{\sum_{p=1}^{m}  \alpha_p^3 A_{(p,k)}}{(\sigma_{\xi}\sqrt{d})^3} \gtrsim \frac{1}{n R^{1/3} \sigma_0 \sigma_{\xi}\sqrt{d} } $ also holds, there exists $\tau_{k v}^m \leq \mathrm{T}_{v}^m=\mathrm{T}_{v}^{k}+O(\log (d))$ such that  $\tau_{k v}^m$ be the first iteration satisfying $\max _{r \in[R]}|\Gamma_{(m,r)}^{(t,k)}| \geq \Theta(\frac{1}{\alpha_k R^{1/5}})$.  
\end{lemma}

% Similar to Lemma 1, Lemma 2 also indicates that before the end of training task $k$, the model can not accumulate feature signal of task $k$. However, after this phase, due to the strong coeffient strength of tasks $p \in (k,m)$ on the task k, they can contribute effective feature $k$ to facilite the mode signal learning process. Thus, after $T_v^m$, the model can correctly classify the sample from distribution of task k.

Similar to Lemma~\ref{lem:noise_ff}, Lemma~\ref{lem:noise_fs} shows that the model fails to learn task $k$ 's feature signal during its own training phase. However, in this case, tasks in later stages $p \in(k, m]$ possess strong alignment with task $k$, contributing sufficient signal to compensate for the earlier deficiency. This cumulative reinforcement enables the model to gradually build up the correct representation of task $k$, and by time $T_v^m$, it can successfully classify samples from task $k$ 's distribution. 

\begin{lemma}[Amplified Noise Memorization]\label{lem:signal_sf}
    Suppose the SNR satisfying $ \frac{\sum_{p=1}^{k} \alpha_p^3 A_{(p,k)}}{(\sigma_{\xi}\sqrt{d})^3} \gtrsim \frac{1+n k^2/\sqrt{d}}{kn} $, and there exists an iteration $\tau_{k v}^k \leq \mathrm{T}_{v}^k=\mathrm{T}_{v}^{-}+O(\log (d))$ such that $\tau_{k v}^k$ is the first iteration where $ \max _{r \in[R]}|\Gamma_{(k,r)}^{(t,k)}| \geq \Theta(\frac{1}{\alpha_k R^{1/3}})$, and for any $t \leq \mathrm{T}_{v}^k$ it holds that $\max _{r \in[R]}|\Phi_{(k, r)}^{(t,k,j)})| \leq \widetilde{O}(\sigma_0\sigma_{\xi}\sqrt{d})$.
    Then, if the additional SNR condition $ \frac{m^2}{R^{2/3} \sigma_0^2 \sigma_{\xi}^2 d^{13/6}} \lesssim \frac{\sum_{p=1}^{m} \alpha_p^3 A_{(p,k)}}{(\sigma_{\xi}\sqrt{d})^3} \lesssim \frac{\alpha_k R^{1/3}}{n}$ also holds, there exists $\tau_{k j}^m \leq \mathrm{T}_{\xi}^m=\mathrm{T}_{\xi}^{k}+O(\log (d))$ such that  $\tau_{k j}^m$ be the first iteration satisfying $\max _{r \in[R]}(y_{k j} \Phi_{(m, r)}^{(t,k,j)}) \geq \Theta(R^{-\frac{1}{5}})$. 
\end{lemma}

% In contrast to Lemma 1,2, Lemma 3 provide a new case of training task $k$, the model can already learn the sufficient feature of task $k$, but failed to preserve it because after the end of training task $k$, the second phase was dominated by the noise memorization, the accumulation of signal $k$ from task $k$ to $m$ is not enough. Thus that’s why catastrophic forgetting happened in Theorem 3. 

In contrast to Lemmas~\ref{lem:noise_ff} and~\ref{lem:noise_fs}, Lemma~\ref{lem:signal_sf} presents a case where the model initially succeeds in learning the feature of task $k$. However, this learned signal is not preserved-subsequent training phases are dominated by noise memorization, and the cumulative signal contribution from tasks $k$ to $m$ is insufficient to maintain the representation. As a result, the model gradually forgets task $k$, leading to catastrophic forgetting as characterized in Theorem~\ref{thm:forget_m_sfs_tasks}.

\begin{lemma}[Continual Signal Learning]\label{lem:signal_ss}
    Suppose the SNR satisfying $ \frac{\sum_{p=1}^{k} \alpha_p^3 A_{(p,k)}}{(\sigma_{\xi}\sqrt{d})^3} \gtrsim \frac{1+n k^2/\sqrt{d}}{kn} $, and there exists an iteration $\tau_{k v}^k \leq \mathrm{T}_{v}^k=\mathrm{T}_{v}^{-}+O(\log (d))$ such that $\tau_{k v}^k$ is the first iteration where $\max _{r \in[R]}|\Gamma_{(k,r)}^{(t,k)}| \geq \Theta(\frac{1}{\alpha_k R^{1/3}})$, and for any $t \leq \mathrm{T}_{v}^k$ it holds that $\max _{r \in[R]}|\Phi_{(k, r)}^{(t,k,j)})| \leq \widetilde{O}(\sigma_0\sigma_{\xi}\sqrt{d})$.
    Then, if the additional SNR condition $ \frac{\sum_{p=1}^{m} \alpha_p^3 A_{(p,k)}}{(\sigma_{\xi}\sqrt{d})^3} \gtrsim \frac{\alpha_k R^{1/3} \sigma_0 \left((1-\frac{k-1}{m}) + nm/\sqrt{d} \right)}{n}$ also holds, there exists $\tau_{k v}^m \leq \mathrm{T}_{v}^m=\mathrm{T}_{v}^{k}+O(\log (d))$ such that $\tau_{k v}^m$ be the first iteration satisfying $\max _{r \in[R]}|\Gamma_{(m,r)}^{(t,k)}| \geq \Theta(\frac{1}{\alpha_k R^{1/5}})$. 
\end{lemma}

To achieve successful continual learning of task $k$, the model must consistently prioritize signal learning over noise memorization-not only during the training of task $k$ but also throughout subsequent tasks up to task $m$. Lemma~\ref{lem:signal_ss} formalizes this by showing that the signal intensity aligned with task $k$ must remain above a certain threshold, while noise memorization must be kept under control. This balance ensures that the feature of task $k$ is both learned and retained over time.

% \kj{if space allowed, we could add a proof sketch here to highlight the key proof ideas. }
% \meng{yeah yeah definitely, i haven't finished this part, should be completed by tonight.}

\section{Experiment}
\label{sec:experiment}

In this section, we present synthetic experimental results to support our theoretical findings. Additional results are provided in the Appendix due to space limitations.

\textbf{Experimental Setup.} 
We design a synthetic continual learning experiment using a two-layer neural network with cubic activation. The model takes an input of dimension $2d$ (with $d = 1000$) and projects it to a hidden layer of size $R = 10$. The network is trained to solve three binary classification tasks sequentially, each associated with a distinct signal sampled from a multivariate Gaussian with varying correlation levels (off-diagonal entries set to $0.1$, $0.3$, and $0.7$ to represent low, medium, and high correlation). For each task $k$, the input is generated from \cref{def:data_distribution}, comprising signal and noise components. The signal strength $\alpha_k$ is scaled based on a task-specific SNR (set to $[0.1, 0.2, 0.3]$), and the noise is drawn from a distribution orthogonal to all signal directions, with fixed deviation $\sigma_\xi = 0.1$.
Training is performed using SGD with a fixed learning rate $\eta = 0.1$ and Gaussian initialization ($\sigma_0 = 0.1$). Each task is trained for 50 epochs with 10 samples. To assess learning dynamics, we track the alignment between hidden weights and both signal and noise across tasks. Notably, the dynamics of signal learning and noise memorization are closely consistent with accuracy performance—stronger signal learning generally corresponds to higher accuracy. Due to space limitations, we present the detailed accuracy figures in the Appendix.

\textbf{Prioritizing Higher-Signal Tasks May Enhance Lower-Signal Tasks Learning.} 
Figures~\ref{fig:dynamics} shows the dynamics of signal learning and noise memorization during continual training under full data replay, comparing different task orderings across varying levels of task correlation. 
In Figures~\ref{fig:align_01}-\ref{fig:align_07}, Task 3—which has the highest signal intensity (corresponding to the highest SNR = 0.3 under fixed noise scale)—is placed earlier in the task sequence. In contrast, Figures~\ref{fig:align_01_order}-\ref{fig:align_07_order} reverse the task order, placing lower-SNR tasks earlier.
When the correlation strength is low ($A_{(m,m^{\prime})} =0.1$, implying near-orthogonality between task vectors and low task similarity), prioritizing the high-signal Task 3 has limited effect: the cumulative signal for the lower-signal Task 1 remains insufficient in both orderings (see Figures~\ref{fig:align_01} and~\ref{fig:align_01_order}). However, as correlation strength increases, the effect of task ordering becomes more pronounced. For instance, in the moderate correlation setting (Figures~\ref{fig:align_03} and~\ref{fig:align_03_order}), prioritizing Task 3 improves signal acquisition for the other tasks—Task 2 achieves higher signal learning in the ordered setting. Furthermore, in Figure~\ref{fig:align_03_order}, the signal learning of Task 1 eventually exceeds its noise memorization, while in the non-prioritized setting (Figure~\ref{fig:align_03}), Task 1 continues to struggle. This effect becomes even more evident under high correlation ($A_{(m,m^{\prime})} =0.7$), where prioritizing high-signal tasks yields better signal learning for lower-SNR tasks, as shown in Figure~\ref{fig:align_07_order}. These empirical observations also validate our theoretical conclusions in Theorem~\ref{thm:forget_m_ffs_tasks} and 
\ref{thm:forget_m_sfs_tasks}.

\textbf{Higher Correlation Enhances Signal Learning.} 
Figures~\ref{fig:align_01}, \ref{fig:align_03}, and \ref{fig:align_07} (and their reordered counterparts) illustrate that increasing the correlation between tasks significantly 
improves signal learning across the board. When the correlation strength is low ($A_{(m,m^{\prime})} =0.1$), tasks contribute little to one another, resulting in limited signal accumulation for earlier, lower-SNR tasks—regardless of ordering. However, as the correlation increases to $0.3$ and $0.7$, tasks—especially those with stronger signals—can contribute more effectively to the overall feature representation, improving the learning of other tasks in the sequence. For example, under the high-correlation setting ($A_{(m,m^{\prime})} =0.7$), even lower-signal tasks (e.g., Task 1) can accumulate sufficient signal to surpass noise memorization, demonstrating that strong task correlation amplifies the benefits of both task ordering and feature sharing in continual learning.

\begin{figure}[t]   
    \centering
    \includegraphics[width=0.7\columnwidth]{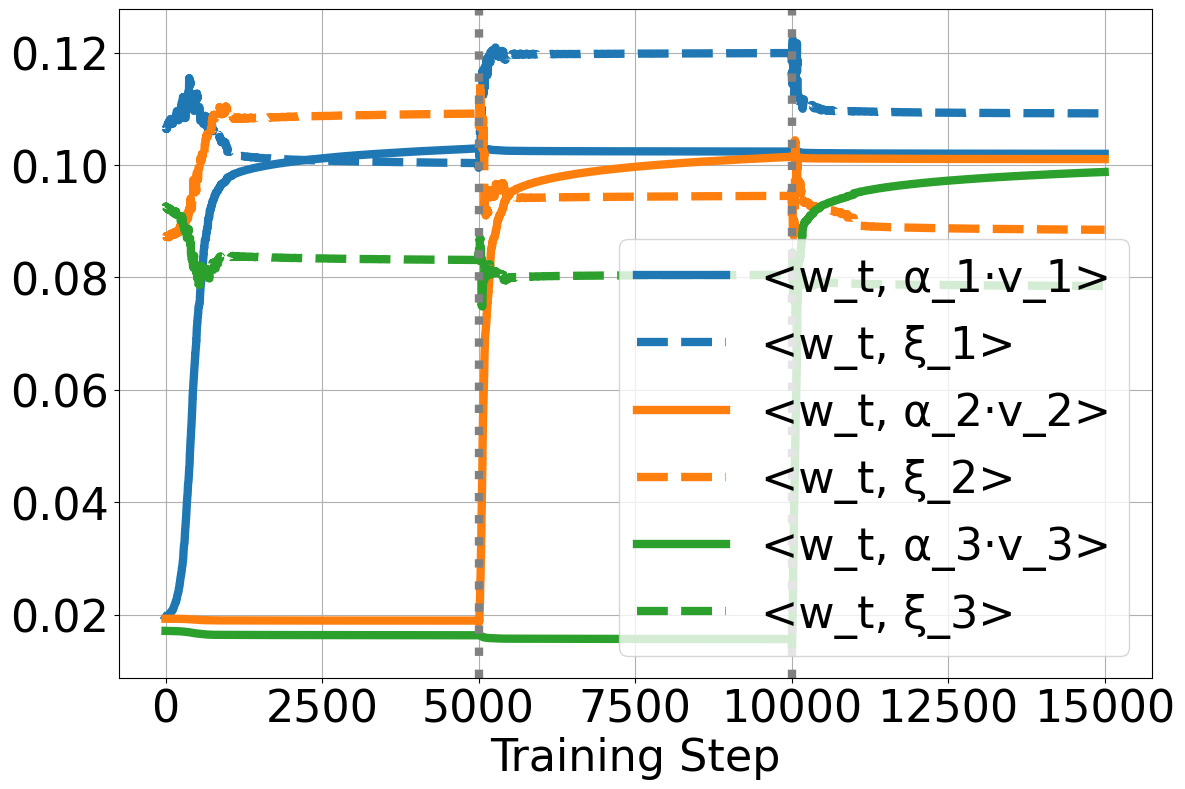}
    \caption{Dynamics of signal learning and noise memorization under lower SNR.}
    \label{fig:dynamics_noise}
    \vspace{-0.1in}
\end{figure}

\textbf{Competition between Noise Memorization and Signal Learning.} In Figure~\ref{fig:dynamics}, it is clear that noise memorization remains relatively stable, which may be attributed to the model focusing more on signal learning during training. To further investigate the behavior of noise memorization, we increase the sample size to $100$, reduce the signal intensity to $0.06$ for all tasks, and set the correlation strength to $0.01$ to simulate a low-correlation regime. As shown in Figure~\ref{fig:dynamics_noise}, Task $1$ performs well during its initial training phase, as the signal learning surpasses noise memorization. However, as new tasks are introduced—each weakly correlated with Task $1$—the model fails to reinforce Task $1$'s features, ultimately leading to catastrophic forgetting of Task $1$. We further explore the impact of correlation by increasing the correlation strength to 0.3 and 0.7. As expected, higher correlation allows the model to benefit from the features learned in Tasks 2 and 3, effectively contributing to Task 1’s signal and mitigating forgetting. These results demonstrate that \textit{catastrophic forgetting tends to occur when tasks are orthogonal}, consistent with Theorem~\ref{thm:forget_m_sfs_tasks}, where the SNR conditions fail to hold due to near-zero correlation $A_{(m,m^{\prime})}=0$.
Due to space limitations, the corresponding figures are deferred to the Appendix.
\vspace{-0.1in}
\section{Conclusion}
\vspace{-0.1in}
In this work, we provide a comprehensive theoretical framework for understanding full data-replay training in continual learning through the lens of feature learning. By adopting a multi-view data model, task-specific signal structures and inter-task correlations, we identify the SNR as a fundamental factor driving forgetting. 
% Our theoretical results reveal that full data replay does not universally prevent forgetting: it fails when cumulative noise dominates the signal, but it can also enhance recovery when later tasks contribute positively to earlier ones. 
A particularly novel insight from our study is the impact of task ordering—prioritizing higher-signal tasks not only improves learning for subsequent tasks but also mitigates forgetting of earlier ones. This highlights the need for order-aware replay strategies in the design of continual learning systems.

% Finally, our synthetic experiments validate the theoretical predictions by visualizing the dynamics of signal and noise across different replay sequences, correlations, and SNR regimes, offering practical evidence for our findings and guiding future empirical and algorithmic developments.

\section*{Acknowledgment}
We thank the AISTATS reviewers and community for their valuable suggestions, which motivated us to conduct and include additional empirical verification on real-world CIFAR-100 data in Appendix. The research of Jinhui Xu was partially supported by startup funds from USTC and a grant from IAI.
% \newpage

\bibliography{cl_ref}

@article{parisi2019continual,
  title={Continual lifelong learning with neural networks: A review},
  author={Parisi, German I and Kemker, Ronald and Part, Jose L and Kanan, Christopher and Wermter, Stefan},
  journal={Neural networks},
  volume={113},
  pages={54--71},
  year={2019},
  publisher={Elsevier}
}

@inproceedings{he2016deep,
  title={Deep residual learning for image recognition},
  author={He, Kaiming and Zhang, Xiangyu and Ren, Shaoqing and Sun, Jian},
  booktitle={Proceedings of the IEEE conference on computer vision and pattern recognition},
  pages={770--778},
  year={2016}
}

@article{krizhevsky2009learning,
  title={Learning multiple layers of features from tiny images},
  author={Krizhevsky, Alex and Hinton, Geoffrey and others},
  year={2009},
  publisher={Toronto, ON, Canada}
}

@inproceedings{ding2024understanding,
  title={Understanding Forgetting in Continual Learning with Linear Regression},
  author={Ding, Meng and Ji, Kaiyi and Wang, Di and Xu, Jinhui},
  booktitle={Forty-first International Conference on Machine Learning},
year = {2024}
}

@article{ding2025understanding,
  title={Understanding Private Learning From Feature Perspective},
  author={Ding, Meng and Lei, Mingxi and Fu, Shaopeng and Wang, Shaowei and Wang, Di and Xu, Jinhui},
  journal={arXiv preprint arXiv:2511.18006},
  year={2025}
}

@article{cao2022benign,
  title={Benign overfitting in two-layer convolutional neural networks},
  author={Cao, Yuan and Chen, Zixiang and Belkin, Misha and Gu, Quanquan},
  journal={Advances in neural information processing systems},
  volume={35},
  pages={25237--25250},
  year={2022}
}

@inproceedings{kou2023benign,
  title={Benign overfitting in two-layer ReLU convolutional neural networks},
  author={Kou, Yiwen and Chen, Zixiang and Chen, Yuanzhou and Gu, Quanquan},
  booktitle={International Conference on Machine Learning},
  pages={17615--17659},
  year={2023},
  organization={PMLR}
}

@inproceedings{jelassi2022towards,
  title={Towards understanding how momentum improves generalization in deep learning},
  author={Jelassi, Samy and Li, Yuanzhi},
  booktitle={International Conference on Machine Learning},
  pages={9965--10040},
  year={2022},
  organization={PMLR}
}

@inproceedings{zou2023benefits,
  title={The benefits of mixup for feature learning},
  author={Zou, Difan and Cao, Yuan and Li, Yuanzhi and Gu, Quanquan},
  booktitle={International Conference on Machine Learning},
  pages={43423--43479},
  year={2023},
  organization={PMLR}
}

@article{li2024optimization,
  title={On the optimization and generalization of two-layer transformers with sign gradient descent},
  author={Li, Bingrui and Huang, Wei and Han, Andi and Zhou, Zhanpeng and Suzuki, Taiji and Zhu, Jun and Chen, Jianfei},
  journal={arXiv preprint arXiv:2410.04870},
  year={2024}
}

@inproceedings{bu2025provable,
  title={Provable In-Context Vector Arithmetic via Retrieving Task Concepts},
  author={Bu, Dake and Huang, Wei and Han, Andi and Nitanda, Atsushi and Zhang, Qingfu and Wong, Hau-San and Suzuki, Taiji},
  booktitle={Forty-second International Conference on Machine Learning},
  year={2025}
}

@article{han2025role,
  title={On the Role of Label Noise in the Feature Learning Process},
  author={Han, Andi and Huang, Wei and Zhou, Zhanpeng and Niu, Gang and Chen, Wuyang and Yan, Junchi and Takeda, Akiko and Suzuki, Taiji},
  journal={arXiv preprint arXiv:2505.18909},
  year={2025}
}

@article{allen2020towards,
  title={Towards understanding ensemble, knowledge distillation and self-distillation in deep learning},
  author={Allen-Zhu, Zeyuan and Li, Yuanzhi},
  journal={arXiv preprint arXiv:2012.09816},
  year={2020}
}

@inproceedings{baoprovable,
  title={Provable Benefits of Local Steps in Heterogeneous Federated Learning for Neural Networks: A Feature Learning Perspective},
  author={Bao, Yajie and Crawshaw, Michael and Liu, Mingrui},
  booktitle={Forty-first International Conference on Machine Learning}
}

@incollection{mccloskey1989catastrophic,
  title={Catastrophic interference in connectionist networks: The sequential learning problem},
  author={McCloskey, Michael and Cohen, Neal J},
  booktitle={Psychology of learning and motivation},
  volume={24},
  pages={109--165},
  year={1989},
  publisher={Elsevier}
}

@article{wang2024comprehensive,
  title={A comprehensive survey of continual learning: Theory, method and application},
  author={Wang, Liyuan and Zhang, Xingxing and Su, Hang and Zhu, Jun},
  journal={IEEE Transactions on Pattern Analysis and Machine Intelligence},
  year={2024},
  publisher={IEEE}
}

@article{kirkpatrick2017overcoming,
  title={Overcoming catastrophic forgetting in neural networks},
  author={Kirkpatrick, James and Pascanu, Razvan and Rabinowitz, Neil and Veness, Joel and Desjardins, Guillaume and Rusu, Andrei A and Milan, Kieran and Quan, John and Ramalho, Tiago and Grabska-Barwinska, Agnieszka and others},
  journal={Proceedings of the national academy of sciences},
  volume={114},
  number={13},
  pages={3521--3526},
  year={2017},
  publisher={National Academy of Sciences}
}

@inproceedings{evron2022catastrophic,
  title={How catastrophic can catastrophic forgetting be in linear regression?},
  author={Evron, Itay and Moroshko, Edward and Ward, Rachel and Srebro, Nathan and Soudry, Daniel},
  booktitle={Conference on Learning Theory},
  pages={4028--4079},
  year={2022},
  organization={PMLR}
}

@inproceedings{goldfarb2023analysis,
  title={Analysis of catastrophic forgetting for random orthogonal transformation tasks in the overparameterized regime},
  author={Goldfarb, Daniel and Hand, Paul},
  booktitle={International Conference on Artificial Intelligence and Statistics},
  pages={2975--2993},
  year={2023},
  organization={PMLR}
}

@inproceedings{lin2023theory,
  title={Theory on forgetting and generalization of continual learning},
  author={Lin, Sen and Ju, Peizhong and Liang, Yingbin and Shroff, Ness},
  booktitle={International Conference on Machine Learning},
  pages={21078--21100},
  year={2023},
  organization={PMLR}
}

@article{swartworth2023nearly,
  title={Nearly optimal bounds for cyclic forgetting},
  author={Swartworth, William and Needell, Deanna and Ward, Rachel and Kong, Mark and Jeong, Halyun},
  journal={Advances in Neural Information Processing Systems},
  volume={36},
  pages={68197--68206},
  year={2023}
}

@article{evron2025better,
  title={Better Rates for Random Task Orderings in Continual Linear Models},
  author={Evron, Itay and Levinstein, Ran and Schliserman, Matan and Sherman, Uri and Koren, Tomer and Soudry, Daniel and Srebro, Nathan},
  journal={arXiv preprint arXiv:2504.04579},
  year={2025}
}

@article{li2024theory,
  title={Theory on mixture-of-experts in continual learning},
  author={Li, Hongbo and Lin, Sen and Duan, Lingjie and Liang, Yingbin and Shroff, Ness B},
  journal={arXiv preprint arXiv:2406.16437},
  year={2024}
}

@article{zhao2024statistical,
  title={A statistical theory of regularization-based continual learning},
  author={Zhao, Xuyang and Wang, Huiyuan and Huang, Weiran and Lin, Wei},
  journal={arXiv preprint arXiv:2406.06213},
  year={2024}
}

@inproceedings{doan2021theoretical,
  title={A theoretical analysis of catastrophic forgetting through the ntk overlap matrix},
  author={Doan, Thang and Bennani, Mehdi Abbana and Mazoure, Bogdan and Rabusseau, Guillaume and Alquier, Pierre},
  booktitle={International Conference on Artificial Intelligence and Statistics},
  pages={1072--1080},
  year={2021},
  organization={PMLR}
}

@inproceedings{cao2022provable,
  title={Provable lifelong learning of representations},
  author={Cao, Xinyuan and Liu, Weiyang and Vempala, Santosh},
  booktitle={International Conference on Artificial Intelligence and Statistics},
  pages={6334--6356},
  year={2022},
  organization={PMLR}
}

@article{chaudhry2019tiny,
  title={On tiny episodic memories in continual learning},
  author={Chaudhry, Arslan and Rohrbach, Marcus and Elhoseiny, Mohamed and Ajanthan, Thalaiyasingam and Dokania, Puneet K and Torr, Philip HS and Ranzato, Marc'Aurelio},
  journal={arXiv preprint arXiv:1902.10486},
  year={2019}
}

@article{riemer2018learning,
  title={Learning to learn without forgetting by maximizing transfer and minimizing interference},
  author={Riemer, Matthew and Cases, Ignacio and Ajemian, Robert and Liu, Miao and Rish, Irina and Tu, Yuhai and Tesauro, Gerald},
  journal={arXiv preprint arXiv:1810.11910},
  year={2018}
}

@article{lopez2017gradient,
  title={Gradient episodic memory for continual learning},
  author={Lopez-Paz, David and Ranzato, Marc'Aurelio},
  journal={Advances in neural information processing systems},
  volume={30},
  year={2017}
}

@article{yoon2021online,
  title={Online coreset selection for rehearsal-based continual learning},
  author={Yoon, Jaehong and Madaan, Divyam and Yang, Eunho and Hwang, Sung Ju},
  journal={arXiv preprint arXiv:2106.01085},
  year={2021}
}

@inproceedings{shim2021online,
  title={Online class-incremental continual learning with adversarial shapley value},
  author={Shim, Dongsub and Mai, Zheda and Jeong, Jihwan and Sanner, Scott and Kim, Hyunwoo and Jang, Jongseong},
  booktitle={Proceedings of the AAAI Conference on Artificial Intelligence},
  volume={35},
  number={11},
  pages={9630--9638},
  year={2021}
}

@inproceedings{tiwari2022gcr,
  title={Gcr: Gradient coreset based replay buffer selection for continual learning},
  author={Tiwari, Rishabh and Killamsetty, Krishnateja and Iyer, Rishabh and Shenoy, Pradeep},
  booktitle={Proceedings of the IEEE/CVF Conference on Computer Vision and Pattern Recognition},
  pages={99--108},
  year={2022}
}

@article{kumari2022retrospective,
  title={Retrospective adversarial replay for continual learning},
  author={Kumari, Lilly and Wang, Shengjie and Zhou, Tianyi and Bilmes, Jeff A},
  journal={Advances in neural information processing systems},
  volume={35},
  pages={28530--28544},
  year={2022}
}

@article{ebrahimi2021remembering,
  title={Remembering for the right reasons: Explanations reduce catastrophic forgetting},
  author={Ebrahimi, Sayna and Petryk, Suzanne and Gokul, Akash and Gan, William and Gonzalez, Joseph E and Rohrbach, Marcus and Darrell, Trevor},
  journal={Applied AI letters},
  volume={2},
  number={4},
  pages={e44},
  year={2021},
  publisher={Wiley Online Library}
}

@article{wang2021triple,
  title={Triple-memory networks: A brain-inspired method for continual learning},
  author={Wang, Liyuan and Lei, Bo and Li, Qian and Su, Hang and Zhu, Jun and Zhong, Yi},
  journal={IEEE Transactions on Neural Networks and Learning Systems},
  volume={33},
  number={5},
  pages={1925--1934},
  year={2021},
  publisher={IEEE}
}

@article{nguyen2017variational,
  title={Variational continual learning},
  author={Nguyen, Cuong V and Li, Yingzhen and Bui, Thang D and Turner, Richard E},
  journal={arXiv preprint arXiv:1710.10628},
  year={2017}
}

@article{cong2020gan,
  title={Gan memory with no forgetting},
  author={Cong, Yulai and Zhao, Miaoyun and Li, Jianqiao and Wang, Sijia and Carin, Lawrence},
  journal={Advances in neural information processing systems},
  volume={33},
  pages={16481--16494},
  year={2020}
}

@inproceedings{ostapenko2019learning,
  title={Learning to remember: A synaptic plasticity driven framework for continual learning},
  author={Ostapenko, Oleksiy and Puscas, Mihai and Klein, Tassilo and Jahnichen, Patrick and Nabi, Moin},
  booktitle={Proceedings of the IEEE/CVF conference on computer vision and pattern recognition},
  pages={11321--11329},
  year={2019}
}

@article{van2020brain,
  title={Brain-inspired replay for continual learning with artificial neural networks},
  author={Van de Ven, Gido M and Siegelmann, Hava T and Tolias, Andreas S},
  journal={Nature communications},
  volume={11},
  number={1},
  pages={4069},
  year={2020},
  publisher={Nature Publishing Group UK London}
}

@inproceedings{liu2020generative,
  title={Generative feature replay for class-incremental learning},
  author={Liu, Xialei and Wu, Chenshen and Menta, Mikel and Herranz, Luis and Raducanu, Bogdan and Bagdanov, Andrew D and Jui, Shangling and de Weijer, Joost van},
  booktitle={Proceedings of the IEEE/CVF conference on computer vision and pattern recognition workshops},
  pages={226--227},
  year={2020}
}

@inproceedings{korbak2022controlling,
  title={Controlling conditional language models without catastrophic forgetting},
  author={Korbak, Tomasz and Elsahar, Hady and Kruszewski, German and Dymetman, Marc},
  booktitle={International Conference on Machine Learning},
  pages={11499--11528},
  year={2022},
  organization={PMLR}
}

@article{ritter2018online,
  title={Online structured laplace approximations for overcoming catastrophic forgetting},
  author={Ritter, Hippolyt and Botev, Aleksandar and Barber, David},
  journal={Advances in Neural Information Processing Systems},
  volume={31},
  year={2018}
}

@inproceedings{aljundi2018memory,
  title={Memory aware synapses: Learning what (not) to forget},
  author={Aljundi, Rahaf and Babiloni, Francesca and Elhoseiny, Mohamed and Rohrbach, Marcus and Tuytelaars, Tinne},
  booktitle={Proceedings of the European conference on computer vision (ECCV)},
  pages={139--154},
  year={2018}
}

@article{titsias2019functional,
  title={Functional regularisation for continual learning with gaussian processes},
  author={Titsias, Michalis K and Schwarz, Jonathan and Matthews, Alexander G de G and Pascanu, Razvan and Teh, Yee Whye},
  journal={arXiv preprint arXiv:1901.11356},
  year={2019}
}

@article{pan2020continual,
  title={Continual deep learning by functional regularisation of memorable past},
  author={Pan, Pingbo and Swaroop, Siddharth and Immer, Alexander and Eschenhagen, Runa and Turner, Richard and Khan, Mohammad Emtiyaz E},
  journal={Advances in neural information processing systems},
  volume={33},
  pages={4453--4464},
  year={2020}
}

@inproceedings{benzing2022unifying,
  title={Unifying importance based regularisation methods for continual learning},
  author={Benzing, Frederik},
  booktitle={International Conference on Artificial Intelligence and Statistics},
  pages={2372--2396},
  year={2022},
  organization={PMLR}
}

@inproceedings{lin2022towards,
  title={Towards better plasticity-stability trade-off in incremental learning: A simple linear connector},
  author={Lin, Guoliang and Chu, Hanlu and Lai, Hanjiang},
  booktitle={Proceedings of the IEEE/CVF conference on computer vision and pattern recognition},
  pages={89--98},
  year={2022}
}

@article{chaudhry2018efficient,
  title={Efficient lifelong learning with a-gem},
  author={Chaudhry, Arslan and Ranzato, Marc'Aurelio and Rohrbach, Marcus and Elhoseiny, Mohamed},
  journal={arXiv preprint arXiv:1812.00420},
  year={2018}
}

@inproceedings{tang2021layerwise,
  title={Layerwise optimization by gradient decomposition for continual learning},
  author={Tang, Shixiang and Chen, Dapeng and Zhu, Jinguo and Yu, Shijie and Ouyang, Wanli},
  booktitle={Proceedings of the IEEE/CVF conference on Computer Vision and Pattern Recognition},
  pages={9634--9643},
  year={2021}
}

@inproceedings{wang2022anti,
  title={Anti-retroactive interference for lifelong learning},
  author={Wang, Runqi and Bao, Yuxiang and Zhang, Baochang and Liu, Jianzhuang and Zhu, Wentao and Guo, Guodong},
  booktitle={European Conference on Computer Vision},
  pages={163--178},
  year={2022},
  organization={Springer}
}

@inproceedings{wu2022class,
  title={Class-incremental learning with strong pre-trained models},
  author={Wu, Tz-Ying and Swaminathan, Gurumurthy and Li, Zhizhong and Ravichandran, Avinash and Vasconcelos, Nuno and Bhotika, Rahul and Soatto, Stefano},
  booktitle={Proceedings of the IEEE/CVF Conference on Computer Vision and Pattern Recognition},
  pages={9601--9610},
  year={2022}
}

@inproceedings{shi2022mimicking,
  title={Mimicking the oracle: An initial phase decorrelation approach for class incremental learning},
  author={Shi, Yujun and Zhou, Kuangqi and Liang, Jian and Jiang, Zihang and Feng, Jiashi and Torr, Philip HS and Bai, Song and Tan, Vincent YF},
  booktitle={Proceedings of the IEEE/CVF conference on computer vision and pattern recognition},
  pages={16722--16731},
  year={2022}
}

@article{wang2022s,
  title={S-prompts learning with pre-trained transformers: An occam’s razor for domain incremental learning},
  author={Wang, Yabin and Huang, Zhiwu and Hong, Xiaopeng},
  journal={Advances in Neural Information Processing Systems},
  volume={35},
  pages={5682--5695},
  year={2022}
}

@article{mcdonnell2023ranpac,
  title={Ranpac: Random projections and pre-trained models for continual learning},
  author={McDonnell, Mark D and Gong, Dong and Parvaneh, Amin and Abbasnejad, Ehsan and Van den Hengel, Anton},
  journal={Advances in Neural Information Processing Systems},
  volume={36},
  pages={12022--12053},
  year={2023}
}

@article{le2024mixture,
  title={Mixture of experts meets prompt-based continual learning},
  author={Le, Minh and Nguyen, Huy and Nguyen, Trang and Pham, Trang and Ngo, Linh and Ho, Nhat and others},
  journal={Advances in Neural Information Processing Systems},
  volume={37},
  pages={119025--119062},
  year={2024}
}

@article{gurbuz2022nispa,
  title={Nispa: Neuro-inspired stability-plasticity adaptation for continual learning in sparse networks},
  author={Gurbuz, Mustafa Burak and Dovrolis, Constantine},
  journal={arXiv preprint arXiv:2206.09117},
  year={2022}
}

@inproceedings{douillard2022dytox,
  title={Dytox: Transformers for continual learning with dynamic token expansion},
  author={Douillard, Arthur and Ram{\'e}, Alexandre and Couairon, Guillaume and Cord, Matthieu},
  booktitle={Proceedings of the IEEE/CVF conference on computer vision and pattern recognition},
  pages={9285--9295},
  year={2022}
}

@inproceedings{miao2021continual,
  title={Continual learning with filter atom swapping},
  author={Miao, Zichen and Wang, Ze and Chen, Wei and Qiu, Qiang},
  booktitle={International Conference on Learning Representations},
  year={2021}
}

@article{ostapenko2021continual,
  title={Continual learning via local module composition},
  author={Ostapenko, Oleksiy and Rodriguez, Pau and Caccia, Massimo and Charlin, Laurent},
  journal={Advances in Neural Information Processing Systems},
  volume={34},
  pages={30298--30312},
  year={2021}
}

@inproceedings{zheng2024multi,
  title={Multi-layer rehearsal feature augmentation for class-incremental learning},
  author={Zheng, Bowen and Zhou, Da-Wei and Ye, Han-Jia and Zhan, De-Chuan},
  booktitle={Forty-first International Conference on Machine Learning},
  year={2024}
}

@article{lesort2022challenging,
  title={Challenging common assumptions about catastrophic forgetting},
  author={Lesort, Timoth{\'e}e and Ostapenko, Oleksiy and Misra, Diganta and Arefin, Md Rifat and Rodr{\'\i}guez, Pau and Charlin, Laurent and Rish, Irina},
  journal={arXiv preprint arXiv:2207.04543},
  year={2022}
}

@article{hemati2025continual,
  title={Continual learning in the presence of repetition},
  author={Hemati, Hamed and Pellegrini, Lorenzo and Duan, Xiaotian and Zhao, Zixuan and Xia, Fangfang and Masana, Marc and Tscheschner, Benedikt and Veas, Eduardo and Zheng, Yuxiang and Zhao, Shiji and others},
  journal={Neural Networks},
  volume={183},
  pages={106920},
  year={2025},
  publisher={Elsevier}
}

@article{li2025optimal,
  title={Optimal Task Order for Continual Learning of Multiple Tasks},
  author={Li, Ziyan and Hiratani, Naoki},
  journal={arXiv preprint arXiv:2502.03350},
  year={2025}
}

@inproceedings{allen2022feature,
  title={Feature purification: How adversarial training performs robust deep learning},
  author={Allen-Zhu, Zeyuan and Li, Yuanzhi},
  booktitle={2021 IEEE 62nd Annual Symposium on Foundations of Computer Science (FOCS)},
  pages={977--988},
  year={2022},
  organization={IEEE}
}

@article{huang2023graph,
  title={Graph neural networks provably benefit from structural information: A feature learning perspective},
  author={Huang, Wei and Cao, Yuan and Wang, Haonan and Cao, Xin and Suzuki, Taiji},
  journal={arXiv preprint arXiv:2306.13926},
  year={2023}
}

@article{jelassi2022vision,
  title={Vision transformers provably learn spatial structure},
  author={Jelassi, Samy and Sander, Michael and Li, Yuanzhi},
  journal={Advances in Neural Information Processing Systems},
  volume={35},
  pages={37822--37836},
  year={2022}
}

@article{li2023theoretical,
  title={A theoretical understanding of shallow vision transformers: Learning, generalization, and sample complexity},
  author={Li, Hongkang and Wang, Meng and Liu, Sijia and Chen, Pin-Yu},
  journal={arXiv preprint arXiv:2302.06015},
  year={2023}
}

@article{han2024feature,
  title={On the feature learning in diffusion models},
  author={Han, Andi and Huang, Wei and Cao, Yuan and Zou, Difan},
  journal={arXiv preprint arXiv:2412.01021},
  year={2024}
}

@article{bu2024provably,
  title={Provably Transformers Harness Multi-Concept Word Semantics for Efficient In-Context Learning},
  author={Bu, Dake and Huang, Wei and Han, Andi and Nitanda, Atsushi and Suzuki, Taiji and Zhang, Qingfu and Wong, Hau-San},
  journal={Advances in Neural Information Processing Systems},
  volume={37},
  pages={63342--63405},
  year={2024}
}

@inproceedings{huang2023understanding,
  title={Understanding convergence and generalization in federated learning through feature learning theory},
  author={Huang, Wei and Shi, Ye and Cai, Zhongyi and Suzuki, Taiji},
  booktitle={The Twelfth International Conference on Learning Representations},
  year={2023}
}

@inproceedings{wen2021toward,
  title={Toward understanding the feature learning process of self-supervised contrastive learning},
  author={Wen, Zixin and Li, Yuanzhi},
  booktitle={International Conference on Machine Learning},
  pages={11112--11122},
  year={2021},
  organization={PMLR}
}

@article{banayeeanzade2024theoretical,
  title={Theoretical insights into overparameterized models in multi-task and replay-based continual learning},
  author={Banayeeanzade, Amin and Soltanolkotabi, Mahdi and Rostami, Mohammad},
  journal={arXiv preprint arXiv:2408.16939},
  year={2024}
}

@article{zhengtowards,
  title={Towards Understanding Memory buffer based Continual Learning},
  author={Zheng, Guodong and Wang, Peng and Shen, Li}
}
\bibliographystyle{plainnat}

% \subsubsection*{Acknowledgements}
% All acknowledgments go at the end of the paper, including thanks to reviewers who gave useful comments, to colleagues who contributed to the ideas, and to funding agencies and corporate sponsors that provided financial support. 
% To preserve the anonymity, please include acknowledgments \emph{only} in the camera-ready papers. The acknowledgements do not count against the 9-page page limit in the camera-ready.

% \subsubsection*{References}

% References follow the acknowledgements. Use an unnumbered third level
% heading for the references section.  Please use the same font
% size for references as for the body of the paper---remember that
% references do not count against your page length total.

%%%%%%%%%%%%%%%%%%%%%%%%%%%%%%%%%%%%%%%%%%%%%%%%%%%%%%%%%%%%
\section*{Checklist}

\begin{enumerate}

  \item For all models and algorithms presented, check if you include:
  \begin{enumerate}
    \item A clear description of the mathematical setting, assumptions, algorithm, and/or model. [Yes]
    \item An analysis of the properties and complexity (time, space, sample size) of any algorithm. [Yes]
    \item (Optional) Anonymized source code, with specification of all dependencies, including external libraries. [Yes]
  \end{enumerate}

  \item For any theoretical claim, check if you include:
  \begin{enumerate}
    \item Statements of the full set of assumptions of all theoretical results. [Yes]
    \item Complete proofs of all theoretical results. [Yes]
    \item Clear explanations of any assumptions. [Yes]     
  \end{enumerate}

  \item For all figures and tables that present empirical results, check if you include:
  \begin{enumerate}
    \item The code, data, and instructions needed to reproduce the main experimental results (either in the supplemental material or as a URL). [Yes]
    \item All the training details (e.g., data splits, hyperparameters, how they were chosen). [Yes]
    \item A clear definition of the specific measure or statistics and error bars (e.g., with respect to the random seed after running experiments multiple times). [Yes]
    \item A description of the computing infrastructure used. (e.g., type of GPUs, internal cluster, or cloud provider). [Yes]
  \end{enumerate}

  \item If you are using existing assets (e.g., code, data, models) or curating/releasing new assets, check if you include:
  \begin{enumerate}
    \item Citations of the creator If your work uses existing assets. [Not Applicable]
    \item The license information of the assets, if applicable. [Not Applicable]
    \item New assets either in the supplemental material or as a URL, if applicable. [Not Applicable]
    \item Information about consent from data providers/curators. [Not Applicable]
    \item Discussion of sensible content if applicable, e.g., personally identifiable information or offensive content. [Not Applicable]
  \end{enumerate}

  \item If you used crowdsourcing or conducted research with human subjects, check if you include:
  \begin{enumerate}
    \item The full text of instructions given to participants and screenshots. [Not Applicable]
    \item Descriptions of potential participant risks, with links to Institutional Review Board (IRB) approvals if applicable. [Not Applicable]
    \item The estimated hourly wage paid to participants and the total amount spent on participant compensation. [Not Applicable]
  \end{enumerate}

\end{enumerate}

\clearpage
\appendix
\thispagestyle{empty}

% Supplementary material: To improve readability, you must use a single-column format for the supplementary material.
\onecolumn
\aistatstitle{Supplementary Materials}

\section{Additional Related Work}

\paragraph{Feature Learning Theory} \cite{allen2022feature} first introduced the feature learning framework to explain the benefits of adversarial training in robust learning. This was further extended by \cite{allen2020towards}, who incorporated a multi-view data structure to show how ensemble methods can enhance generalization. Since then, feature learning has been studied across a range of model architectures, including graph neural networks \cite{huang2023graph}, convolutional neural networks \cite{cao2022benign,kou2023benign}, vision transformers \cite{jelassi2022vision,li2023theoretical}, and diffusion models \cite{han2024feature}.
Beyond model architectures, the framework has also been used to analyze the behavior of optimization algorithms and training techniques—such as Adam \cite{zou2023benefits}, momentum \cite{jelassi2022towards}, and Mixup \cite{zou2023benefits}. Furthermore, feature learning provides new insights into broader learning paradigms, including federated learning \cite{huang2023understanding,baoprovable}, contrastive learning \cite{wen2021toward}, and in-context learning \cite{bu2024provably}. To the best of our knowledge, this work is the first to investigate the effects of data replay in continual learning from the perspective of feature learning. Compared to standard learning settings, continual learning introduces additional challenges—such as task-specific feature vectors, and complex interactions between signal and noise across sequential tasks—which make theoretical analysis significantly more intricate.

\section{Additional Experimental}

\subsection{Synthetic Data}

\textbf{Accuracy Reflects Learning Dynamics.} Figure~\ref{fig:acc_dynamics} highlights how both task ordering and inter-task similarity influence model accuracy during continual learning, with trends that align closely with the signal and noise dynamics presented in Figure~\ref{fig:dynamics}. When the task with the strongest signal (i.e., highest $\alpha_k$) is placed earlier in the sequence—such as Task 3 in subplots (\ref{fig:acc_01_order}–\ref{fig:acc_07_order})—the model is better able to acquire meaningful representations, resulting in higher accuracy even for subsequent lower-signal tasks. In contrast, when lower-signal tasks are prioritized (subplots \ref{fig:acc_01}–\ref{fig:acc_07}), signal learning for those tasks becomes less effective, and overall accuracy suffers. Specifically, when the alignment with task-specific signal directions dominates over noise components , task accuracy exceeds 50\%. Conversely, when noise memorization exceeds signal learning, accuracy deteriorates to near-random levels. For instance, under low task correlation ($A_{(m,m')} = 0.1$), Task 1 performs poorly when it appears last in the training sequence (Figure~\ref{fig:acc_01}), but its performance significantly improves when prioritized earlier (Figure~\ref{fig:acc_01_order}), confirming that task ordering matters.
Additionally, across all orderings, stronger inter-task correlations (e.g., $A_{(m,m')} = 0.7$) facilitate signal transfer across tasks, allowing lower-signal tasks to benefit from earlier learned features. These patterns underscore the consistency between accuracy outcomes and the learning dynamics: accuracy increases when signal learning outweighs noise memorization, and fails when the noise dominates the representation.

\begin{figure}[!h]
    \centering
    \begin{subfigure}[b]{0.32\textwidth}
        \includegraphics[width=\textwidth]{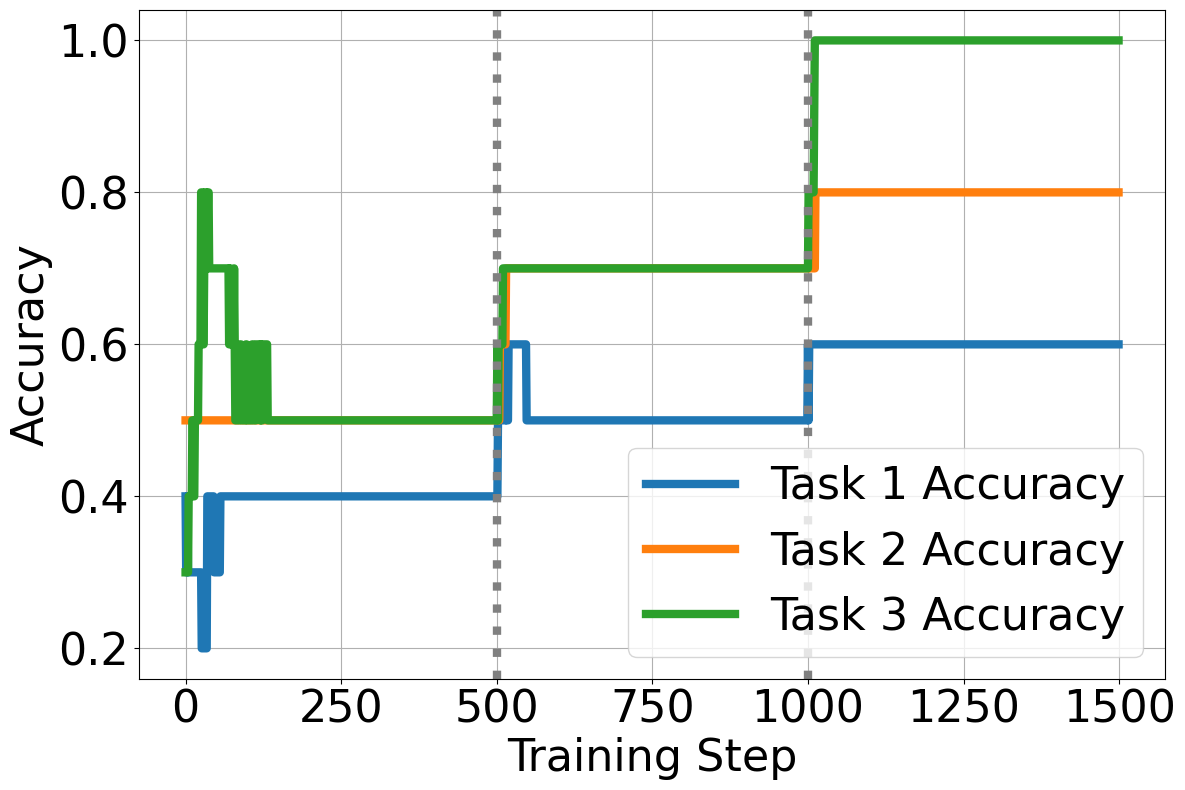}
        \caption{$A_{(m,m^{\prime})}$ = 0.1}
        \label{fig:acc_01}
    \end{subfigure}
    \hfill
    \begin{subfigure}[b]{0.32\textwidth}
        \includegraphics[width=\textwidth]{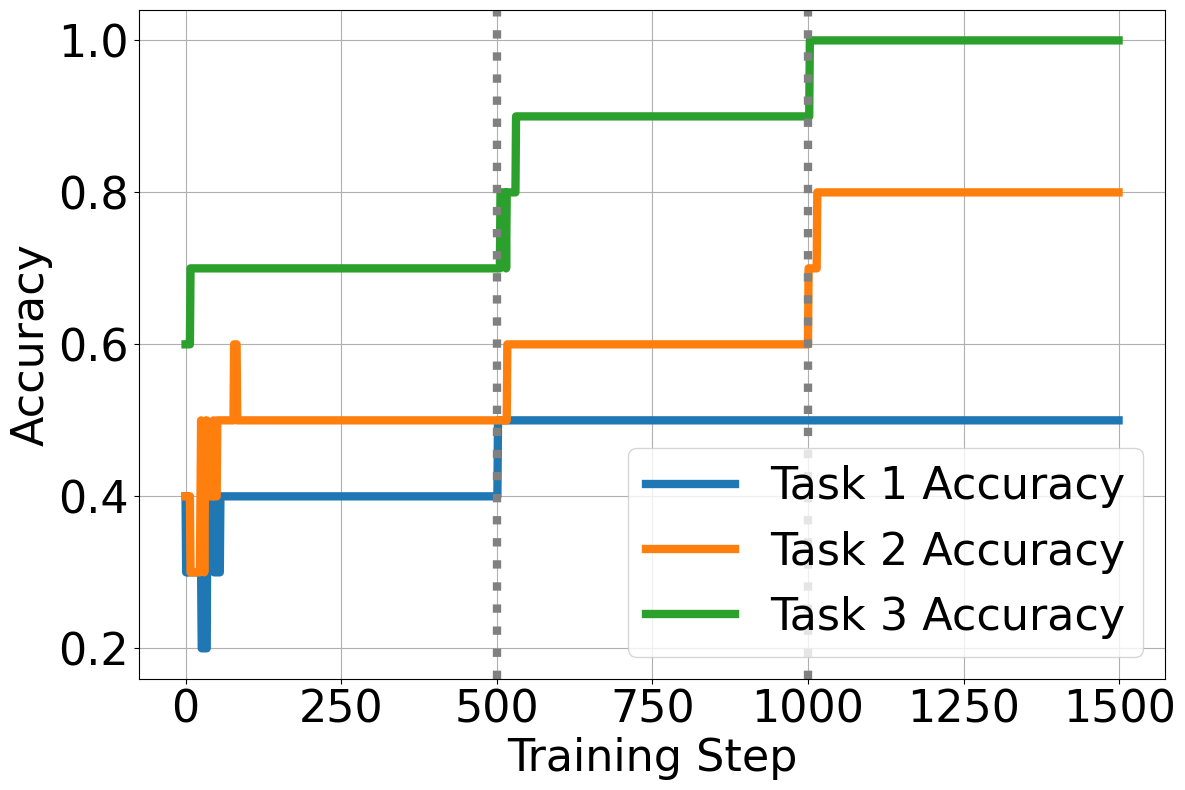}
        \caption{$A_{(m,m^{\prime})}$ = 0.3}
        \label{fig:acc_03}
    \end{subfigure}
    \hfill
    \begin{subfigure}[b]{0.32\textwidth}
        \includegraphics[width=\textwidth]{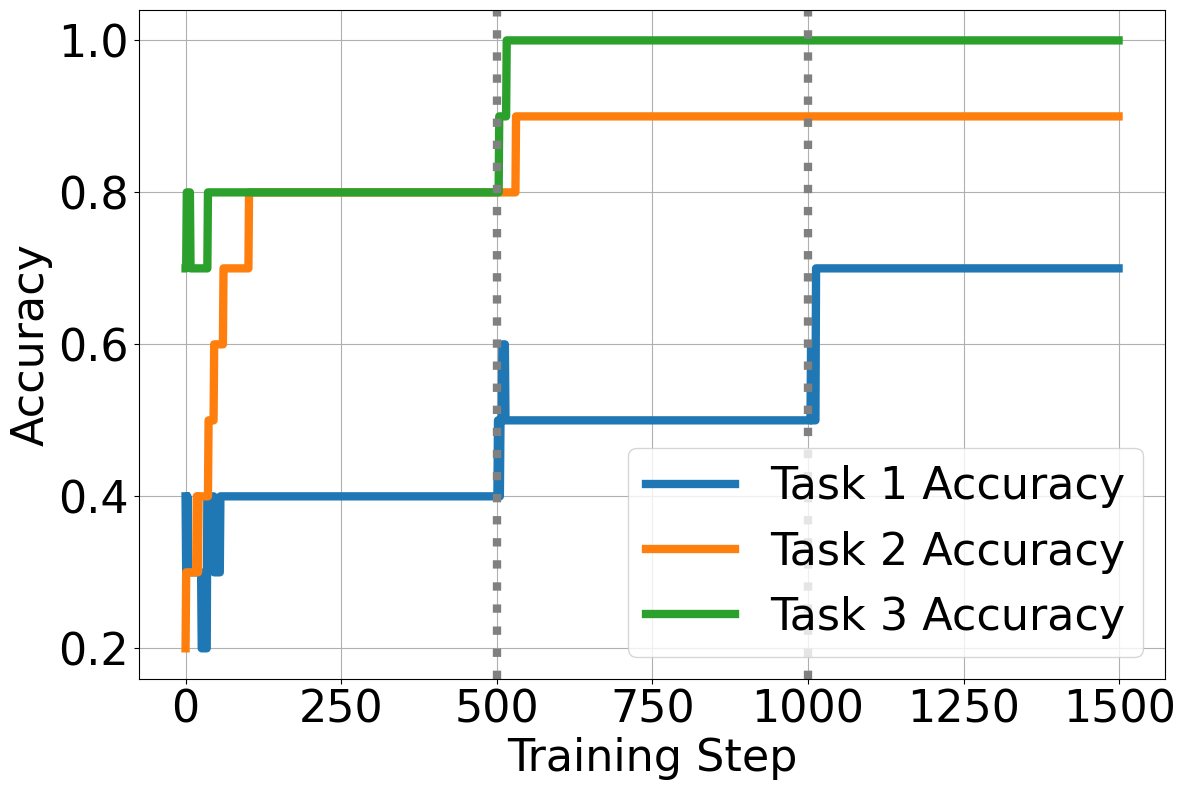}
        \caption{$A_{(m,m^{\prime})}$ = 0.7}
        \label{fig:acc_07}
    \end{subfigure}

    \begin{subfigure}[b]{0.32\textwidth}
        \includegraphics[width=\textwidth]{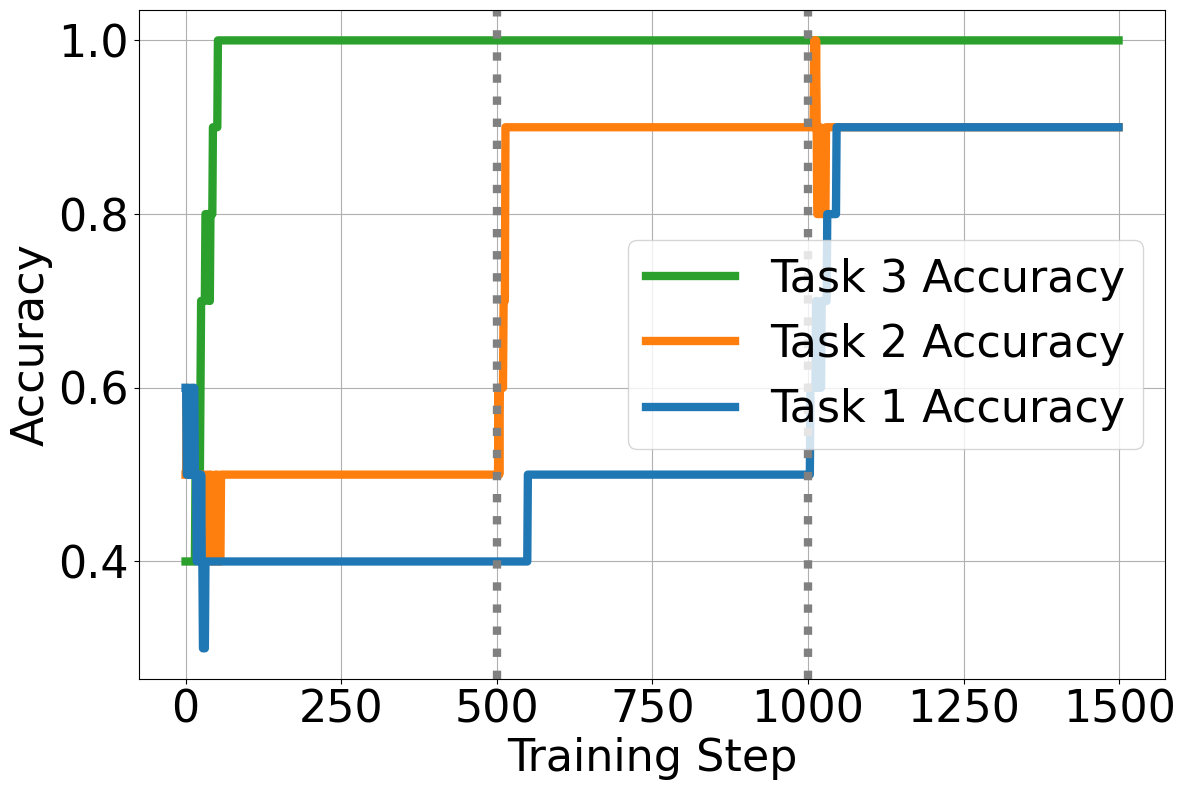}
        \caption{$A_{(m,m^{\prime})}$ = 0.1}
        \label{fig:acc_01_order}
    \end{subfigure}
    \hfill
    \begin{subfigure}[b]{0.32\textwidth}
        \includegraphics[width=\textwidth]{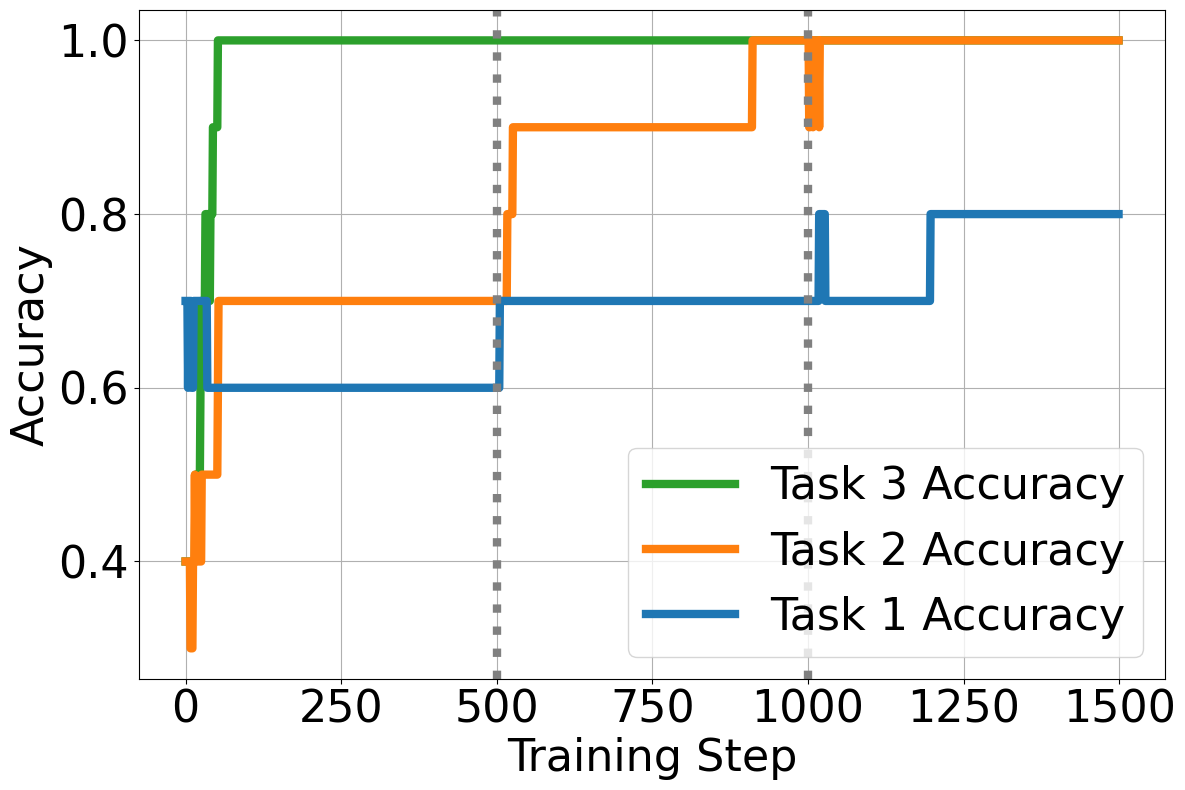}
        \caption{$A_{(m,m^{\prime})}$ = 0.3}
        \label{fig:acc_03_order}
    \end{subfigure}
    \hfill
    \begin{subfigure}[b]{0.32\textwidth}
        \includegraphics[width=\textwidth]{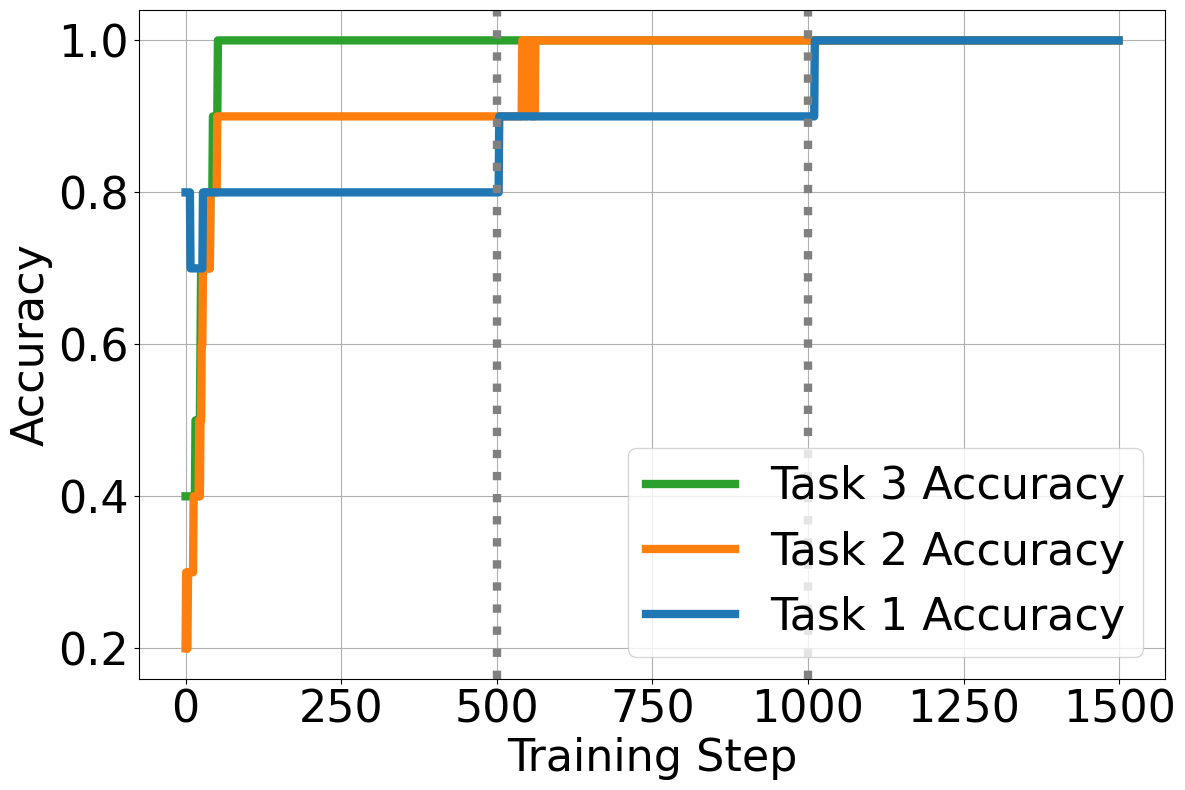}
        \caption{$A_{(m,m^{\prime})}$ = 0.7}
        \label{fig:acc_07_order}
    \end{subfigure}
    \caption{Accuracy under full data-replay continual training across different task orderings and correlation strengths. }
    \label{fig:acc_dynamics}
\end{figure}

\textbf{Catastrophic Forgetting Occurs with Lower Task Similarity.}
Figure~\ref{fig:cf_dynamics} investigates catastrophic forgetting under full data-replay continual learning by varying the inter-task correlation $A_{(m,m')}$. When the correlation is extremely low or near zero (e.g., $A_{(m,m')} = 0.01$), the tasks are nearly orthogonal—meaning their signal directions share no meaningful relationship. In this regime, newly introduced tasks overwrite earlier ones, and previously learned signal components decay, resulting in forgetting. As the correlation increases to 0.1, tasks begin to share overlapping features, which helps stabilize the representations and retain earlier task knowledge over time. These results highlight that task similarity, measured through correlation, is critical for mitigating forgetting: when tasks are orthogonal (i.e., $A \approx 0$), they compete destructively during training, whereas higher similarity allows for constructive feature reuse and knowledge retention.

\begin{figure}[!h]
    \centering
    \begin{subfigure}[b]{0.32\textwidth}
        \includegraphics[width=\textwidth]{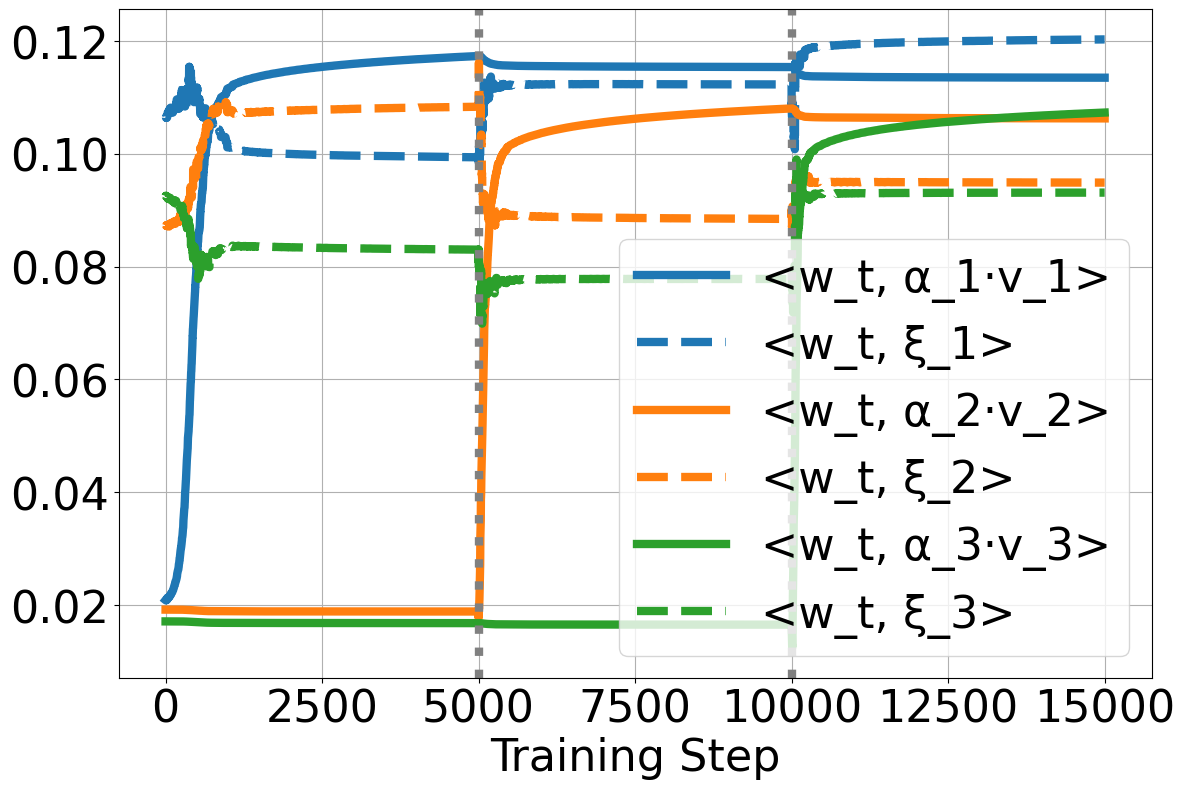}
        \caption{$A_{(m,m^{\prime})}$ = 0}
        \label{fig:cf_0}
    \end{subfigure}
    \hfill
    \begin{subfigure}[b]{0.32\textwidth}
        \includegraphics[width=\textwidth]{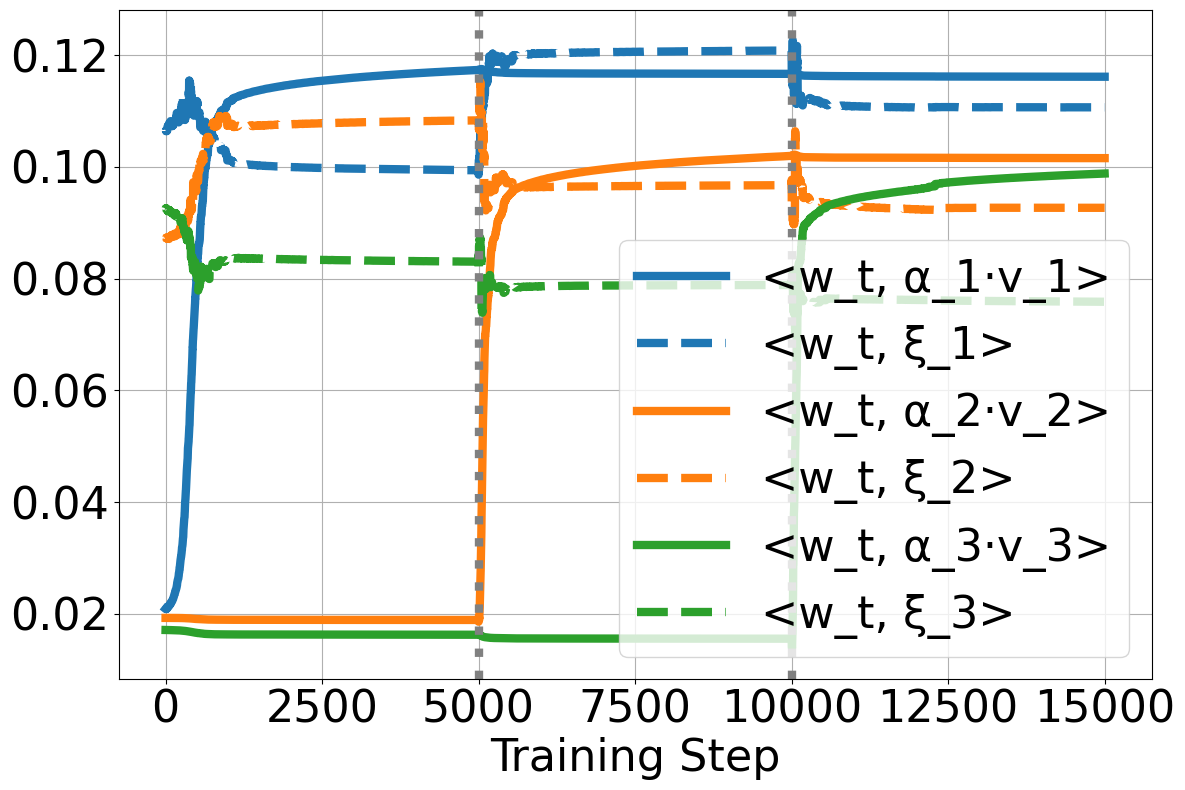}
        \caption{$A_{(m,m^{\prime})}$ = 0.01}
        \label{fig:cf_001}
    \end{subfigure}
    \hfill
    \begin{subfigure}[b]{0.32\textwidth}
        \includegraphics[width=\textwidth]{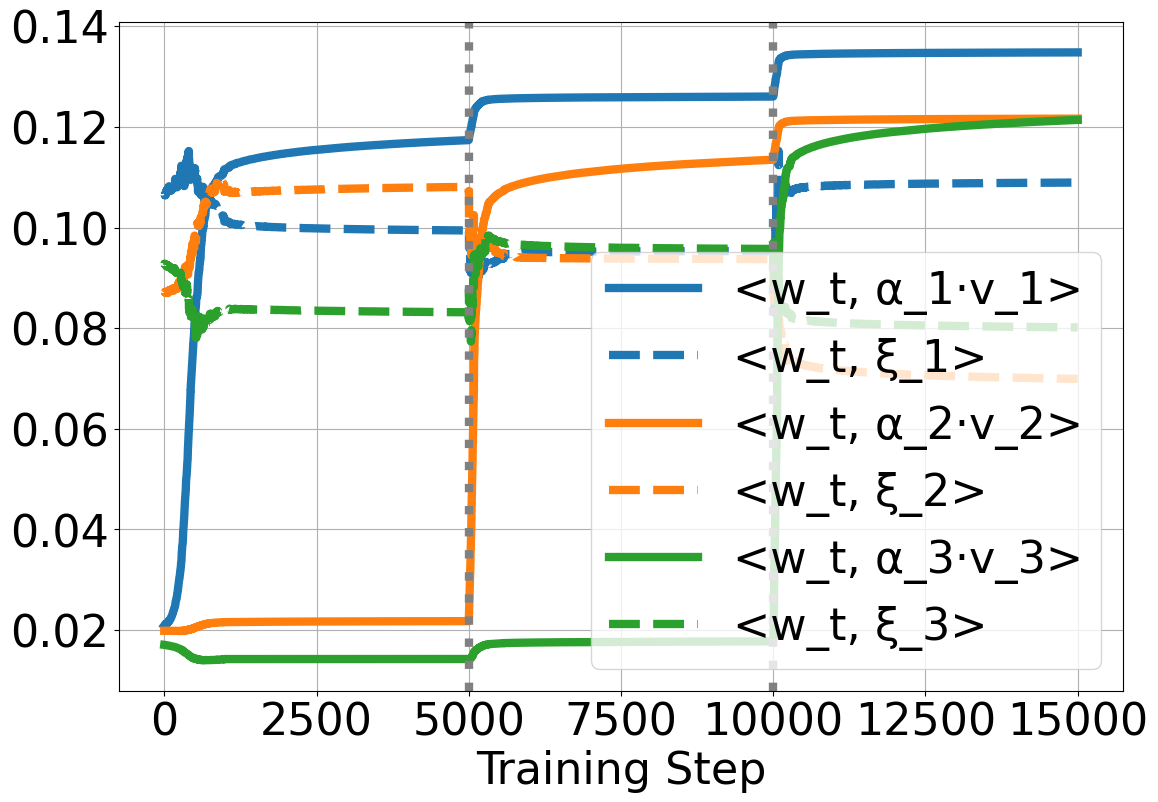}
        \caption{$A_{(m,m^{\prime})}$ = 0.1}
        \label{fig:cf_01}
    \end{subfigure}

    \begin{subfigure}[b]{0.32\textwidth}
        \includegraphics[width=\textwidth]{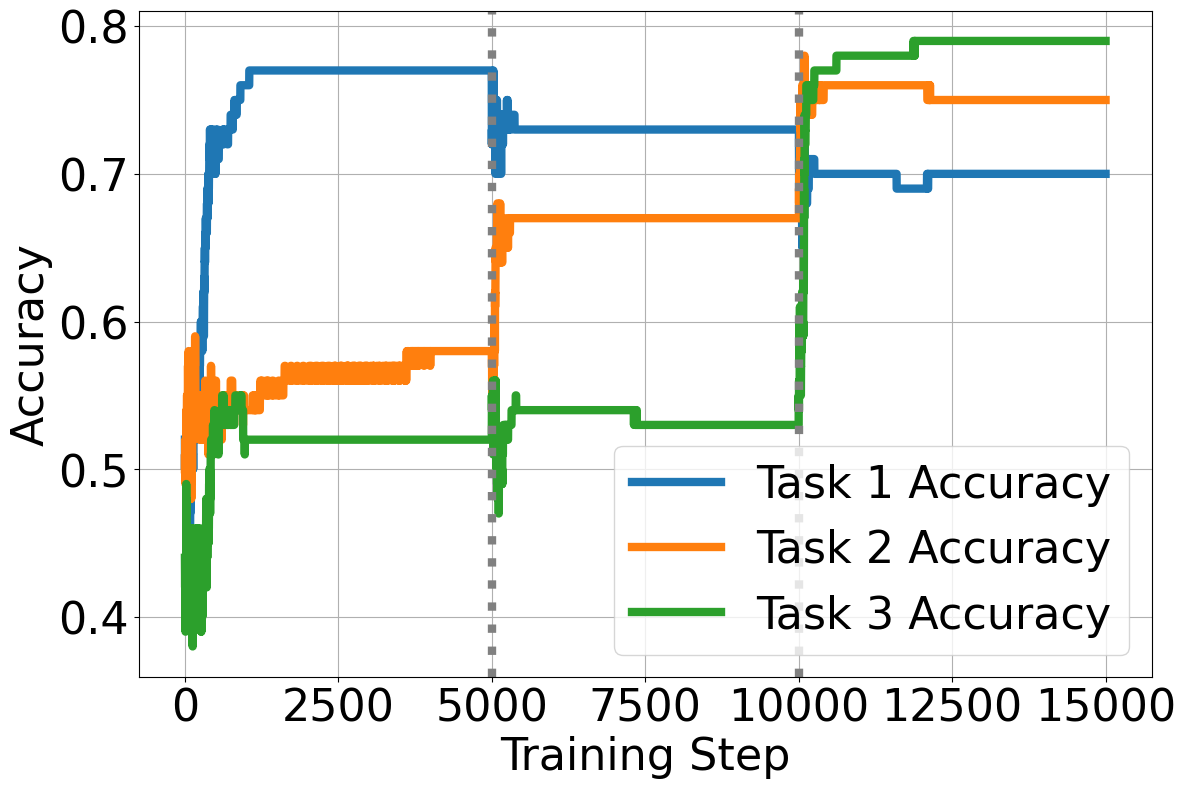}
        \caption{$A_{(m,m^{\prime})}$ = 0}
        \label{fig:acc_cf_0}
    \end{subfigure}
    \hfill
    \begin{subfigure}[b]{0.32\textwidth}
        \includegraphics[width=\textwidth]{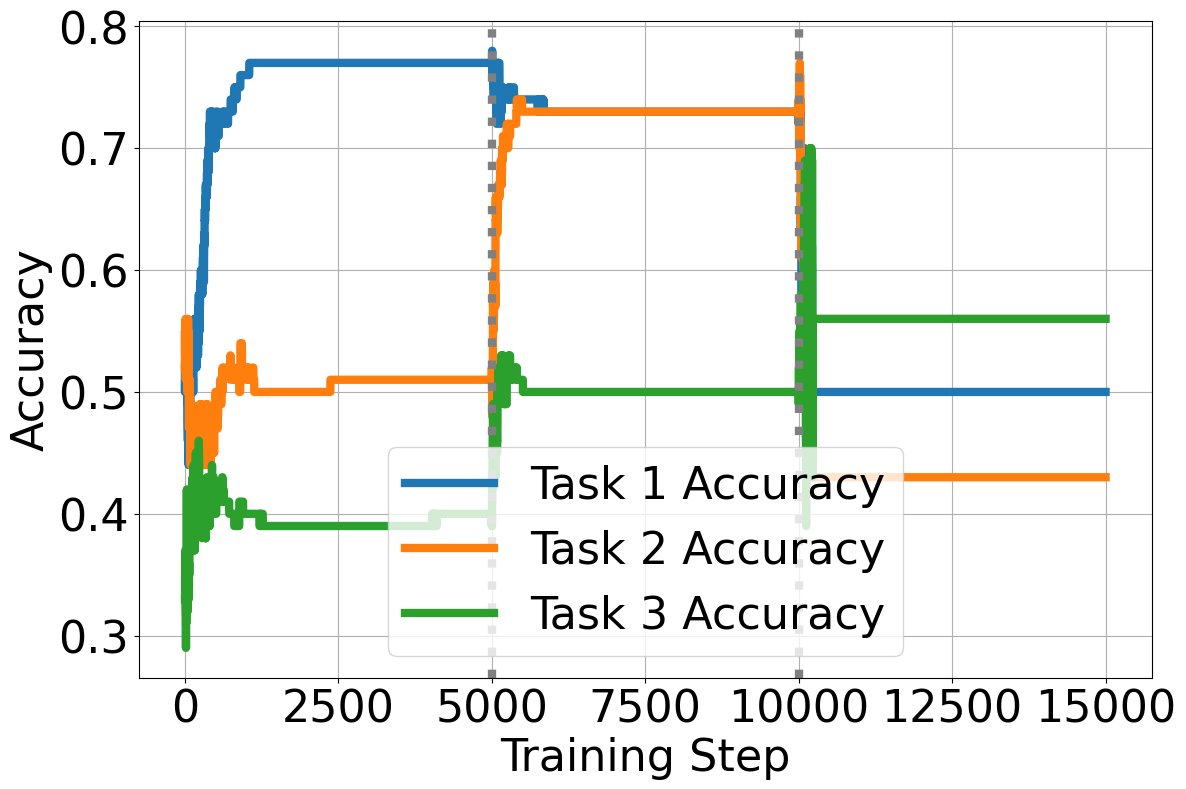}
        \caption{$A_{(m,m^{\prime})}$ = 0.01}
        \label{fig:acc_cf_001}
    \end{subfigure}
    \hfill
    \begin{subfigure}[b]{0.32\textwidth}
        \includegraphics[width=\textwidth]{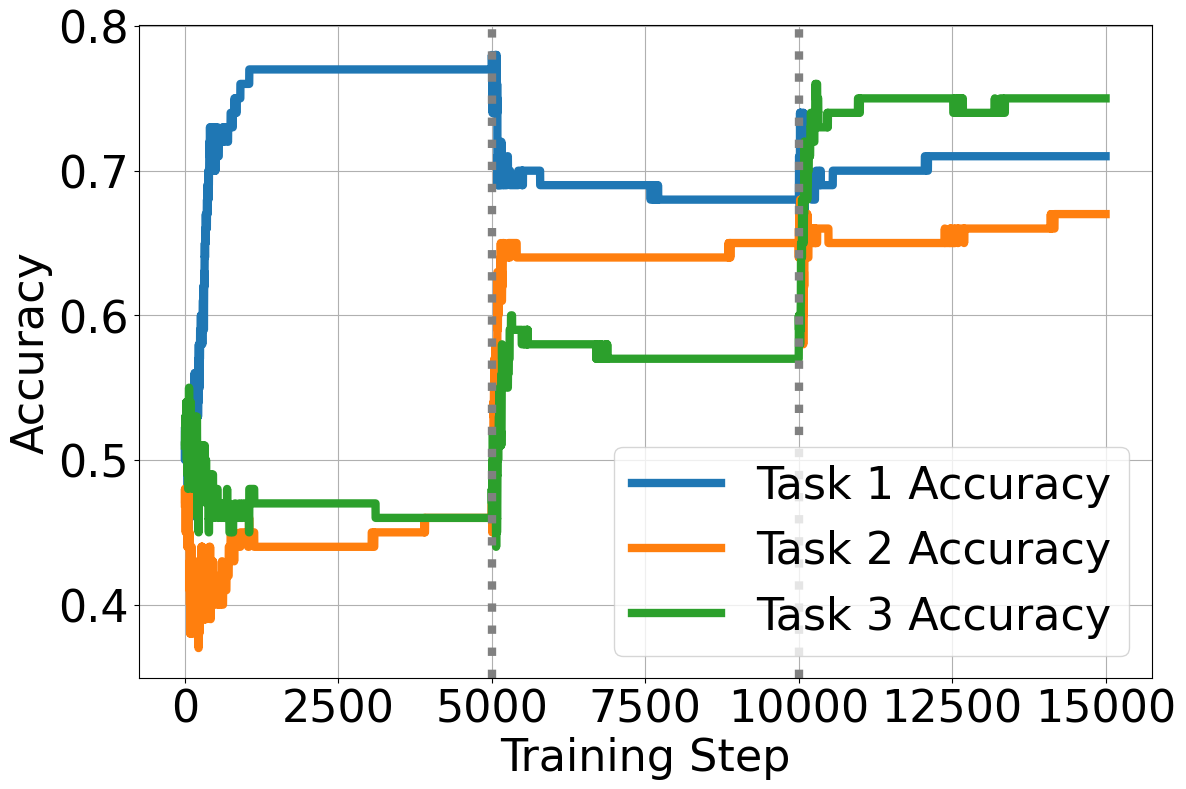}
        \caption{$A_{(m,m^{\prime})}$ = 0.1}
        \label{fig:acc_cf_01}
    \end{subfigure}

    \caption{Catastrophic forgetting under full data-replay continual training across various correlation strengths. }
    \label{fig:cf_dynamics}
\end{figure}

\subsection{Empirical Verification on Real-World Data}
\label{appendix:real_world_exp}

To address the limitations of synthetic data and shallow networks, and to further validate our theoretical findings in a realistic deep learning scenario, we conduct experiments using the {CIFAR-100} benchmark \cite{krizhevsky2009learning} with a {ResNet-18} architecture \cite{he2016deep}.

Crucially, to ensure a rigorous alignment with our theoretical framework—which analyzes {task-incremental binary classification} (see Definition~1 and Section~4), we adapt the CIFAR-100 tasks into binary classification problems (e.g., ``Class A vs. Rest''). This setup allows us to strictly verify the impact of {signal-to-noise ratio (SNR)} and {task correlation} ($A_{(m,m')}$) on feature learning and forgetting.

\paragraph{Experimental Setup.} 
We construct binary tasks from CIFAR-100 superclasses. For a target class $C$ (e.g., \textit{Bicycle}), positive samples are drawn from $C$, and negative samples are randomly sampled from disjoint classes to create a balanced binary dataset.
\begin{itemize}
    \item \textbf{Model:} We employ a ResNet-18 backbone. To isolate feature transfer from classifier interference, we utilize a multi-head architecture where the backbone is shared across tasks, but each task possesses an independent binary linear classifier.
    \item \textbf{Training:} Consistent with our theoretical premise, we employ \textit{Full Data Replay}. When training on Task $m$, the model is optimized on the union of all datasets $\mathcal{D}_{1} \cup \dots \cup \mathcal{D}_{m}$.
\end{itemize}

\begin{figure}
    \centering
    \includegraphics[width=0.85\linewidth]{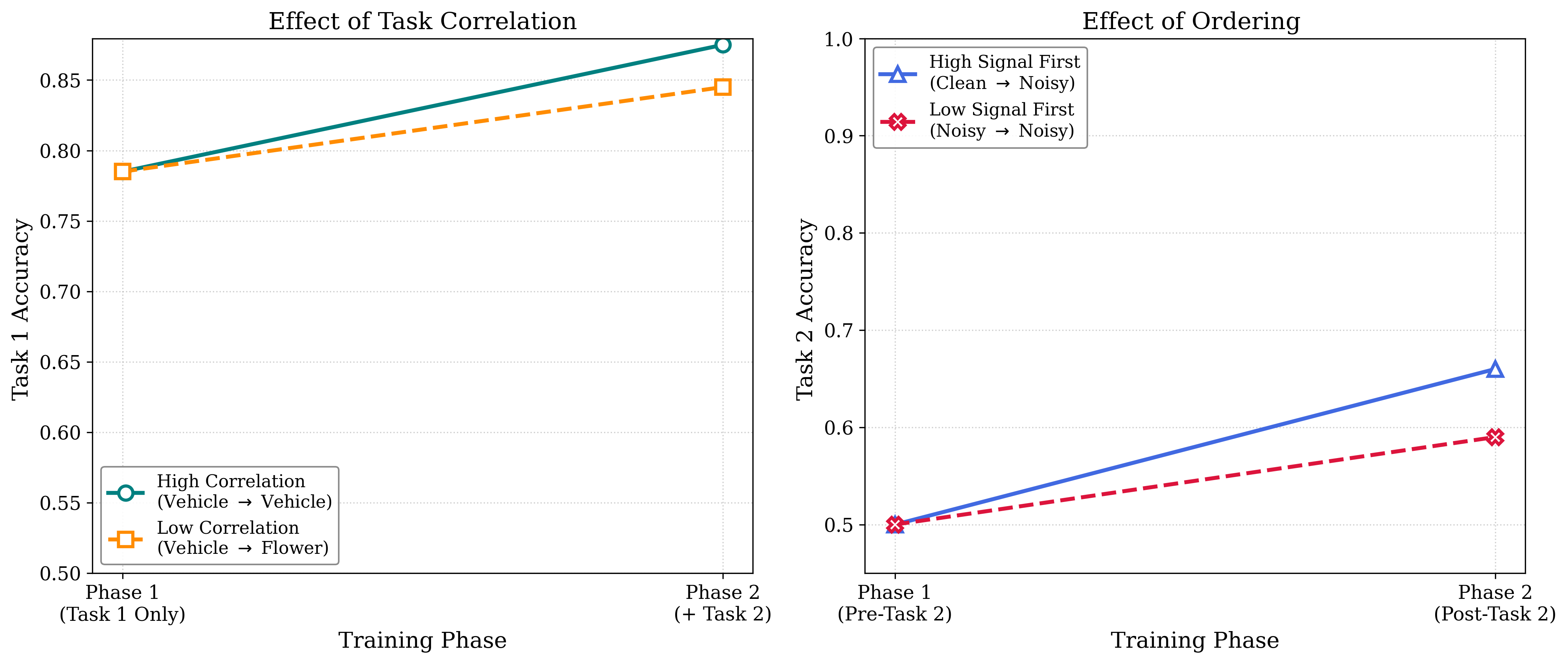}
    \caption{Empirical Verification of Correlation and Ordering Effects on CIFAR-100 (ResNet-18)}
    \label{fig:cifar100}
\end{figure}

\paragraph{Impact of Task Correlation.}
Theorem~1 suggests that high inter-task correlation ($A_{(m,m')} > 0$) facilitates signal accumulation. When tasks share feature subspaces, training on a subsequent Task $m$ should reinforce the features relevant to Task 1. We design two sequences:
\begin{enumerate}
    \item \textbf{High Correlation:} Task 1 (\textit{Bicycle}) $\to$ Task 2 (\textit{Motorcycle}). Both belong to the \textit{Vehicles 1} superclass and share semantic features (e.g., wheels).
    \item \textbf{Low Correlation:} Task 1 (\textit{Bicycle}) $\to$ Task 2 (\textit{Orchid}). The classes belong to disjoint superclasses and represent orthogonal tasks.
\end{enumerate}

\textbf{Results:} As illustrated in {Figure~\ref{fig:cifar100} (Left)}, while full replay allows both models to maintain performance, the {High Correlation} sequence (Teal line) exhibits superior retention and positive backward transfer compared to the Low Correlation sequence (Orange dashed line). The introduction of the semantically related \textit{Motorcycle} task reinforces the feature subspace used by \textit{Bicycle}, validating our theoretical insight that feature sharing is critical for robust signal accumulation.

\paragraph{Impact of Task Ordering and SNR.}
Theorem~2 uncovers that prioritizing higher-signal tasks facilitates the learning of subsequent tasks. To simulate varying SNR in real-world images, we inject strong Gaussian noise ($\sigma_{\text{noise}}$) into the inputs. We focus on two aligned tasks from the \textit{Fruit} superclass: \textit{Apple} (Task 1) and \textit{Pear} (Task 2). We investigate whether a high-signal Task 1 facilitates the learning of a low-signal Task 2:
\begin{enumerate}
    \item \textbf{High-Signal First (Setup A):} Task 1 is \textit{Clean Apple} ($\sigma=0$) $\to$ Task 2 is \textit{Noisy Pear} ($\sigma=5.0$).
    \item \textbf{Low-Signal First (Setup B):} Task 1 is \textit{Noisy Apple} ($\sigma=5.0$) $\to$ Task 2 is \textit{Noisy Pear} ($\sigma=5.0$).
\end{enumerate}

\textbf{Results:} {Figure~\ref{fig:cifar100} (Right)} demonstrates the critical role of ordering. In Setup A (Blue line), the model learns robust ``fruit" features from the Clean Apple task in Phase 1. When the Noisy Pear task arrives in Phase 2, the model leverages these pre-learned features to achieve significantly higher accuracy ($\sim 66\%$). In contrast, in Setup B (Red dashed line), the model struggles to learn meaningful features from the initial Noisy Apple task; consequently, its ability to learn the subsequent Noisy Pear task is impaired ($\sim 59\%$). This empirically confirms prioritizing high-signal tasks is essential for effective feature transfer to downstream low-signal tasks.

\section{Proof of Main Results}

\subsection{Notations.}
Given the iterate $\mathbf{W}^{(t)}$ in sequential training, we define the following notations during the training process:
\begin{itemize}
    \item The learning dynamics of task $k$'s feature at time $t$ under current task $m$: $\Gamma_{(m, r)}^{(t,k)}:=\langle\mathbf{w}_{(m,r)}^{(t)}, \mathbf{v}_k^*\rangle$.
    \item The learning dynamics of task $k$'s noise at time $t$ under current task $m$: $\Phi_{(m, r)}^{(t,k,j)}:=\langle\mathbf{w}_{(m,r)}^{(t)}, \bm{\xi}_k^j\rangle$.
    \item Derivative: $\ell_{m j}^{(t)}=\ell_{m j}(\mathbf{W}^{(t)})=1 /(1+e^{y_{m j} F(\mathbf{W}^{(t)}, \mathbf{x}_{m j})})$ for $j \in[n]$.
    % \item Derivative in each task: $\nu_m^{(t)}=\frac{1}{n} \sum_{j \in[n]} \ell_{m j}^{(t)}$ for $m \in[M]$.
    \item Maximum signal intensity: $\Gamma_{(m, r^*)}^{(t,k)} \equiv \Gamma_{(m, r_k^*)}^{(t,k)}$, where $r_k^*=\arg \max _{r \in[R]} \Gamma_{(0, r)}^{(0,k)} $.
    \item Maximum noise memorization: $\Phi_{(m, r^*)}^{(t,k,j)} \equiv \Phi_{(m, r_{kj}^*)}^{(t,k,j)} $, where $r_{k j}^*=\arg \max _{r \in[R]} y_{k j} \Phi_{(0, r)}^{(0,k,j)}$.
    % \item $\tau_{k j}^m$ : the first iteration $\max _{r \in[R]} y_{k j} \Phi_{(m, r)}^{(t,k,j)} \geq \Theta(R^{-\frac{1}{3}})$.
\end{itemize}

\subsection{Learning dynamics of task $k$'s feature and noise  at time $t$ under current task $m$.}
\label{sec:update_of_Phi_Gamma_replay}

According to Definition~\ref{def:data_distribution}, we assume that tasks share common features, i.e.,$A_{(m,m^{\prime})} > 0$. As a result, even without direct training on the target task, the model can still accumulate relevant features through similar tasks. Furthermore, based on the gradient computation, the learned signal can be characterized as follows:

\begin{equation}\label{eq:update_signal_data_replay}
        \begin{aligned}
            \Gamma_{(m, r)}^{(t,k)} &= \langle\mathbf{w}_{(m,r)}^{(t)}, \mathbf{v}_k^*\rangle  \\
            & = \langle\mathbf{w}_{(m,r)}^{(t-1)}-\eta \nabla_{\mathbf{w}_r} L(\mathbf{W}_m^{(t-1)},D_1, ...,D_m), \mathbf{v}_k^*\rangle \\
            & = \langle\mathbf{w}_{(m,r)}^{(t-1)}+\frac{\eta}{nm} \sum_{p \in[m]} \sum_{j \in[n]} y_{k j} \ell_{p j}(\mathbf{W}_m^{(t-1)})  [ 3\langle\mathbf{w}_{(m,r)}^{(t-1)}, \alpha_p y_{p j} \mathbf{v}_p^*\rangle^2 \cdot \alpha_p y_{p j} \mathbf{v}_p^* ], \mathbf{v}_k^*\rangle \\ 
            & = \Gamma_{(m, r)}^{(t-1,k)} + \frac{\eta}{nm} \sum_{j \in[n]} \sum_{p \in[m]}  3\alpha_p^3 A_{(p,k)} \ell_{p j}(\mathbf{W}_m^{(t-1)}) (\Gamma_{(m, r)}^{(t-1,p)})^2   \\
            & = \Gamma_{(m, r)}^{(0,k)} + \frac{\eta}{nm} \sum_{j \in[n]} \sum_{p \in[m]} \sum_{s \in[T_m]} 3\alpha_p^3 A_{(p,k)} \ell_{p j}(\mathbf{W}_m^{(s-1)}) (\Gamma_{(m, r)}^{(s-1,p)})^2   \\
            & =\Gamma_{(0, r)}^{(0,k)} + \frac{\eta}{nm} \sum_{j \in[n]} \sum_{q \in [m]}\sum_{p \in[q]} \sum_{s \in[T_q]} 3\alpha_p^3 A_{(p,k)} \ell_{p j}(\mathbf{W}_q^{(s-1)}) (\Gamma_{(q, r)}^{(s-1,p)})^2.
        \end{aligned}
    \end{equation}

When considering noise memorization, it can be observed that the noise also continues to accumulate regardless of the relationship between task $m$ and $k$.
\begin{equation}\label{eq:update_noise_data_replay}
        \begin{aligned}
            &\Phi_{(m, r)}^{(t,k,j)} = \langle\mathbf{w}_{(m,r)}^{(t)}, \bm{\xi}_{kj}\rangle = \langle\mathbf{w}_{(m,r)}^{(t)}, \bm{\xi}_{kj}\rangle \\
            & = \langle\mathbf{w}_{(m,r)}^{(t-1)}-\eta \nabla_{\mathbf{w}_r} L(\mathbf{W}_m^{(t-1)},D_1, ...,D_m), \bm{\xi}_{kj}\rangle \\
            & = \langle\mathbf{w}_{(m,r)}^{(t-1)}+\frac{\eta}{nm} \sum_{j^{\prime} \in[n]} y_{m j^{\prime}} \ell_{m j^{\prime}}(\mathbf{W}_m^{(t-1)})  [ 3\langle\mathbf{w}_{(m,r)}^{(t-1)}, \alpha y_{m j^{\prime}} \mathbf{v}_m^*\rangle^2 \cdot \alpha y_{m j^{\prime}} \mathbf{v}_m^* + 3\langle\mathbf{w}_{(m,r)}^{(t-1)}, \bm{\xi}_{m j^{\prime}}\rangle^2 \cdot \bm{\xi}_{m j^{\prime}} ], \bm{\xi}_{kj}\rangle \\ 
            & + \langle \sum_{p=1}^{m} \frac{\eta}{nm} \sum_{j^{\prime} \in[n]} y_{p j^{\prime}} \ell_{p j^{\prime}}(\mathbf{W}_m^{(t-1)})  [ 3\langle\mathbf{w}_{(m,r)}^{(t-1)}, \alpha y_{p j^{\prime}} \mathbf{v}_p^*\rangle^2 \cdot \alpha y_{p j^{\prime}} \mathbf{v}_p^* + 3\langle\mathbf{w}_{(m,r)}^{(t-1)}, \bm{\xi}_{p j^{\prime}}\rangle^2 \cdot \bm{\xi}_{p j^{\prime}} ], \bm{\xi}_{kj}\rangle \\ 
            & = \Phi_{(m, r)}^{(t-1,k,j)} +  \sum_{p=1}^{m}  \frac{\eta}{nm} \sum_{j^{\prime} \in[n]} y_{p j^{\prime}} \ell_{p j^{\prime}}(\mathbf{W}_m^{(t-1)}) (\Phi_{(m, r)}^{(t-1,p,j)})^2 \langle \bm{\xi}_{p j^{\prime}},\bm{\xi}_{k j} \rangle.
        \end{aligned}
\end{equation}

\subsection{Proof of Theorem~\ref{thm:forget_m_ffs_tasks}.}

In this section, we present the proof of Theorem~\ref{thm:forget_m_ffs_tasks} in two parts. The first part analyzes the failure of signal learning after training on $k$ tasks (i.e., before task $k+1$ ). The second part focuses on noise memorization after training on $m>k$ tasks (i.e., before task $m+1$ ) and further considers two scenarios in the later phase: one where learning continues to fail, and another where signal learning is enhanced.

In the following, we show that the signal learning is always under control before training task $m+1$.

\begin{lemma}\label{lem:max_Gamma_task_k}
    In the data replay training process on task $m$, with probability at least $1-1 /\operatorname{poly}(d)$, it holds that
    $\max _{r \in[R], p \in [k]}|\Gamma_{(k,r)}^{(t,p)}| \leq \widetilde{O}(\sigma_0) $ for any $t \in [T_k], p \in [k]$.
\end{lemma}

\begin{proof}[\bf Proof of \cref{lem:max_Gamma_task_k}]
    We consider the induction process to prove the statement. We assume that, for any $s \leq t$, it holds that $\max _{r \in[R], p \in [k]}|\Gamma_{(k,r)}^{(t,p)}| \leq \widetilde{O}(\sigma_0) $. Then, we proceed to analyze the case for $s=t+1.$ According to \cref{eq:update_signal_data_replay}, we have:
    
    \begin{equation}
        \begin{aligned}
            |\Gamma_{(k, r)}^{(t,k)}| &= |\langle\mathbf{w}_{(m,r)}^{(t)}, \mathbf{v}_k^*\rangle|  \\
            & \leq  |\Gamma_{(k, r)}^{(0,k)}| + |\frac{\eta}{nm} \sum_{j \in[n]} \sum_{p \in[k]} \sum_{s \in[T_k]} 3\alpha_p^3 A_{(p,k)} \ell_{p j}(\mathbf{W}_m^{(s-1)}) (\Gamma_{(k, r)}^{(s-1,p)})^2|   \\
            & = |\Gamma_{(0, r)}^{(0,k)}| + |\frac{\eta}{nm} \sum_{j \in[n]} \sum_{q \in [k]}\sum_{p \in[q]} \sum_{s \in[T_q]} 3\alpha_p^3 A_{(p,k)} \ell_{p j}(\mathbf{W}_q^{(s-1)}) (\Gamma_{(q, r)}^{(s-1,p)})^2|\\
            & \leq \widetilde{O}(\sigma_0) + |\frac{{\eta} }{m}\sum_{q \in [k]}\sum_{p \in[q]} 3\alpha_p^3 A_{(p,k)} T_q  \widetilde{O}(\sigma_0^2) |\\
            & \overset{(i)}{\leq} \widetilde{O}(\sigma_0) + |\frac{\eta}{m} T_v \widetilde{O}(\sigma_0^2)  \cdot \sum_{p=1}^k (k-p+1) \alpha_p^3 A_{(p, k)}|  \\
            & \overset{(ii)}{\leq} \widetilde{O}(\sigma_0).
        \end{aligned}
    \end{equation}
    Here, $(i)$ follows from the assumption that every task before $k$ is trained for the same number of iterations $T_v$; $(ii)$ drives from the choice of $T_v \leq \widetilde{O} \left( \frac{m}{\eta \sigma_0  \sum_{p=1}^k (k-p+1) \alpha_p^3 A_{(p, k)}} \right) $.
\end{proof}

\begin{lemma}\label{lem:m_yPhi_leq_Rd}
    Let $T_{\xi}^{-} = \frac{nm}{\eta \sigma_0 (\sigma_{\xi} \sqrt{d})^3} $. In the data replay training process on task $k$, with probability at least $1-1 /\operatorname{poly}(d)$, it holds that:
    $$
    \max _{r \in[R], p \in[k], j \in[n]} y_{mj}\Phi_{(k, r)}^{(t,p,j)} \leq (Rd)^{-1/3} \quad \text { for any } \quad t \leq \mathrm{T}_{\xi}^{-}, p \in [k].
    $$ 
\end{lemma}

\begin{proof}[\bf Proof of \cref{lem:m_yPhi_leq_Rd}]
   We first assume \cref{lem:m_yPhi_leq_Rd} holds for any $t \leq T_{\xi}^{-} -1$, then the following can be obtained:
    \begin{equation*}
        \begin{aligned}
        \ell_{p j}^{(t)} & =\frac{1}{1+\exp \{\sum_{r=1}^R[\alpha_p^3(\Gamma_{(k, r)}^{(t,p)})^3+(y_{p j} \Phi_{(k, r)}^{(t,p,j)})^3]\}} \\
        & \overset{(i)}{\geq} \frac{1}{1+\exp \{\widetilde{O}(d^{-1})+\widetilde{O}(\alpha_p^3 R \widetilde{O}(\sigma_0^3)\}} \\
        & \overset{(ii)}{\geq} \frac{1}{1+\exp \{  \widetilde{O}(d^{-1}) + \widetilde{O}(d^{-3/2})\}} \\
        & \geq \frac{1}{2}-\frac{e^{2 d^{-1}}-1}{2(1+e^{2 d^{-1}})} \\
        & = \frac{1}{2} - \widetilde{O}(d^{-1}),
        \end{aligned}
    \end{equation*}
    where the inequality $(i)$ derives from the induction hypothesis and \cref{lem:max_Gamma_task_k} and $(ii)$ holds due to the Condition~\ref{con:parameter} and SNR choices.
    
    Therefore, using recursion \cref{eq:update_noise_data_replay}, with high probability $1 - 1/\operatorname{poly}(d)$, we have
    \begin{equation}\label{eq:yPhi}
        \begin{aligned}
            y_{kj}\Phi_{(k, r^*)}^{(t+1,k,j)} &= y_{kj}\Phi_{(k, r^*)}^{(t,k,j)} +   \sum_{p=1}^{k}  \frac{ 3\eta }{nm} \sum_{j^{\prime} \in[n]} y_{kj}  y_{p j^{\prime}} \ell_{p j^{\prime}}(\mathbf{W}_k^{(t)}) (\Phi_{(k, r)}^{(t,p,j)})^2 \langle \bm{\xi}_{p j^{\prime}},\bm{\xi}_{k j} \rangle \\
            y_{kj}\Phi_{(k, r^*)}^{(t+1,k,j)} &= y_{kj}\Phi_{(k, r^*)}^{(0,k,j)} +   \sum_{p=1}^{k} \sum_{s=1}^{t-1} \frac{ 3\eta }{nm} \sum_{j^{\prime} \in[n]} y_{kj}  y_{p j^{\prime}} \ell_{p j^{\prime}}(\mathbf{W}_k^{(s)}) (\Phi_{(k, r)}^{(s,p,j)})^2 \langle \bm{\xi}_{p j^{\prime}},\bm{\xi}_{k j} \rangle \\
            y_{kj}\Phi_{(k, r^*)}^{(t+1,k,j)} &= y_{kj}\Phi_{(0, r^*)}^{(0,k,j)} +   \sum_{q=1}^{k} \sum_{p=1}^{q} \sum_{s=1}^{t-1} \frac{ 3\eta }{nm} \sum_{j^{\prime} \in[n]} y_{kj}  y_{p j^{\prime}} \ell_{p j^{\prime}}(\mathbf{W}_k^{(s)}) (\Phi_{(k, r)}^{(s,p,j)})^2 \langle \bm{\xi}_{p j^{\prime}},\bm{\xi}_{k j} \rangle \\
            y_{kj}\Phi_{(k, r^*)}^{(t+1,k,j)} & \overset{(i)}{=} y_{kj}\Phi_{(0, r^*)}^{(0,k,j)} + \Theta\left(\frac{3\eta{d} \sigma_{\xi}^2}{nm}\right) \sum_{s=1}^{t-1} \left(y_{kj}\Phi_{(k, r^*)}^{(s,k,j)}\right)^2 \\
            &\pm \Theta\left(\frac{3\eta \sqrt{d} \sigma_{\xi}^2}{nm}\right) \sum_{q=1}^{k} \sum_{\substack{p \in [q],\, j^{\prime} \in [n] \\ (p, j^{\prime}) \neq (k, j)}} \sum_{s=1}^{t-1} \ell_{p j^{\prime}}(\mathbf{W}_k^{(s)}) \left(y_{kj}\Phi_{(k, r^*)}^{(s,p,j)}\right)^2 \\
            & \overset{(ii)}{=} y_{kj}\Phi_{(0, r^*)}^{(0,k,j)} + \Theta\left(\frac{3\eta{d} \sigma_{\xi}^2}{nm}\right) \sum_{s=1}^{t-1} \left(y_{kj}\Phi_{(k, r^*)}^{(s,p,j)}\right)^2 \pm \widetilde{O}\left(\frac{3 \sqrt{d} \sigma_{\xi}^2 k^2}{ \sigma_0 \sum_{p \in [k]} 3 \alpha_p^3 A_{(p,k)}} \cdot \frac{1}{(Rd)^{2/3}}\right) 
        \end{aligned}
    \end{equation}
    where equality $(i)$ holds due to \cref{lem:xi_bound} and $(ii)$ comes from \cref{lem:k_sum_l}. Let ${T}_{\xi_{p j}}^{-}=\Theta(\frac{nm}{\eta M_{p j}})$ with $M_{p j}=(y_{kj}\Phi_{(k, r^*)}^{(T_1,p,j)} \sqrt{d} \sigma_{\xi}^2)$. Then, according to \cref{lem:E7}, it holds $ y_{kj}\Phi_{(k, r^*)}^{(T_1,p,j)} \leq(R d)^{-1 / 3}$ for any $t \leq {T}_{\xi_{p j}}^{-}$ since $(R d)^{-1 / 3}\geq \sigma_0\sigma_{\xi}\sqrt{d} \geq 2 y_{kj}\Phi_{(k, r)}^{(T_1,p,j)} $. Moreover, we also know $\max _{i j} {~T}_{\xi_{p j}}^{-}+1 \leq {T}_{\xi}^{-}$ by concentration, which indicates that $(y_{kj}\Phi_{(k, r^*)}^{(t,p,j)}  \leq(R d)^{-1 / 3}$.
\end{proof}

\begin{lemma}\label{lem:k_sum_l}
    Given any $p \in [k]$ and $k \in [M]$, with high probability $1 - 1/\operatorname{poly}(d)$, it holds that: 
    $$ \sum_{s=1}^t \frac{1}{n} \sum_{j^{\prime}}  \ell_{p j^{\prime}}(\mathbf{W}_k^{(s)})  \leq \widetilde{O}\left(\frac{m}{ \eta \sigma_0 \sum_{p \in [k]} \alpha_p^3 A_{(p,k)} } \right).$$
\end{lemma}

\begin{proof}[\bf Proof of \cref{lem:k_sum_l}]
    Applying \cref{lem:E8} with $z^{(0)}=\Gamma_{ k,r^*}^{(\tau_{kv}^k,k)}, h=H=\frac{3 \eta}{nm} \sum_{j^{\prime} \in [n]} \sum_{p \in [k]} \alpha_p^3 A_{(p,k)}$ and $\frac{1}{nm} \sum_{j^{\prime}} \ell_{p j^{\prime}} \leq \frac{1}{m}$, then it holds that 
    \begin{equation}
        \sum_{s=1}^t \frac{1}{n} \sum_{j^{\prime}}  \ell_{p j^{\prime}}(\mathbf{W}_k^{(t)})  \leq \log d + \widetilde{O}\left(\frac{m}{ \eta \sigma_0 \sum_{p \in [k]} \alpha_p^3 A_{(p,k)} }\right). 
    \end{equation}
\end{proof}

\begin{lemma}[Restatement of Lemma~\ref{lem:noise_ff}]\label{lem:m_yPhi_leq_R}
     Suppose the SNR condition satisfying
    $ \frac{k^2}{R^{2/3} \sigma_0^2 \sigma_{\xi}^2 d^{13/6}} \lesssim \frac{\sum_{p=1}^{k} (1-\frac{p-1}{k}) \alpha_p^3 A_{(p,k)}}{(\sigma_{\xi}\sqrt{d})^3} \lesssim \frac{1}{n} $, and there exists an iteration $\tau_{k j}^k \leq \mathrm{T}_{\xi}^k=\mathrm{T}_{\xi}^{-}+O(\log (d))$ such that $\tau_{k j}^k$ is the first iteration for which $\max _{r \in[R]}(y_{k j} \Phi_{(k, r)}^{(t,k,j)}) \geq \Theta(R^{-\frac{1}{3}})$, and for any $t \leq \mathrm{T}_{\xi}^k$ it holds that $\max _{r \in[R]}|\Gamma_{(k,r)}^{(t,k)}| \leq \widetilde{O}(\sigma_0)$.
    Then, if the additional SNR condition $\frac{m^2-k^2}{R^{1/3} \sigma_0 \sigma_{\xi} d} \lesssim \frac{\sum_{p=1}^{m} (1-\frac{p-1}{m}) \alpha_p^3 A_{(p,k)}}{(\sigma_{\xi}\sqrt{d})^3} \lesssim \frac{1}{n} $ also holds, there exists an iteration $\tau_{k j}^m$ such that $\tau_{k j}^m$ is the first iteration satisfying $\max _{r \in[R]}(y_{k j} \Phi_{(m, r)}^{(t,k,j)}) \geq \Theta(R^{-\frac{1}{4}})$. In this case, we can also guarantee that $\max _{r \in[R]}|\Gamma_{(m,r)}^{(t,k)}| \leq \widetilde{O}(\sigma_0)$ for any $t \leq \mathrm{T}_{\xi}^m$.
\end{lemma}

\begin{proof}[\bf Proof of \cref{lem:m_yPhi_leq_R}]
    According to the definition of $\tau_{kj}^{k}$, it is clear that $\max _{r \in[R]} y_{k j} \Phi_{(k, r)}^{(s,k,j)} \leq \Theta(R^{-\frac{1}{3}})$ for $s \leq \tau_{kj}^{k}$. Furthermore, it holds that $ \tau_{kj}^{k} \geq T_{\xi}^{-}$ due to \cref{lem:m_yPhi_leq_Rd}. For any $s \leq \tau_{kj}^{k}$, we also have
    \begin{equation}
        \begin{aligned}
        \ell_{k j}^{(s)} & =\frac{1}{1+\exp \{\sum_{r=1}^R[\alpha_k^3(\Gamma_{(k, r)}^{(s,k)})^3+(y_{k j} \Phi_{(k, r)}^{(s,k,j)})^3]\}} \\
        & \geq \frac{1}{1+\exp \{R \Theta(1 / R)+\widetilde{O}(\alpha_k^3 R  \sigma_0^3)\}} \\
        & \geq \frac{1}{1+\exp \{\Theta(1)\}} \\
        & =\Theta(1).
\end{aligned}
    \end{equation}

    Let $\tau_{r^*, k j}^{-}$ be the first iteration such that $y_{k j} \Phi_{(k,r^*)}^{(t,k,j)} \geq \Theta((R d)^{-\frac{1}{3}})$, then it follows $\tau_{r^*, k j}^{-}>\mathrm{T}_{\xi}^{-}$. After enrolling update rule with $r=r_{k j}^*$, for any $\tau_{r^*, k j}^{-} \leq t \leq \min \{\tau_{k j}^k, T_{\xi}^k\}$, it holds that
    \begin{equation*}
        \begin{aligned}
            y_{kj}\Phi_{(k, r^*)}^{(t+1,k,j)} &= y_{kj}\Phi_{(k, r^*)}^{(\tau_{r^*, k j}^{-}, k,j)} +\frac{3\eta }{nm} \sum_{p=1}^k \sum_{s= \tau_{r^*, k j}^{-}}^{t-1} \frac{ 3\eta }{nm} \sum_{j^{\prime} \in[n]} y_{kj}  y_{p j^{\prime}} \ell_{p j^{\prime}}(\mathbf{W}_k^{(s)}) (\Phi_{(k, r)}^{(s,p,j)})^2 \langle \bm{\xi}_{p j^{\prime}},\bm{\xi}_{k j} \rangle \\
            & \overset{(i)}{=} y_{kj}\Phi_{(k, r^*)}^{((\tau_{r^*, k j}^{-},k,j),k,j)} + \Theta\left(\frac{3\eta{d} \sigma_{\xi}^2}{nm}\right)\sum_{s= \tau_{r^*, k j}^{-}}^{t-1} \left(y_{kj}\Phi_{(k, r^*)}^{(s,k,j)}\right)^2 \\
            & \pm \Theta\left(\frac{3 \sqrt{d} \sigma_{\xi}^2 k^2}{ \sigma_0 \sum_{p \in [k]} 3 \alpha_p^3 A_{(p,k)}} \cdot \frac{1}{(Rd)^{2/3}}\right) \\
            & \overset{(ii)}{=} y_{kj}\Phi_{(k, r^*)}^{((\tau_{r^*, 1 j}^{-},k,j),k,j)} + \Theta\left(\frac{3\eta{d} \sigma_{\xi}^2}{nm}\right)\sum_{s= \tau_{r^*, k j}^{-}}^{t-1} \left(y_{kj}\Phi_{(k, r^*)}^{(s,k,j)}\right)^2 \pm o\left(R^{-\frac{1}{3}} d^{-\frac{1}{3}}\right). 
        \end{aligned}
    \end{equation*}

    The inequality $(i)$ holds due to Lemma~\ref{lem:xi_bound} and Lemma~\ref{lem:k_sum_l} and $(ii)$ comes from SNR choices.
    
    Let $A=\Theta(\frac{\eta d \sigma_{\xi}^2}{nm}), C=o(R^{-\frac{1}{3}} d^{-\frac{1}{3}}),  v=\Theta(R^{-\frac{1}{3}}).$ 
    By applying the tensor power method via Lemma~\ref{lem:K16}, we have:
    \begin{equation*}
        \begin{aligned} 
        \tau_{r^*, k j}^k & \leq \tau_{r^*, k j}^{-}+\frac{21}{A y_{k j} \Phi_{r^*, k j}^{\left(\tau_{r^*, k j}^{-}\right)}}+8\left[\frac{\log \left(v /\left[y_{i j} \Phi_{r^*, k j}^{\left(\tau_{r^*, k j}^{-}\right)}\right]\right)}{\log (2)}\right] \\ 
        & \leq \Theta\left(\frac{1}{\eta} \frac{nm}{\left(\sqrt{d} \sigma_{\xi}\right)^3 \sigma_0}\right)+\Theta\left(\frac{1}{\eta} \frac{ nm R^{1 / 3}}{d^{2 / 3} \sigma_{\xi}^2}\right)+O(\log d) \\ 
        & \leq O\left(\frac{1}{\eta} \frac{mn}{\left(\sqrt{d} \sigma_{\xi}\right)^3 \sigma_0}+ \frac{1}{\eta} \frac{mn R^{1 / 3}}{d^{2 / 3} \sigma_{\xi}^2}+ \log d\right) = T_{\xi}^k.
        \end{aligned}
    \end{equation*}
    Next, we will show that the above also holds when training task $m$ for the first scenario. First, we have:
    \begin{equation*}
        \begin{aligned}
        \ell_{m j}^{(t)} & =\frac{1}{1+\exp \{\sum_{r=1}^R[\alpha_m^3(\Gamma_{(m, r)}^{(t,m)})^3+(y_{m j} \Phi_{(m, r)}^{(t,m,j)})^3]\}} \\
        & \overset{(i)}{\geq} \frac{1}{1+\exp \{\Theta(1)+\widetilde{O}(\alpha_m^3 R \widetilde{O}(\sigma_0^3)\}} \\
        & {\geq} \Theta(1).
        \end{aligned}
    \end{equation*}
    Then, when training task $m \geq k$, noise memorization satisfies:
    \begin{equation}\label{eq:yPhi_2_m}
        \begin{aligned}
            y_{kj}\Phi_{(m, r^*)}^{(t+1,k,j)} 
            &= y_{kj}\Phi_{(m, r^*)}^{(t,k,j)} +   \sum_{p=1}^{k}  \frac{ 3\eta }{nm} \sum_{j^{\prime} \in[n]} y_{kj}  y_{p j^{\prime}} \ell_{p j^{\prime}}(\mathbf{W}_m^{(t)}) (\Phi_{(m, r)}^{(t,p,j)})^2 \langle \bm{\xi}_{p j^{\prime}},\bm{\xi}_{k j} \rangle \\
             &= y_{kj}\Phi_{(m, r^*)}^{(0,k,j)} +   \sum_{p=1}^{m} \sum_{s=1}^{t-1} \frac{ 3\eta }{nm} \sum_{j^{\prime} \in[n]} y_{kj}  y_{p j^{\prime}} \ell_{p j^{\prime}}(\mathbf{W}_m^{(s)}) (\Phi_{(m, r)}^{(s,p,j)})^2 \langle \bm{\xi}_{p j^{\prime}},\bm{\xi}_{k j} \rangle \\
        \end{aligned}
    \end{equation}
    Then, according to Lemma~\ref{lem:xi_bound}, it also holds that: 
    \begin{equation}\label{eq:yPhi_2_m2}
        \begin{aligned}
            y_{kj}\Phi_{(m, r^*)}^{(t+1,k,j)} & {=} y_{kj}\Phi_{(k+1, r^*)}^{(0,k,j)} + \Theta\left(\frac{3\eta{d} \sigma_{\xi}^2}{nm}\right) \sum_{s=1}^{t-1} \left(y_{kj}\Phi_{(m, r^*)}^{(s,m,j)}\right)^2 \\
            &\pm \Theta\left(\frac{3\eta \sqrt{d} \sigma_{\xi}^2}{nm}\right) \sum_{q=k}^{m} \sum_{\substack{p \in [q],\, j^{\prime} \in [n] \\ (p, j^{\prime}) \neq (k, j)}} \sum_{s=1}^{t-1} \ell_{p j^{\prime}}(\mathbf{W}_m^{(s)}) \left(y_{kj}\Phi_{(m, r^*)}^{(s,p,j)}\right)^2 \\
            & \overset{(i)}{=} y_{kj}\Phi_{(k+1, r^*)}^{(0,k,j)} + \Theta\left(\frac{3\eta{d} \sigma_{\xi}^2}{nm}\right) \sum_{s=1}^{t-1} \left(y_{kj}\Phi_{(m, r^*)}^{(s,p,j)}\right)^2 \pm \widetilde{O}\left(\frac{3 \sqrt{d} \sigma_{\xi}^2 (m^2 - k^2)}{ \sigma_0 \sum_{p \in [m]} 3 \alpha_p^3 A_{(p,k)}} \cdot \frac{1}{(R)^{2/3}}\right) \\
            & \overset{(ii)}{=} y_{kj}\Phi_{(k+1, r^*)}^{(0,k,j)} + \Theta\left(\frac{3\eta{d} \sigma_{\xi}^2}{nm}\right) \sum_{s=T_1}^t\left(y_{1j}\Phi_{}^{(s,1,j)}\right)^2 \pm o\left( {R^{-1/3}}\right). 
        \end{aligned}
    \end{equation}
    Here, $(i)$ follows from Lemma~\ref{lem:k_sum_l} with the range of $p \in[m]$ adjusted accordingly; $(ii)$ is derived from the robustness of the SNR choices.
    
    Let $A=\Theta(\frac{\eta d \sigma_{\xi}^2}{mn}), C=o(R^{-\frac{1}{3}} ),  v=\Theta(R^{-\frac{1}{4}}).$ 
    By applying the tensor power method via Lemma~\ref{lem:K16}, we also have:
    \begin{equation*}
        \begin{aligned} 
        \tau_{r^*, k j}^m & \leq T_k+\frac{21}{A y_{k j} \Phi_{r^*, k j}^{\left(T_k\right)}}+8\left[\frac{\log \left(v /\left[y_{k j} \Phi_{r^*, k j}^{\left(T_k\right)}\right]\right)}{\log (2)}\right] \\ 
        & \leq \Theta\left(\frac{1}{\eta} \frac{mn}{\left(\sqrt{d} \sigma_{\xi}\right)^3 \sigma_0}\right)+ \Theta\left(\frac{1}{\eta} \frac{ nm R^{1 / 3}}{d^{2/3} \sigma_{\xi}^2}\right)  + \Theta\left(\frac{1}{\eta} \frac{ n m R^{1 / 3}}{d \sigma_{\xi}^2}\right) + \log d\\ 
        & \leq \Theta\left(\frac{1}{\eta} \frac{ mn}{\left(\sqrt{d} \sigma_{\xi}\right)^3 \sigma_0}+ \frac{1}{\eta} \frac{ mn R^{1 / 3}}{d^{2 / 3} \sigma_{\xi}^2} + \frac{1}{\eta} \frac{ nm R^{1 / 3}}{d \sigma_{\xi}^2}  \right)+ \log d = T_{\xi}^m.
        \end{aligned}
    \end{equation*}
\end{proof}

\begin{lemma}\label{lem:m_yPhi_geq_d}
    For any $ 0 \leq t \leq T_{\xi}^{-}$, with probability at least $1-1 /\operatorname{poly}(d)$, it holds that:
    $$
    \min _{r \in[R], m \in[M], p,k \in [m], j \in[n]} y_{kj}\Phi_{(m, r)}^{(t,p,j)} \geq -(d)^{-1/2}.
    $$ 
\end{lemma}

\begin{proof}[\bf Proof of Lemma~\ref{lem:m_yPhi_geq_d}]
    % We first consider training phase 1 (before training task $k+1$).
    According to Lemma~\ref{lem:w0_v_xi}, we know $\min _{r \in[m]} y_{k j} \Phi_{(r, 0)}^{(0,k,j)} \geq-\widetilde{O}(\sqrt{d} \sigma_{\xi} \sigma_0)$ holds for any $j \in[n]$. By Lemma~\ref{lem:m_yPhi_leq_Rd}, we know $\ell_{k j}^{(s)} \geq \frac{1}{2}-O\left(d^{-1}\right)$ for any $s \leq \mathrm{T}_{\xi}^{-}$. Similar to Lemma~\ref{eq:yPhi}, we can obtain that for any $t \leq \mathrm{T}_{\xi}^{-}$,
    $$
    \begin{aligned}
    y_{kj}\Phi_{(k, r)}^{(t+1,k,j)} & =y_{kj}\Phi_{(0, r^*)}^{(0,k,j)} + \Theta\left(\frac{3\eta{d} \sigma_{\xi}^2}{nm}\right) \sum_{s=1}^{t-1} \left(y_{kj}\Phi_{(k, r^*)}^{(s,p,j)}\right)^2 \pm o\left(\sqrt{d} \sigma_{\xi} \sigma_0\right) \\
    & \overset{(i)}{\geq}-\widetilde{O}\left(\sqrt{d} \sigma_{\xi} \sigma_0\right)-o\left(\sqrt{d} \sigma_{\xi} \sigma_0\right) \\
    & =-\widetilde{O}\left(\sqrt{d} \sigma_{\xi} \sigma_0\right),
    \end{aligned}
    $$
    where $(i)$ holds due to the second term being always positive.
\end{proof}

\begin{lemma}[Restatement of Lemma~\ref{lem:noise_fs}]\label{lem:noise_fs_app}
    Suppose the SNR satisfying $ \frac{k^2}{R^{2/3} \sigma_0^2 \sigma_{\xi}^2 d^{13/6}} \lesssim \frac{\sum_{p=1}^{k} (1-\frac{p-1}{k}) \alpha_p^3 A_{(p,k)}}{(\sigma_{\xi}\sqrt{d})^3} \lesssim \frac{1}{n} $, and there exists an iteration $\tau_{k j}^k \leq \mathrm{T}_{\xi}^k=\mathrm{T}_{\xi}^{-}+O(\log (d))$ such that $\tau_{k j}^k$ is the first iteration where $\max _{r \in[R]}(y_{k j} \Phi_{(k, r)}^{(t,k,j)}) \geq \Theta(R^{-\frac{1}{3}})$, and for any $t \leq \mathrm{T}_{\xi}^k$ it holds that $\max _{r \in[R]}|\Gamma_{(k,r)}^{(t,k)}| \leq \widetilde{O}(\sigma_0)$.
    Then, if the additional SNR condition $\frac{\sum_{p=1}^{m}  \alpha_p^3 A_{(p,k)}}{(\sigma_{\xi}\sqrt{d})^3} \gtrsim \frac{1}{n R^{1/3} \sigma_0 \sigma_{\xi}\sqrt{d} } $ also holds, there exists $\tau_{k v}^m \leq \mathrm{T}_{v}^m=\mathrm{T}_{v}^{k}+O(\log (d))$ such that  $\tau_{k v}^m$ be the first iteration satisfying $\max _{r \in[R]}|\Gamma_{(m,r)}^{(t,k)}| \geq \Theta(\frac{1}{\alpha_k R^{1/5}})$.  
\end{lemma}

\begin{proof}[\bf Proof of Lemma~\ref{lem:noise_fs_app}]
    The proof of the first part of Lemma~\ref{lem:noise_fs_app} follows directly from Lemma~\ref{lem:m_yPhi_leq_R} for the initial training phase. Therefore, we focus on the second training phase, beginning with the analysis of enhanced signal learning, followed by a demonstration that noise memorization remains controlled under certain SNR conditions.
    
    It is clear that before $T_k = k T_v$ in Lemma~\ref{lem:max_Gamma_task_k}, we have $\max _{r \in[R], p \in [k]}|\Gamma_{(k,r)}^{(t,p)}| \leq \widetilde{O}(\sigma_0) $ for any $t \in [T_k], p \in [k]$. Then, according to \cref{eq:update_signal_data_replay}, we have the following since $T_k$:
    \begin{equation}
        \begin{aligned}
            \Gamma_{(m, r^*)}^{(t,k)} 
             &=  \Gamma_{(m, r^*)}^{(t-1,k)} + \frac{\eta}{nm} \sum_{j \in[n]} \sum_{p \in[m]}  3\alpha_p^3 A_{(p,k)} \ell_{p j}(\mathbf{W}_m^{(t-1)}) (\Gamma_{(m, r^*)}^{(t-1,p)})^2  \\
             &= \Gamma_{(m, r^*)}^{(t-1,k)} + \Theta\left( \frac{\eta}{m} \sum_{p \in[m]}  3\alpha_p^3 A_{(p,k)}  \right)  (\Gamma_{(m, r^*)}^{(t-1,p)})^2 .
        \end{aligned}
    \end{equation}
    Then, by applying the tensor power method from Lemma~\ref{lem:K15_extension} to the sequence $\{\Gamma_{(m, r^*)}^{(s,k)}\}_{s \geq T_k}$, let $h=H= {3\frac{\eta}{m} (\sum_{p \in[m]}  3\alpha_p^3 A_{(p,k)}) }$, $z^{(0)}= \Gamma_{(k+1, r^*)}^{(0,k)} \geq O(\sigma_0)$, $ v=\Theta(\frac{1}{\alpha_k R^{1/5}}),$ then we obtain:
    \begin{equation*}
        \begin{aligned} 
        \tau_{k v}^m & \leq T_k+\frac{m}{3\eta \sigma_0 \sum_{p \in[m]}  3\alpha_p^3 A_{(p,k)} }+8\left[\frac{\log \left(v /z^{(0)}\right)}{\log (2)}\right] \\ 
        & \leq O\left(\frac{1}{\eta} \frac{ mn}{\left(\sqrt{d} \sigma_{\xi}\right)^3 \sigma_0} + \frac{1}{\eta} \frac{m n R^{1 / 3}}{d^{2 / 3} \sigma_{\xi}^2}\right) + \frac{m}{3\eta \sigma_0 \sum_{p \in[m]}  3\alpha_p^3 A_{(p,k)}} + \log d \\ 
        & \leq O\left(\frac{1}{\eta} \frac{m n}{\left(\sqrt{d} \sigma_{\xi}\right)^3 \sigma_0} + \frac{1}{\eta} \frac{m n R^{1 / 3}}{d^{2 / 3} \sigma_{\xi}^2} + \frac{m}{3\eta \sigma_0 \sum_{p \in[m]}  3\alpha_p^3 A_{(p,k)}} \right) + \log d  = T_{v}^m.
        \end{aligned}
    \end{equation*}
    Then, it is noticed that if the additional SNR condition $\frac{\sum_{p=1}^{m}  \alpha_p^3 A_{(p,k)}}{(\sigma_{\xi}\sqrt{d})^3} \gtrsim \frac{1}{n R^{1/3} \sigma_0 \sigma_{\xi}\sqrt{d} } $ also holds, we will have $T_v^m \leq T_{\xi}^m$ according to Lemma~\ref{lem:m_yPhi_leq_R}, which indicates that noise memorization remain controlled within $\Theta(R^{-1/4})$ and is slower than the signal learning in the second training phase.
\end{proof}

\begin{theorem}[Restatement of Theorem~\ref{thm:forget_m_ffs_tasks}]\label{thm:forget_m_ffs_tasks_app}
    Suppose the setting in Condition~\ref{con:parameter} holds, and the SNR satisfies $ \frac{k^2}{R^{2/3} \sigma_0^2 \sigma_{\xi}^2 d^{13/6}} \lesssim \frac{\sum_{p=1}^{k} (1-\frac{p-1}{k}) \alpha_p^3 A_{(p,k)}}{(\sigma_{\xi}\sqrt{d})^3} \lesssim \frac{1}{n} $. Consider full data-replay training with learning rate $\eta \in(0, \widetilde{O}(1)]$, and let $(\mathbf{x}_k, y_k) \sim \mathcal{D}_k$ be a test sample from the task $k$. Then, with high probability, there exist training times $T_k$ and $T_m$ ($m>k$) such that 
    \begin{itemize}
        \item The model fails to correctly classify task $k$ immediately after learning it:
        \begin{equation}
            \mathbb{P}\left\{y_k F\left(\mathbf{W}^{(T_k)}, \mathbf{x}_k\right)<0\right\} \geq \frac{1}{2}-{\frac{1}{\operatorname{polylog}(d)}}.
        \end{equation}
        \item (Persistent Learning Failure on Task $k$) If the additional SNR condition holds $\frac{m^2-k^2}{R^{1/3} \sigma_0 \sigma_{\xi} d} \lesssim \frac{\sum_{p=1}^{m} (1-\frac{p-1}{m}) \alpha_p^3 A_{(p,k)}}{(\sigma_{\xi}\sqrt{d})^3} \lesssim \frac{1}{n} $, then the model still fails to correctly classify task $k$ after subsequent training to task $m$: 
        \begin{equation}
            \mathbb{P}\left\{y_k F\left(\mathbf{W}^{(T_m)}, \mathbf{x}_k\right)<0\right\} \geq \frac{1}{2}-{\frac{1}{\operatorname{polylog}(d)}}.
        \end{equation}
        \item (Enhanced Signal Learning on Task $k$) If the additional SNR conditions holds $\frac{\sum_{p=1}^{m}  \alpha_p^3 A_{(p,k)}}{(\sigma_{\xi}\sqrt{d})^3} \gtrsim \frac{1}{n R^{1/3} \sigma_0 \sigma_{\xi}\sqrt{d} } $, then the model can correctly classify task $k$ after subsequent training to task $m$:
        \begin{equation}
            \mathbb{P}\left\{y_k F\left(\mathbf{W}^{(T_m)}, \mathbf{x}_k\right)<0\right\} \leq {\frac{1}{\operatorname{poly}(d)}}.
        \end{equation}
    \end{itemize}
\end{theorem}

\begin{proof}[\bf Proof of Theorem~\ref{thm:forget_m_ffs_tasks_app}]
    We first prove the training phase one (before training task $k+1$).
    For the new test data $(\mathbf{x}_k, y_k) \sim \mathcal{D}_k$, with probability at least $1-1 /$ poly $(d)$, we have
    \begin{equation}\label{eq:f_phase1}
        \begin{aligned}
        y_k F(\mathbf{W}^{(T_k)}, \mathbf{x}_k) & =y_k \sum_{r \in[R]} (\langle\mathbf{w}_r^{(T_k)}, \mathbf{x}_k^1\rangle^3 + \langle\mathbf{w}_r^{(T_k)}, \mathbf{x}_k^2\rangle^3)\\
        & =y_k \sum_{r \in[R]} (\langle\mathbf{w}_r^{(T_k)}, \alpha_k y_k \mathbf{v}_k^*\rangle^3 + \langle\mathbf{w}_r^{(T_k)}, \bm{\xi}_k \rangle^3)\\
        & =\sum_{r \in[R]}[\alpha_k^3(\Gamma_{(k, r)}^{(T_k)})^3+y\langle\mathbf{w}_r^{(T_k)}, \bm{\xi}_k\rangle^3] \\
        & \leq \sum_{r \in[R]} y_k\langle\mathbf{w}_r^{(T_k)}, \bm{\xi}_k\rangle^3+ \widetilde{O}(R \alpha_k^3 \sigma_0^3) 
        \end{aligned}
    \end{equation}
    Let $\mathbf{P}_v=\sum_{m \in[M]} \mathbf{v}_m^*\left(\mathbf{v}_m^*\right)^{\top}$ and $\mathbf{P}_v^{\perp}=\mathbf{I}_d-\mathbf{P}_v$. Since $\bm{\xi} \sim \mathcal{N}\left(0, \sigma_{\xi}^2 \mathbf{P}_v^{\perp}\right)$, there exists a vector $\bm{\xi}_d \sim \mathcal{N}\left(0, \sigma_{\xi}^2 \mathbf{I}_d\right)$ such that $\bm{\xi}=\mathbf{P}_v^{\perp} \bm{\xi}_d$.
    Now, decompose $\mathbf{w}_r^{(T_k)}$ as: $\mathbf{w}_r^{(T_k)}=\mathbf{P}_v \mathbf{w}_r^{(T_k)}+\mathbf{P}_v^{\perp} \mathbf{w}_r^{(T_k)}.$ According to the definition of $\Phi_{(k, r)}^{(T_k,k,j)}=\left\langle\mathbf{w}_r^{(T_k)}, \bm{\xi}_{k j}\right\rangle=\left\langle\mathbf{P}_v^{\perp} \mathbf{w}_r^{(T_k)}, \bm{\xi}_{k j}\right\rangle$ and Lemma~\ref{lem:m_yPhi_leq_R}, for Task $1$'s data $(\mathbf{x}_k,y_k)$, we have
    $$
    \Theta\left(R^{-\frac{1}{3}}\right) \leq \max _{r \in[R]} y_{k j} \Phi_{(k, r)}^{(T_k,k,j)}=\max _{r \in[m]}\left\langle\mathbf{P}_v^{\perp} \mathbf{w}_r^{(T_k)}, y_{k j} \bm{\xi}_{k j}\right\rangle .
    $$
    Denote $r^* = \arg \max y_{k j}\Phi_{(k, r)}^{(T_k,k,j)}$, then it holds that 
    \begin{equation}
    \begin{aligned}
    \sum_{r \in[R]}\left\langle\mathbf{P}_v^{\perp} \mathbf{w}_r^{(T_k)}, \frac{y_{k j} \bm{\xi}_{k j}}{\left\|\bm{\xi}_{k j}\right\|}\right\rangle^3 & \geq \frac{1}{\left\|\bm{\xi}_{k j}\right\|^3} \left[\left( y_{k j}\Phi_{k,r^*}^{(T_k,k,j)}\right)^3-\sum_{r \neq r^*}\left(y_{k j} \Phi_{(k, r)}^{(T_k,k,j)}\right)^3\right] \\
    & \overset{(i)}{\geq} \tilde{\Omega}\left(\frac{1}{d^{3 / 2} \sigma_{\xi}^3}\right)\left[\Theta\left(R^{-1}\right)-\widetilde{O}\left(R \left(d^{-1/2} \right)^3\right)\right] \\
    & =\widetilde{\Omega}\left(\frac{1}{d^{3 / 2} \sigma_{\xi}^3}\right) \overset{(ii)}{\geq}  1 .
    \end{aligned}
    \end{equation}
    Here, $(i)$ comes from Lemma~\ref{lem:m_yPhi_geq_d} and Lemma~\ref{lem:m_yPhi_leq_R} and $(ii)$ holds due to the assumption on $\sigma_{\xi}$.
    Given that the model $\mathbf{W}^{(T_k)}$ and the test label $y_k$ are independent of the noise $\boldsymbol{\xi}_d$, it follows that the distribution of $\sum_{r \in[R]} y_k\langle\mathbf{P}_v^{\perp} \mathbf{w}_r^{(T_k)}, \boldsymbol{\xi}_d\rangle^3$ is symmetric. This holds under the condition that $\mathbf{W}^{(T_k)}$ and $y \boldsymbol{\xi}_d$ are distributed as $\mathcal{N}(0, \sigma_{\xi}^2 \mathbf{I}_d)$, where $y \in\{-1,+1\}$.
    According to Lemma~\ref{lem:K12}, let $\mathbf{w}_r=\mathbf{P}_v^{\perp} \mathbf{w}_r^{(T_k)}$ and $\boldsymbol{u}=y_{k j} \boldsymbol{\xi}_{k j} /\|\boldsymbol{\xi}_{k j}\|$, then we derive:
    \begin{equation}
        \begin{aligned}        &\mathbb{P}_{\boldsymbol{\xi}_d}\left(\sum_{r \in[R]}\left\langle\mathbf{P}_v^{\perp} \mathbf{w}_r^{(T_k)}, y_k \boldsymbol{\xi}_d\right\rangle^3<-\epsilon \sigma_{\xi}^3\right) \\
        & \geq \frac{1}{2}-\mathbb{P}_{\boldsymbol{\xi}_d}\left(\left|\sum_{r \in[R]} \left\langle\mathbf{P}_v^{\perp} \mathbf{w}_r^{(T_k)}, y_k \boldsymbol{\xi}_d\right\rangle^3\right| \leq \epsilon \sigma_{\xi}^3\left|\sum_{r \in[R]}\left\langle\mathbf{P}_v^{\perp} \mathbf{w}_r^{(T_k)}, \boldsymbol{u}\right\rangle^3\right|\right) \\
        & \geq \frac{1}{2}-O\left(\epsilon^{1 / 3}\right) .
        \end{aligned}
    \end{equation}
    Taking $\epsilon= 1/\operatorname{polylog}(d)$, it holds that
    $$
    \mathbb{P}\left(\sum_{r \in[R]}\left\langle\mathbf{P}_v^{\perp} \mathbf{w}_r^{(T_k)}, y_k \boldsymbol{\xi}_d\right\rangle^3<-\widetilde{O}\left(\sigma_{\xi}^3\right)\right) \geq \frac{1}{2}-\frac{1}{\operatorname{polylog}(d)}.
    $$
    Moreover, along with \cref{eq:f_phase1}, we can further obtain the following:
    $$
    \begin{aligned}
        y_k F\left(\boldsymbol{W}^{(T_k)}, \mathbf{x}_k\right) &\leq \sum_{r \in [R]} y_k\left\langle\mathbf{w}_r^{(T_k)}, \boldsymbol{\xi}\right\rangle^3+ \widetilde{O}(R \alpha_k^3 \sigma_0^3) \\
        &{\leq} -\widetilde{O}\left(\sigma_{\xi}^3\right) +  \widetilde{O}(R \alpha_k^3 \sigma_0^3)   \\
        & \overset{(i)}{\leq} 0,
    \end{aligned}
    $$
    where $(i)$ comes from the Condition~\ref{con:parameter} and the SNR choices. 
    The proofs for $T_k$ and $T_m$ are identical, where $t=T_m$, it still holds that $\Gamma \leq \widetilde{O}(\sigma_0)$ and $  \max _{r \in[R]} y_{1 j} \Phi_{(m, r)}^{(T_m,k,j)} \geq \Theta(R^{-\frac{1}{4}})$. Hence, the remainder of the proof proceeds exactly as in the case $t = T_m$.

    For the scenario of enhanced signal learning, the noise memorization of training phase 2 will be under control and the signal learning will increase as stated in Lemmar~\ref{lem:noise_fs_app}. Thus, given the new test data $(\mathbf{x}_k, y_k) \sim \mathcal{D}_k$ for task k, with probability at least $1-1 /$ poly $(d)$, we have
    \begin{equation*}
        \begin{aligned}
            y_k F\left(\boldsymbol{W}^{(T_m)}, \mathbf{x}\right) &=y_k \sum_{r \in[R]} (\langle\mathbf{w}_r^{(T_m)}, \mathbf{x}_k^1\rangle^3 + \langle\mathbf{w}_r^{(T_m)}, \mathbf{x}_k^2\rangle^3) \\
            & = y_k \sum_{r \in[R]} \left\langle\boldsymbol{W}^{(T_m)}, y_k \alpha_k\mathbf{v}_k^*\right\rangle^3  + \left\langle\boldsymbol{W}^{(T_m)}, \bm{\xi}_k \right\rangle^3 \\
            & \overset{(i)}{\geq} \Theta(\alpha_k^3 \cdot R \cdot \alpha_k^{-3} R^{-3/5}) \pm \Theta(R \cdot R^{-3/4})\\
            & \geq \widetilde{\Omega}(1).
        \end{aligned}
    \end{equation*}
    Here, $(i)$ follows from Lemma~\ref{lem:noise_fs_app}, which shows that, under the SNR condition stated in Theorem~\ref{thm:forget_m_ffs_tasks_app}, noise memorization is slower than signal learning.
\end{proof}

\subsection{Proof of Theorem~\ref{thm:forget_m_sfs_tasks}}

In this section, we present the proof of Theorem~\ref{thm:forget_m_sfs_tasks} in two parts. The first part analyzes the success of signal learning after training on $k$ tasks (i.e., before task $k+1$ ). The second part focuses on noise memorization after training on $m>k$ tasks (i.e., before task $m+1$ ) and further considers two scenarios in the later phase: one where learning fails to retain previously acquired features, and another where signal learning continues to improve.

\begin{lemma}\label{lem:phi_leq_ss}
    During the data replay training process, with probability at least $1-1 /\operatorname{poly}(d)$, it holds that:
    $$
    \max _{r \in[R], k \in [m], j \in[n]}| \Phi_{(k, r)}^{(t,k,j)}| \leq \widetilde{O}(\sigma_0 \sigma_{\xi} \sqrt{d}) \quad \text { for any } \quad t \leq \mathrm{T}_k.
    $$
\end{lemma}

\begin{proof}[\bf Proof of Lemma~\ref{lem:phi_leq_ss}]
    According to the initialization and the concentration by Lemma~\ref{lem:w0_v_xi}, with probability at least $1-1 /\operatorname{poly}(d)$, it holds that
    \begin{equation*}
        \bar{\Phi}_{(0, r)}^{(0)} := \max _{ k \in [M] ,j \in[n]}| \Phi_{(0, r)}^{(0,k,j)}| = \max _{ j \in[n]}| \langle \mathbf{w}_{(0,r)}^{(0)}, \bm{\xi}_{kj} \rangle|\leq \widetilde{O}(\sigma_0 \sigma_{\xi} \sqrt{d}).
    \end{equation*}
    Next, we consider the induction process to prove the statement. First, we assume that $\Phi_{(k, r)}^{(s)} \leq \widetilde{O}(\sigma_0 \sigma_{\xi} \sqrt{d})$ holds for any $s \leq t$. Then, we proceed to analyze the case for $s=t+1$. Denote $\bar{\Phi}_{(k, r)}^{(s,k,j)} = \max _{ k \in [M] ,j \in[n]}| \Phi_{(k, r)}^{(s,k,j)}|$, according to the update rule \eqref{eq:update_noise_data_replay}, we have
    \begin{equation*}
        \begin{aligned}
            \bar{\Phi}_{(k, r)}^{(s+1)} &\leq \max _{k \in [M] ,j \in[n]} \Phi_{(k, r)}^{(s,k,j)} +  \frac{3\eta }{nm} \sum_{p=1}^{k} \sum_{j^{\prime} \in[n]} y_{p j^{\prime}} \ell_{p j^{\prime}}(\mathbf{W}_k^{(t-1)}) (\Phi_{(k, r)}^{(t-1,p,j)})^2 \langle \bm{\xi}_{p j^{\prime}},\bm{\xi}_{k j} \rangle \\
            & \overset{(i)}{\leq} \bar{\Phi}_{(k, r)}^{(s)} + \frac{3 \eta \sqrt{d} \sigma_{\xi}^2 (n-1) k }{nm} ( \bar{\Phi}_{(k, r)}^{(s)})^2 + \frac{3 \eta {d} \sigma_{\xi}^2}{nm} ( \Phi_{(k, r)}^{(s,k,j)})^2 \\
            & \overset{(ii)}{\leq}  \bar{\Phi}_{(k, r)}^{(s)} + \frac{3k (n-1) \eta d\sqrt{d} \sigma_{\xi}^4 \sigma_0^2}{nm} + \frac{3 \eta d^2 \sigma_{\xi}^4 \sigma_0^2}{nm} \\
            & \leq  \Phi_{(0, r)}^{0} +  \frac{3k(s-1)(n-1) \eta d\sqrt{d} \sigma_{\xi}^4 \sigma_0^2}{nm} + \frac{3 (s-1)\eta d^2 \sigma_{\xi}^4 \sigma_0^2}{nm} \\
            & \overset{(iii)}{\leq} \widetilde{O}(\sigma_0 \sigma_{\xi} \sqrt{d}) + O(\frac{3 T_k k (n-1) \eta d\sqrt{d} \sigma_{\xi}^4 \sigma_0^2}{nm} + \frac{3 T_k \eta d^2 \sigma_{\xi}^4 \sigma_0^2}{nm}) \\
            & \overset{(iv)}{\leq} \widetilde{O}(\sigma_0 \sigma_{\xi} \sqrt{d}),
        \end{aligned}
    \end{equation*}
    where $(i)$ holds due to the concentration in Lemma~\ref{lem:xi_bound}; $(ii)$ derives from the induction hypothesis; $(iii)$ comes from $s+1 \leq T_k$; $(iv)$ holds due to $T_k \leq \frac{m}{\eta \sigma_{0} \sigma_{\xi}}$.
\end{proof}

\begin{lemma}[Restatement of Lemma~\ref{lem:signal_sf}]\label{lem:signal_sf_app}
    Suppose the SNR satisfying $ \frac{\sum_{p=1}^{k} \alpha_p^3 A_{(p,k)}}{(\sigma_{\xi}\sqrt{d})^3} \gtrsim \frac{1+n k^2/\sqrt{d}}{kn} $, and there exists an iteration $\tau_{k v}^k \leq \mathrm{T}_{v}^k=\mathrm{T}_{v}^{-}+O(\log (d))$ such that $\tau_{k v}^k$ is the first iteration where $ \max _{r \in[R]}|\Gamma_{(k,r)}^{(t,k)}| \geq \Theta(\frac{1}{\alpha_k R^{1/3}})$, and for any $t \leq \mathrm{T}_{v}^k$ it holds that $\max _{r \in[R]}|\Phi_{(k, r)}^{(t,k,j)})| \leq \widetilde{O}(\sigma_0\sigma_{\xi}\sqrt{d})$.
    Then, if the additional SNR condition $ \frac{m^2}{R^{2/3} \sigma_0^2 \sigma_{\xi}^2 d^{13/6}} \lesssim \frac{\sum_{p=1}^{m} \alpha_p^3 A_{(p,k)}}{(\sigma_{\xi}\sqrt{d})^3} \lesssim \frac{\alpha_k R^{1/3}}{n}$ also holds, there exists $\tau_{k j}^m \leq \mathrm{T}_{\xi}^m=\mathrm{T}_{\xi}^{k}+O(\log (d))$ such that  $\tau_{k j}^m$ be the first iteration satisfying $\max _{r \in[R]}(y_{k j} \Phi_{(m, r)}^{(t,k,j)}) \geq \Theta(R^{-\frac{1}{5}})$. 
\end{lemma}

\begin{proof}[\bf Proof of Lemma~\ref{lem:signal_sf_app}]
    In the first training phase, the noise memorization can be controlled by Lemma~\ref{lem:phi_leq_ss}. Thus, we only need to consider the signal learning process here. By learning dynamic of signal in \cref{eq:update_signal_data_replay}, we have:
    \begin{equation}
        \begin{aligned}
            \Gamma_{(k, r^*)}^{(t,k)} 
             &=  \Gamma_{(k, r^*)}^{(t-1,k)} + \frac{\eta}{n} \sum_{j \in[n]} \sum_{p \in[k]}  3\alpha_p^3 A_{(p,k)} \ell_{p j}(\mathbf{W}_k^{(t-1)}) (\Gamma_{(k, r^*)}^{(t-1,p)})^2  \\
             &= \Gamma_{(k, r^*)}^{(t-1,k)} + \Theta\left( \eta \sum_{p \in[k]}  3\alpha_p^3 A_{(p,k)}  \right)  (\Gamma_{(k, r^*)}^{(t-1,p)})^2 .
        \end{aligned}
    \end{equation}
    Then, by applying the tensor power method from Lemma~\ref{lem:K15_extension} to the sequence $\{\Gamma_{(k, r^*)}^{(s,k)}\}_{s \geq T_k}$, let $h=H= {3\eta (\sum_{p \in[k]}  3\alpha_p^3 A_{(p,k)}) }$, $z^{(0)}= \Gamma_{(k+1, r^*)}^{(0,k)} \geq O(\sigma_0)$, $ v=\Theta(\frac{1}{\alpha_k R^{1/3}}),$ then we obtain:
    \begin{equation*}
        \begin{aligned} 
        \tau_{k v}^k & \leq \frac{m}{3\eta \sigma_0 \sum_{p \in[k]}  3\alpha_p^3 A_{(p,k)} }+8\left[\frac{\log \left(v /z^{(0)}\right)}{\log (2)}\right] \\ 
        & \leq  \frac{m}{3\eta \sigma_0 \sum_{p \in[k]}  3\alpha_p^3 A_{(p,k)}} + \log d = T_{v}^k.
        \end{aligned}
    \end{equation*}
    When considering the second phase training, we first consider the signal learning and denote ${\tau}_{kv}^m$ as the first time that $\Gamma_{(m, r^*)}^{(t,k)}$ exceeds $(\alpha_k R)^{-\frac{1}{4}}$. Then, the signal learning dynamic will be:
    \begin{equation}
        \begin{aligned}
            \Gamma_{(m, r^*)}^{(t,k)} 
             &=  \Gamma_{(k, r^*)}^{(t-1,k)} + \frac{\eta}{mn} \sum_{j \in[n]} \sum_{p \in[m]}  3\alpha_p^3 A_{(p,k)} \ell_{p j}(\mathbf{W}_k^{(t-1)}) (\Gamma_{(m, r^*)}^{(t-1,p)})^2  \\
             &= \Gamma_{(m, r^*)}^{(t-1,k)} + \Theta\left( \frac{\eta}{m} \sum_{p \in[m]}  3\alpha_p^3 A_{(p,k)}  \right)  (\Gamma_{(m, r^*)}^{(t-1,p)})^2 .
        \end{aligned}
    \end{equation}
    Then, we still apply the tensor power method from Lemma~\ref{lem:K15_extension} to the sequence $\{\Gamma_{(m, r^*)}^{(s,k)}\}_{s \geq T_m}$, but with modified parameters, such that: $h=H= {3\frac{\eta}{m} (\sum_{p \in[m]}  3\alpha_p^3 A_{(p,k)}) }$, $z^{(0)}= \Gamma_{(k+1, r^*)}^{(0,k)} \geq \Theta(\frac{1}{\alpha_k R^{1/3}})$, $ v=\Theta(\frac{1}{\alpha_k R^{1/4}}),$ then we obtain:
    \begin{equation*}
        \begin{aligned} 
        \tau_{k v}^k & \leq T_v^k + \frac{m}{3\eta \sigma_0 \sum_{p \in[m]}  3\alpha_p^3 A_{(p,k)} }+8\left[\frac{\log \left(v /z^{(0)}\right)}{\log (2)}\right] \\ 
        & \leq  \frac{m}{3\eta \sigma_0 \sum_{p \in[k]}  3\alpha_p^3 A_{(p,k)}} + \frac{m}{3\eta \sigma_0 \sum_{p \in[m]}  3\alpha_p^3 A_{(p,k)}} + \log d = T_{v}^m.
        \end{aligned}
    \end{equation*}
    Therefore, as long as $t \leq T_{v}^m$, signal learning remains bounded by $\Theta\left(\frac{1}{\alpha_k R^{1/4}}\right)$. In the sequel, we show that noise memorization can accumulate to $R^{-1 / 5}$, making the noise term larger than the signal.
    
    % It can be shown that the following holds, 
    Similar to Lemma~\ref{lem:noise_fs}, it can be derived that $y_{kj}\Phi_{(m, r^*)}^{(t+1,k,j)} \leq (Rd)^{-1/3}$ for any $t \leq \tau_{r^*, k j}^{k} := T_k +  \Theta\left(\frac{1}{\eta} \frac{mn}{\left(\sqrt{d} \sigma_{\xi}\right)^3 \sigma_0}\right)$. Then, it can be shown that the following holds:
    \begin{equation}
        \begin{aligned}
            y_{kj}\Phi_{(m, r^*)}^{(t+1,k,j)} & {=} y_{kj}\Phi_{(k+1, r^*)}^{(0,k,j)} + \Theta\left(\frac{3\eta{d} \sigma_{\xi}^2}{nm}\right) \sum_{s=1}^{t-1} \left(y_{kj}\Phi_{(m, r^*)}^{(s,m,j)}\right)^2 \\
            &\pm \Theta\left(\frac{3\eta \sqrt{d} \sigma_{\xi}^2}{nm}\right) \sum_{q=k}^{m} \sum_{\substack{p \in [q],\, j^{\prime} \in [n] \\ (p, j^{\prime}) \neq (k, j)}} \sum_{s=1}^{t-1} \ell_{p j^{\prime}}(\mathbf{W}_m^{(s)}) \left(y_{kj}\Phi_{(m, r^*)}^{(s,p,j)}\right)^2 \\
            & \overset{(i)}{=} y_{kj}\Phi_{(k+1, r^*)}^{(0,k,j)} + \Theta\left(\frac{3\eta{d} \sigma_{\xi}^2}{nm}\right) \sum_{s=1}^{t-1} \left(y_{kj}\Phi_{(m, r^*)}^{(s,p,j)}\right)^2 \\
            &\pm \widetilde{O}\left(\frac{3 (\alpha_k R)^{1/3} \sqrt{d} \sigma_{\xi}^2 (m^2 - k^2)}{  \sum_{p \in [m]} 3 \alpha_p^3 A_{(p,k)}} \cdot \frac{1}{(Rd)^{2/3}}\right) \\
            & \overset{(ii)}{=} y_{kj}\Phi_{(k+1, r^*)}^{(0,k,j)} + \Theta\left(\frac{3\eta{d} \sigma_{\xi}^2}{nm}\right) \sum_{s=T_1}^t\left(y_{1j}\Phi_{}^{(s,1,j)}\right)^2 \pm o\left( {(Rd)^{-1/3}}\right). 
        \end{aligned}
    \end{equation}
    Here, $(i)$ follows from Lemma~\ref{lem:k_sum_l} with the range of $p \in[m]$ and $z^{(0)} = \Gamma_{(k,r^*)}^{\tau_{kv}^k,k}$ adjusted accordingly; $(ii)$ is derived from the robustness of the SNR choices.
    
    Let $A=\Theta(\frac{\eta d \sigma_{\xi}^2}{mn}), C=o((Rd)^{-\frac{1}{3}} ),  v=\Theta(R^{-\frac{1}{5}}).$ 
    By applying the tensor power method via Lemma~\ref{lem:K16}, we also have:
    \begin{equation*}
        \begin{aligned} 
        \tau_{r^*, k j}^m & \leq \tau_{r^*, k j}^k+\frac{21}{A y_{k j} \Phi_{r^*, k j}^{\left(\tau_{r^*, k j}^k\right)}}+8\left[\frac{\log \left(v /\left[y_{k j} \Phi_{r^*, k j}^{\left(\tau_{r^*, k j}^k\right)}\right]\right)}{\log (2)}\right] \\ 
        & \leq T_k + \Theta\left(\frac{1}{\eta} \frac{ mn}{\left(\sqrt{d} \sigma_{\xi}\right)^3 \sigma_0}\right)+ \Theta\left(\frac{1}{\eta} \frac{ n mR^{1 / 3}}{d^{2/3} \sigma_{\xi}^2}\right) + \log d\\ 
        & \leq \Theta\left(\frac{1}{\eta} \frac{ mn}{\left(\sqrt{d} \sigma_{\xi}\right)^3 \sigma_0}+ \frac{1}{\eta} \frac{ mn R^{1 / 3}}{d^{2 / 3} \sigma_{\xi}^2}  \right)+ \log d = T_{\xi}^m.
        \end{aligned}
    \end{equation*}
    Based on the condition of SNR, we have $T_{\xi}^m \leq T_v^m$, which indicates that noise memorization exceeds signal learning during the second phase.
\end{proof}

\begin{lemma}[Restatement of Lemma~\ref{lem:signal_ss}]\label{lem:signal_ss_app}
    Suppose the SNR satisfying $ \frac{\sum_{p=1}^{k} \alpha_p^3 A_{(p,k)}}{(\sigma_{\xi}\sqrt{d})^3} \gtrsim \frac{1+n k^2/\sqrt{d}}{kn} $, and there exists an iteration $\tau_{k v}^k \leq \mathrm{T}_{v}^k=\mathrm{T}_{v}^{-}+O(\log (d))$ such that $\tau_{k v}^k$ is the first iteration where $\max _{r \in[R]}|\Gamma_{(k,r)}^{(t,k)}| \geq \Theta(\frac{1}{\alpha_k R^{1/3}})$, and for any $t \leq \mathrm{T}_{v}^k$ it holds that $\max _{r \in[R]}|\Phi_{(k, r)}^{(t,k,j)})| \leq \widetilde{O}(\sigma_0\sigma_{\xi}\sqrt{d})$.
    Then, if the additional SNR condition $ \frac{\sum_{p=1}^{m} \alpha_p^3 A_{(p,k)}}{(\sigma_{\xi}\sqrt{d})^3} \gtrsim \frac{\alpha_k R^{1/3} \sigma_0 \left((1-\frac{k-1}{m}) + nm/\sqrt{d} \right)}{n}$ also holds, there exists $\tau_{k v}^m \leq \mathrm{T}_{v}^m=\mathrm{T}_{v}^{k}+O(\log (d))$ such that $\tau_{k v}^m$ be the first iteration satisfying $\max _{r \in[R]}|\Gamma_{(m,r)}^{(t,k)}| \geq \Theta(\frac{1}{\alpha_k R^{1/5}})$. 
\end{lemma}    

\begin{proof}[\bf Proof of Lemma~\ref{lem:signal_ss_app}]
    The proof of the first training phase is identical to Lemma~\ref{lem:signal_sf_app}. Thus, we only focus on the second training phase. Similarly, we have the update for signal learning as follows: 
    \begin{equation}
        \begin{aligned}
            \Gamma_{(m, r^*)}^{(t,k)} 
             &=  \Gamma_{(k, r^*)}^{(t-1,k)} + \frac{\eta}{mn} \sum_{j \in[n]} \sum_{p \in[m]}  3\alpha_p^3 A_{(p,k)} \ell_{p j}(\mathbf{W}_k^{(t-1)}) (\Gamma_{(m, r^*)}^{(t-1,p)})^2  \\
             &= \Gamma_{(m, r^*)}^{(t-1,k)} + \Theta\left(\frac{\eta}{m} \sum_{p \in[m]}  3\alpha_p^3 A_{(p,k)}  \right)  (\Gamma_{(m, r^*)}^{(t-1,p)})^2 .
        \end{aligned}
    \end{equation}
    Then, we still apply the tensor power method from Lemma~\ref{lem:K15_extension} to the sequence $\{\Gamma_{(m, r^*)}^{(s,k)}\}_{s \geq T_m}$, but with modified parameters, such that: $h=H= {3\frac{\eta}{m} (\sum_{p \in[m]}  3\alpha_p^3 A_{(p,k)}) }$, $z^{(0)}= \Gamma_{(k+1, r^*)}^{(0,k)} \geq \Theta(\frac{1}{\alpha_k R^{1/3}})$, $ v=\Theta(\frac{1}{\alpha_k R^{1/5}}),$ then we obtain:
    \begin{equation*}
        \begin{aligned} 
        \tau_{k v}^k & \leq T_v^k + \frac{m}{3\eta \sigma_0 \sum_{p \in[m]}  3\alpha_p^3 A_{(p,k)} }+8\left[\frac{\log \left(v /z^{(0)}\right)}{\log (2)}\right] \\ 
        & \leq  \frac{m}{3\eta \sigma_0 \sum_{p \in[k]}  3\alpha_p^3 A_{(p,k)}} + \frac{m}{3\eta \sigma_0 \sum_{p \in[m]}  3\alpha_p^3 A_{(p,k)}} + \log d = T_{v}^m.
        \end{aligned}
    \end{equation*}
    Therefore, as long as $t \leq T_{v}^m$, signal learning remains bounded by $\Theta\left(\frac{1}{\alpha_k R^{1/5}}\right)$. Moreover, according to the SNR condition, we have $T_v^m \leq T_{\xi}^m$, which indicates that during the training phase the noise memorization will not exceed $\Theta(\frac{1}{R^{1/3}})$.
\end{proof}

\begin{theorem}[Restatement of Theorem~\ref{thm:forget_m_sfs_tasks_app}]\label{thm:forget_m_sfs_tasks_app}
    Suppose the setting in Condition~\ref{con:parameter} holds, and the SNR satisfies $ \frac{\sum_{p=1}^{k} \alpha_p^3 A_{(p,k)}}{(\sigma_{\xi}\sqrt{d})^3} \gtrsim \frac{1+n k^2/\sqrt{d}}{kn} $. Consider full data-replay training with learning rate $\eta \in(0, \widetilde{O}(1)]$, and let $(\mathbf{x}_k, y_k) \sim \mathcal{D}_k$ be a test sample from the task $k$. Then, with high probability, there exist training times $T_k$ and $T_m$ ($m>k$) such that 
    \begin{itemize}
        \item The model can correctly classify task $k$ immediately after learning it: 
        \begin{equation}
            \mathbb{P}\left\{y_k F\left(\mathbf{W}^{(T_k)}, \mathbf{x}_k\right)<0\right\} \leq {\frac{1}{\operatorname{poly}(d)}}. 
        \end{equation}
        \item (Catastrophic Forgetting on Task $k$) If the additional SNR conditions holds $ \frac{m^2}{R^{2/3} \sigma_0^2 \sigma_{\xi}^2 d^{13/6}} \lesssim \frac{\sum_{p=1}^{m} \alpha_p^3 A_{(p,k)}}{(\sigma_{\xi}\sqrt{d})^3} \lesssim \frac{\alpha_k R^{1/3}}{n} $, then it occurs \textit{Catastrophic Forgetting} on task $k$ after subsequent training to task $m$:
        \begin{equation}
            \mathbb{P}\left\{y_k F\left(\mathbf{W}^{(T_m)}, \mathbf{x}_k\right)<0\right\} \geq \frac{1}{2}-{\frac{1}{\operatorname{polylog}(d)}}.
        \end{equation}
        \item (Continual Learning on Task $k$) If the additional SNR conditions holds $ \frac{\sum_{p=1}^{m} \alpha_p^3 A_{(p,k)}}{(\sigma_{\xi}\sqrt{d})^3} \gtrsim \frac{\alpha_k R^{1/3} \sigma_0 \left((1-\frac{k-1}{m}) + nm/\sqrt{d} \right)}{n} $, then the model can still correctly classify task $k$ after subsequent training to task $m$:
        \begin{equation}
            \mathbb{P}\left\{y_k F\left(\mathbf{W}^{(T_m)}, \mathbf{x}_k\right)<0\right\} \leq {\frac{1}{\operatorname{poly}(d)}}.
        \end{equation}
    \end{itemize}
\end{theorem}

\begin{proof}[\bf Proof of Theorem~\ref{thm:forget_m_sfs_tasks_app}]
    We first present the analysis for the initial training phase; the results for the second phase in the continual learning scenario follow analogously, with the primary difference lying in the bound on noise memorization. Given the new test data $(\mathbf{x}_k, y_k) \sim \mathcal{D}_k$ for task k, with probability at least $1-1 /$ poly $(d)$, we have
    \begin{equation*}
        \begin{aligned}
            y_k F\left(\boldsymbol{W}^{(T_k)}, \mathbf{x}\right) &=y_k \sum_{r \in[R]} (\langle\mathbf{w}_r^{(T_k)}, \mathbf{x}_k^1\rangle^3 + \langle\mathbf{w}_r^{(T_k)}, \mathbf{x}_k^2\rangle^3) \\
            & = y_k \sum_{r \in[R]} \left\langle\boldsymbol{W}^{(T_k)}, y_k \alpha_k\mathbf{v}_k^*\right\rangle^3  + \left\langle\boldsymbol{W}^{(T_k)}, \bm{\xi}_k \right\rangle^3 \\
            & \overset{(i)}{\geq} \Theta(\alpha_k^3 \cdot R \cdot \alpha_k^{-3} R^{-1}) \pm \Theta(R \cdot \sigma_0^3\sigma_{\xi}^3 d^{3/2})\\
            & \geq \widetilde{\Omega}(1).
        \end{aligned}
    \end{equation*}
    Here, $(i)$ follows from Lemma~\ref{lem:signal_ss_app} and the SNR condition stated in Theorem~\ref{thm:forget_m_sfs_tasks}. The second phase differs from $\Gamma \geq \Theta(\frac{1}{\alpha_k R^{1/5}})$ and $\Phi \leq \Theta(R^{-1/3})$.

    Next, we present the proof of Catastrophic Forgetting during the second phase. Given a new test sample $(\mathbf{x}_k, y_k) \sim \mathcal{D}_k$ from task $k$, and noting that we consider binary classification with labels $y = \pm 1$, it follows that, with probability at least $1/2 - 1/\mathrm{poly}(d)$, the label $y_k$ will interact oppositely with the $\Phi$, which implies that:
    \begin{equation*}
        \begin{aligned}
            y_k F\left(\boldsymbol{W}^{(T_k)}, \mathbf{x}\right) &=y_k \sum_{r \in[R]} (\langle\mathbf{w}_r^{(T_k)}, \mathbf{x}_k^1\rangle^3 + \langle\mathbf{w}_r^{(T_k)}, \mathbf{x}_k^2\rangle^3) \\
            & = y_k \sum_{r \in[R]} \left\langle\boldsymbol{W}^{(T_k)}, y_k \alpha_k\mathbf{v}_k^*\right\rangle^3  + \left\langle\boldsymbol{W}^{(T_k)}, \bm{\xi}_k \right\rangle^3 \\
            & \overset{(i)}{\leq} \Theta(\alpha_k^3 \cdot R \cdot \alpha_k^{-3} R^{-1}) - \Theta(R \cdot R^{-3/4})\\
            & \leq 0.
        \end{aligned}
    \end{equation*}
     Here, $(i)$ holds due to Lemma~\ref{lem:signal_sf_app} and the SNR condition stated in Theorem~\ref{thm:forget_m_sfs_tasks}.
\end{proof}

\section{Supplementary Lemmas}

\begin{lemma}\label{lem:xi_bound}
Suppose that $\delta_{\xi}>0$ and $d=\Omega(\log (4 n / \delta_{\xi}))$. Then, for all $i, i^{\prime} \in[n]$, with probability at least $1-\delta_{\xi}$,
$$
\begin{aligned}
& \sigma_{\xi}^2 d / 2 \leq\|\bm{\xi}_i\|_2^2 \leq 3 \sigma_{\xi}^2 d / 2 \\
& |\langle\bm{\xi}_i, \bm{\xi}_{i^{\prime}}\rangle| \leq 2 \sigma_{\xi}^2 \cdot \sqrt{d \log (4 n^2 / \delta_{\xi})}.
\end{aligned}
$$
\end{lemma}

\begin{lemma}\label{lem:w0_v_xi}
    Under the Gaussian initialization, with probability $1-1 /$ poly $(d)$, we have
    \begin{itemize}
        \item Given any $m \in[M]$, $\max _{r \in[R]} \Gamma_{(m,r)}^{(0)}>\Omega\left(\sigma_0\right)$. In addition, $\max _{r \in[R], m \in[M]}\left|\Gamma_{(m,r)}^{(0)}\right| \leq O\left(\sigma_0 \sqrt{\log d}\right)$.
        \item Given any $k \in [M]$ and $j \in[n], \max _{r \in[R]} y_{k j} \Phi_{(0,r)}^{(0,k,j)}>\Omega\left(\sqrt{d} \sigma_{\xi} \sigma_0\right)$. In addition, $\max _{r \in[R], k \in[M], j \in[n]}\left|\Phi_{(0,r)}^{(0,k,j)}\right| \leq O\left(\sigma_{\xi} \sigma_0 \sqrt{d \log d}\right)$.
    for all $r \in[R]$ and $m \in[M]$.
    \end{itemize}
\end{lemma}

The proof of Lemma~\ref{lem:xi_bound} and Lemma~\ref{lem:w0_v_xi} can be derived directly from the properties of the Gaussian distribution.
In the following, we will provide some tensor power lemmas that can be extended to $m$ cases.

\begin{lemma}[Lemma K. 12 in \cite{jelassi2022towards}]\label{lem:K12}
Let $\{\mathbf{w}_r\}_{r=1}^R$ be vectors in $\mathbb{R}^d$ and $\boldsymbol{\xi} \sim \mathcal{N}(0, \sigma_{\xi}^2 \mathbf{I}_d)$. If there exists a unit norm vector $\boldsymbol{u}$ such that $|\sum_{r=1}^R\langle\mathbf{w}_r, \boldsymbol{u}\rangle^3| \geq 1$, then for any $\epsilon \in(0,1)$, we have
$$
\mathbb{P}\left(\left|\sum_{r=1}^R\left\langle\mathbf{w}_r, \boldsymbol{\xi}\right\rangle^3\right| \leq \epsilon \sigma_{\xi}^3\right) \leq O\left(\epsilon^{1 / 3}\right) .
$$
\end{lemma}

\begin{lemma}[Lemma K. 15 in \cite{jelassi2022towards}]\label{lem:K15}
     Let $\left\{z^{(t)}\right\}_{t=0}^T$ be a positive sequence defined by the following recursions:
    $$
    \begin{aligned}
    & z^{(t+1)} \geq z^{(t)}+h\left[z^{(t)}\right]^2, \\
    & z^{(t+1)} \leq z^{(t)}+H\left[z^{(t)}\right]^2,
    \end{aligned}
    $$
    where $z^{(0)}>0$ is the initialization and $h, H>0$. Let $v>0$ such that $z^{(0)} \leq v$ and $t_0$ be the first iteration $z^{(t)} \geq v$. Then, we have
    $$
    t_0 \leq \frac{3}{h z^{(0)}}+\frac{8 H}{h}\left\lceil\frac{\log \left(v / z^{(0)}\right)}{\log (2)}\right\rceil.
    $$
\end{lemma}

\begin{lemma}\label{lem:K15_extension}
     Let $\left\{z_i^{(t)}\right\}_{t=0}^T$ be a positive sequence defined by the following recursions:
    $$
    \begin{aligned}
    & z_i^{(t+1)} \geq z_i^{(t)}+h\sum_{j=1}^m \left[z_j^{(t)}\right]^2, \\
    & z_i^{(t+1)} \leq z_i^{(t)}+H\sum_{j=1}^m \left[z_j^{(t)}\right]^2, \\
    \end{aligned}
    $$
    where $z_j^{(0)}>0 (j \in [m])$ is the initialization and $h, H>0$. Let $v>0$ such that $\max_j z_j^{(0)} \leq v$ and $t_0$ be the first iteration $z_j^{(t)} \geq v$. Then, we have
    $$
    t_0 \leq \frac{3}{h \max_j z_j^{(0)}}+\frac{8 H m}{h}\left\lceil\frac{\log \left(v / z^{(0)}\right)}{\log (2)}\right\rceil.
    $$
\end{lemma}

\begin{proof}[\bf Proof of Lemma~\ref{lem:K15_extension}]
    Let$M^{(t)}=\max \left(z_1^{(t)}, \ldots, z_m^{(t)}\right)$. Due to symmetry, it suffices to analyze $M^{(t)}$. Fix any time step $t$. Suppose $M^{(t)}=z_k^{(t)}$, then the lower bound is:
    $$
    z_k^{(t+1)} \geq z_k^{(t)}+h\left(\left[z_k^{(t)}\right]^2+\sum_{j \neq k}\left[z_j^{(t)}\right]^2\right) \geq M^{(t)}+h\left[M^{(t)}\right]^2
    $$
    Therefore, we have $M^{(t+1)} \geq M^{(t)}+h\left[M^{(t)}\right]^2$.
    The sum of squares of all variables satisfies:
    $$
    \sum_{j=1}^m\left[z_j^{(t)}\right]^2 \leq m\left[M^{(t)}\right]^2.
    $$
    Therefore, for any $z_i^{(t+1)}$, we have:
    $$
    z_i^{(t+1)} \leq z_i^{(t)}+H \cdot m\left[M^{(t)}\right]^2.
    $$
    Hence,
    $$
    M^{(t+1)} \leq M^{(t)}+H m\left[M^{(t)}\right]^2.
    $$
    Replace $H$ in Lemma~\ref{lem:K15} with $H m$, and let the initial value be $M^{(0)}$. Applying the result directly yields:
    $$
    t_0 \leq \frac{3}{h M^{(0)}}+\frac{8 H m}{h}\left[\frac{\log \left(v / M^{(0)}\right)}{\log 2}\right].
    $$

\end{proof}

\begin{lemma}[Lemma K. 16 in \cite{jelassi2022towards} ]\label{lem:K16} Let $\left\{z^{(t)}\right\}_{t=0}^T$ be a positive sequence defined by the following recursions
$$
\begin{aligned}
& z^{(t)} \geq z^{(0)}+A \sum_{s=0}^{t-1}\left[z^{(s)}\right]^2-C, \\
& z^{(t)} \leq z^{(0)}+A \sum_{s=0}^{t-1}\left[z^{(s)}\right]^2+C,
\end{aligned}
$$
where $A, C>0$ and $z^{(0)}>0$ is the initialization. Assume that $C \leq z^{(0)} / 8$. Let $t_0$ be the first iteration $z^{(t)} \geq v$. If $v>z^{(0)}$, we have the following upper bound
$$
t_0 \leq \frac{21}{A z^{(0)}}+8\left\lceil\frac{\log \left(v / z^{(0)}\right)}{\log (2)}\right\rceil.
$$
\end{lemma}

\begin{lemma}[Lemma E. 7 in \cite{baoprovable} ]\label{lem:E7}
For the same sequence $\left\{z^{(t)}\right\}_{t \geq 0}$ be a positive sequence satisfying the recursive upper bound in \cref{lem:K16} Let $v>0$ such that $z^{(0)} \leq v$ and $t_0$ be the first iteration $z^{(t)} \geq v$. For any $v \geq 2 z^{(0)}$, we have the following lower bound
$$
t_0 \geq \frac{1}{8 A z^{(0)}}.
$$

\begin{lemma}[Lemma E. 8 in \cite{baoprovable} ]\label{lem:E8}
    Let $\left\{z^{(t)}\right\}_{t=0}^T$ and $\left\{a^{(t)}\right\}_{t=0}^T$ be two positive sequences admitting the following recursions
$$
\begin{aligned}
& z^{(t+1)} \geq z^{(t)}+h a^{(t)}\left[z^{(t)}\right]^2, \\
& z^{(t+1)} \leq z^{(t)}+H a^{(t)}\left[z^{(t)}\right]^2,
\end{aligned}
$$

where $0<h<H$ and $z^{(0)}>0$. If $\max _{t \leq T} a^{(t)} \leq A$, we have

$$
\sum_{s=0}^T a^{(s)} \leq \frac{4}{h z^{(0)}}+\frac{8 H A}{h}\left\lceil\frac{\log \left(z^{(T)} / z^{(0)}\right)}{\log (2)}\right\rceil,
$$

and

$$
\sum_{s=0}^T a^{(s)} \geq \frac{z^{(T)}-z^{(0)}}{H\left[z^{(T)}\right]^2}.
$$

\end{lemma}
\end{lemma}

\end{document}